\documentclass{article}
\usepackage[margin=1in]{geometry}

\usepackage{xcolor}  
\usepackage{preamble/color-edits}
\usepackage[utf8]{inputenc}
\usepackage[T1]{fontenc}
\usepackage{nicefrac}  
\usepackage{mathtools,verbatim}
\usepackage{times}
\usepackage{amsmath}
\usepackage{amsfonts}
\usepackage{amssymb}
\usepackage{textcomp}
\usepackage{inconsolata}
\usepackage{pgfplots}
\usepackage{graphicx}
\usepackage{wrapfig}
\usepackage{amsthm}
\usepackage{microtype}
\usepackage{enumerate}
\usepackage{bbm}

\newcommand{\algcomment}[1]{\textcolor{blue!70!black}{\footnotesize{\texttt{\textbf{\% #1}}}}}

\usepackage{eucal}
\usepackage{mathrsfs}
\usepackage{bbm}
\pgfplotsset{width=10cm,compat=1.9}
\usepackage{natbib}
\usepackage{booktabs}
\usepackage{tablefootnote}
\usepackage{makecell}
\usepackage{hyperref}
\hypersetup{
    colorlinks=true,
    linkcolor=blue,
    citecolor=blue
}
\usepackage{cleveref}
\usepackage[noend]{algpseudocode}
  \usepackage{algorithm}

\numberwithin{equation}{section}
\newcommand\numberthis{\addtocounter{equation}{1}\tag{\theequation}}

\newtheorem{example}{Example}
\newtheorem{theorem}{Theorem}
\newtheorem{lemma}[theorem]{Lemma}
\newtheorem{claim}{Claim}
\newtheorem{assumption}{Assumption}

\newtheorem{proposition}[theorem]{Proposition}
\newtheorem{remark}{Remark}
\newtheorem{corollary}[theorem]{Corollary}
\newtheorem{definition}[theorem]{Definition}

\DeclareMathOperator*{\argmax}{arg\,max}

\newcommand{\nc}{\newcommand}
\nc{\ra}{\rightarrow}
\nc{\llin}{\ell^{{\rm lin}}}
\nc{\MN}{\mathcal{N}}

\nc{\erma}{\zeta}
\nc{\ep}{\epsilon}
\nc{\MX}{\mathcal{X}}
\nc{\BZ}{\mathbb{Z}}
\nc{\st}{\star}
\nc{\BR}{\mathbb{R}}
\nc{\MC}{\mathcal{C}}
\nc{\MZ}{\mathcal{Z}}
\nc{\close}[3]{\MC_{{#1},{#2}}({#3})}
\nc{\bv}{\mathbf{v}}
\nc{\bz}{\mathbf{z}}
\nc{\bw}{\mathbf{w}}
\nc{\bb}{\mathbf{b}}
\nc{\MA}{\mathcal{A}}
\nc{\ME}{\mathcal{E}}
\nc{\YY}{[-1,1]}
\nc{\CLS}{\textbf{Cls}}
\nc{\REG}{\textbf{Reg}}
\nc{\condt}{\ | \ }
\nc{\MI}{\mathcal{I}}

\DeclareMathOperator{\E}{\mathbb{E}}

\nc{\x}{{\mathbf{x}}}

\newcommand{\G}{\mathcal{G}}
\newcommand{\rr}{\mathbb{R}}
\newcommand{\pp}{\mathbb{P}}
\newcommand{\ee}{\mathbb{E}}

\newcommand{\R}{\mathcal{R}}
\newcommand{\F}{\mathcal{F}}

\newcommand{\norm}[1]{\left|\left| #1 \right|\right|}
\newcommand{\inprod}[2]{\left\langle #1 , #2 \right\rangle}
\newcommand{\abs}[1]{\left| #1 \right|}

\DeclareMathOperator*{\argmin}{argmin}
\newcommand{\yhat}{\hat{y}}

\DeclareMathOperator{\reg}{Reg}

\renewcommand{\epsilon}{\varepsilon}

\DeclareMathOperator{\sign}{sign}
\DeclareMathOperator{\diam}{diam}
\DeclareMathOperator{\vol}{vol}

\newcommand{\filt}{\scrF}
\newcommand{\N}{\bbN}
\newcommand{\Dists}{\triangle}
\newcommand{\I}{\bbI}
\newcommand{\Unif}{\mathrm{Unif}}

\newcommand{\rmd}{\mathrm{d}}
\newcommand{\rmD}{\mathrm{D}}

\newcommand{\Flin}{\F_{\mathrm{lin}}}
\newcommand{\Faff}{\F_{\mathrm{aff}}}
\newcommand{\algfont}[1]{\bm{\mathsf{#1}}}
\newcommand{\classify}{\algfont{classify}}
\newcommand{\errupdate}{\algfont{errorUpdate}}
\newcommand{\johnellipsoid}{\algfont{JohnEllpsoidCenter}}

\newcommand{\selfdot}{\algfont{self}.}
\newcommand{\cM}{\mathcal{M}}
\newcommand{\Abin}{\cA_{\mathrm{bin}}}

\newcommand{\sigdir}{\sigma_{\mathrm{dir}}}
\newcommand{\bst}{b^\star}
\newcommand{\wst}{w^\star}
\newcommand{\Nerr}{N_{\mathrm{err}}}

\newcommand{\pd}[1]{\mathbb{S}_{++}^{#1}}
\renewcommand{\ast}{\star}

\def\ddefloop#1{\ifx\ddefloop#1\else\ddef{#1}\expandafter\ddefloop\fi}
\def\ddef#1{\expandafter\def\csname bb#1\endcsname{\ensuremath{\mathbb{#1}}}}
\ddefloop ABCDEFGHIJKLMNOPQRSTUVWXYZ\ddefloop

\def\ddefloop#1{\ifx\ddefloop#1\else\ddef{#1}\expandafter\ddefloop\fi}
\def\ddef#1{\expandafter\def\csname fr#1\endcsname{\ensuremath{\mathfrak{#1}}}}
\ddefloop ABCDEFGHIJKLMNOPQRSTUVWXYZ\ddefloop

\def\ddefloop#1{\ifx\ddefloop#1\else\ddef{#1}\expandafter\ddefloop\fi}
\def\ddef#1{\expandafter\def\csname eul#1\endcsname{\ensuremath{\EuScript{#1}}}}
\ddefloop ABCDEFGHIJKLMNOPQRSTUVWXYZ\ddefloop

\def\ddefloop#1{\ifx\ddefloop#1\else\ddef{#1}\expandafter\ddefloop\fi}
\def\ddef#1{\expandafter\def\csname scr#1\endcsname{\ensuremath{\mathscr{#1}}}}
\ddefloop ABCDEFGHIJKLMNOPQRSTUVWXYZ\ddefloop

\def\ddefloop#1{\ifx\ddefloop#1\else\ddef{#1}\expandafter\ddefloop\fi}
\def\ddef#1{\expandafter\def\csname b#1\endcsname{\ensuremath{\mathbf{#1}}}}
\ddefloop ABCDEFGHIJKLMNOPQRSTUVWXYZ\ddefloop

\def\ddefloop#1{\ifx\ddefloop#1\else\ddef{#1}\expandafter\ddefloop\fi}
\def\ddef#1{\expandafter\def\csname bhat#1\endcsname{\ensuremath{\hat{\mathbf{#1}}}}}
\ddefloop ABCDEFGHIJKLMNOPQRSTUVWXYZ\ddefloop

\def\ddefloop#1{\ifx\ddefloop#1\else\ddef{#1}\expandafter\ddefloop\fi}
\def\ddef#1{\expandafter\def\csname btil#1\endcsname{\ensuremath{\tilde{\mathbf{#1}}}}}
\ddefloop ABCDEFGHIJKLMNOPQRSTUVWXYZ\ddefloop

\def\ddefloop#1{\ifx\ddefloop#1\else\ddef{#1}\expandafter\ddefloop\fi}
\def\ddef#1{\expandafter\def\csname bst#1\endcsname{\ensuremath{\mathbf{#1}^\star}}}
\ddefloop ABCDEFGHIJKLMNOPQRSTUVWXYZ\ddefloop

\def\ddef#1{\expandafter\def\csname c#1\endcsname{\ensuremath{\mathcal{#1}}}}
\ddefloop ABCDEFGHIJKLMNOPQRSTUVWXYZ\ddefloop
\def\ddef#1{\expandafter\def\csname h#1\endcsname{\ensuremath{\widehat{#1}}}}
\ddefloop ABCDEFGHIJKLMNOPQRSTUVWXYZ\ddefloop
\def\ddef#1{\expandafter\def\csname hc#1\endcsname{\ensuremath{\widehat{\mathcal{#1}}}}}
\ddefloop ABCDEFGHIJKLMNOPQRSTUVWXYZ\ddefloop
\def\ddef#1{\expandafter\def\csname t#1\endcsname{\ensuremath{\widetilde{#1}}}}
\ddefloop ABCDEFGHIJKLMNOPQRSTUVWXYZ\ddefloop
\def\ddef#1{\expandafter\def\csname tc#1\endcsname{\ensuremath{\widetilde{\mathcal{#1}}}}}
\ddefloop ABCDEFGHIJKLMNOPQRSTUVWXYZ\ddefloop
\newcommand{\ghat}{\hat{g}}
\newcommand{\gst}{{g}_{\star}}
\usepackage{bm}

\newcommand{\Exp}{\E}
\renewcommand{\Pr}{\mathbb{P}}

\newcommand{\ceil}[1]{\lceil #1 \rceil}
\newcommand{\floor}[1]{\lfloor #1 \rfloor}

\title{Efficient and Near-Optimal Smoothed Online Learning for  Generalized Linear Functions}

\title{Efficient and Near-Optimal Smoothed Online Learning for Generalized Linear Functions}
\author{Adam Block \\ MIT \and Max Simchowitz \\ MIT}
\date{}

\begin{document}
\maketitle
\begin{abstract}
    
%!TEX root = ../neurips_submission.tex

Due to the drastic gap in complexity between sequential and batch statistical learning, recent work has studied a smoothed  sequential learning setting, where Nature is constrained to select contexts with density bounded by $1/\sigma$ with respect to a known measure $\mu$. Unfortunately, for some function classes, there is an exponential gap between the statistically optimal regret and that which can be achieved efficiently.  In this paper, we give a computationally efficient algorithm that is the first to enjoy the statistically optimal $\log(T/\sigma)$ regret for realizable $K$-wise linear classification. We extend our results to settings where the true classifier is linear in an over-parameterized polynomial featurization of the contexts, as well as to a realizable piecewise-regression setting assuming access to an appropriate ERM oracle.  Somewhat surprisingly, standard disagreement-based analyses are insufficient to achieve regret logarithmic in $1/\sigma$.  Instead,  we develop a novel characterization of the geometry of the disagreement region induced by generalized linear classifiers. Along the way, we develop numerous technical tools of independent interest, including a general anti-concentration bound for the determinant of certain  matrix averages.

\end{abstract}
\tableofcontents
%!TEX root = ../neurips_submission.tex
\newcommand{\fst}{f^{\star}}
\newcommand{\poly}{\mathrm{poly}}
\section{Introduction}

In batch statistical learning, a learner faces a set of independent examples drawn from a given distribution, and is tasked with generalizing to novel examples drawn from that same distribution. In sequential or \emph{online} learning, however, Nature may adversarially select examples to thwart the learner's progress and success is defined only in comparison to the best a priori predictor.  Due to the wide range of application and minimal set of assumptions, online learning has received considerable recent attention.  For concreteness, consider binary classification, where a sequence of $T$ examples takes the form $(x_t,y_t) \in \R^d \times \{-1,+1\}$. Even in the \emph{realizable setting}, where there exists a true $\fst$ in a pre-specified class of functions $\cF$ for which $\fst(x_t) = y_t$ for all $t \in \{1,2,\dots,T\}$,   the gap between batch and statistical learning and sequential learning can be drastic: when $d = 1$, the class of linear thresholds $f_\theta(x) = \sign(x - \theta)$ has VC dimension one and is thus learnable in the PAC framework \citep{wainwright2019high}. A sequential adversary, however, can select $x_t$ so as to force the linear to misclassify $\Omega(T)$ points \citep{littlestone1988learning}. 

To circumvent the pessimism of the sequential setting, recent works \citep{rakhlin2011online,haghtalab2020smoothed,haghtalab2021smoothed,block2022smoothed,haghtalab2022oracle} have studied the \emph{smoothed sequential learning} paradigm, where the adversary is constrained to choose $x_t$ at random from any probability distribution $p_t$ with density at most $1/\sigma$ with respect to a known measure $\mu$. The most current of these results point to a striking statistical computational gap: whereas there exist algorithms which attain regret that scales with $\sqrt{T\log(/\sigma)}$, computationally efficient algorithms can only hope for $\poly(T/\sigma)$ regret in general, even against a realizable adversary \citep[Theorem 5.2]{haghtalab2022oracle}. In many  $d$-dimensional settings, natural choices of $\mu$ yield $\sigma = \exp(-\Omega(d))$, and thus the exponential separation in $\sigma$ translates into an exponential separation in dimension. This gap motivates the following question: can the statistical-computational gap be eliminated in more structured settings?  In this work, we answer the question affirmatively.

\paragraph{Contributions. } We show that for certain classes of realizable smoothed online classification problems, there exists a \emph{computationally efficient} algorithm which enjoys the statistically optimal $\log(T/\sigma)$ regret scaling, when the base measure $\mu$ is uniform on the unit-ball. Specifically, we provide computationally efficient algoirthms for achieving the statistically optimal regret bound for the following function classes:
\begin{itemize}
    \item For affine thresholds,  
    \item For affine thresholds in nonlinear features,
    \item For $K$-class affine classification, 
    \item For piecewise affine regression
\end{itemize} 
We also provide lower bounds that demonstrate the statistical optimality of our algorithms.  Furthermore, we apply our results to noiseless contextual bandits and get a fast algorithm that achieves optimal regret dependence on the horizon, up to logarithmic factors.  Finally, we present a complementary approach based on the perceptron algorithm which is robust to adversarial corruptions of the labels $y_t$, and  enjoys a polynomial regret in a ``directional smoothness'' parameter which interpolates between the $\log(1/\sigma)$-guarantees attained above in the realizable setting, and the $\poly(1/\sigma)$ bounds from prior work.  We emphasize that, though we adopt the smoothed online learning setting of \citet{rakhlin2011online,haghtalab2021smoothed,block2022smoothed}, we use entirely different techniques involving Ville's inequality \citep{ville1939etude}, geometric measure theory, and convex geometry. Moreover, in none of these works was the question of adapting to realizability explored; thus, we provide the first regret bounds that are \emph{logarithmic in both the horizon} and the smoothness parameter. We now discuss some related work:

\paragraph*{Online Learning.}  Extensions of classical learning theory to the online setting have proliferated due to the scope of application. Several works \citep{littlestone1988learning,blumer1989learnability,ben2009agnostic,rakhlin2015online} have explored the gap in statistical rates between classical and online learning settings, with \citet{littlestone1988learning,blumer1989learnability} showing that the class of one dimensional thresholds, which is easy to learn in the batch setting, is not learnable with adversarial data.  Other works, such as \citet{rakhlin2015online,rakhlin2013online,rakhlin2015sequential,block2021majorizing,rakhlin2014online} have provided sequential analogues of classical notions of complexity that characterize minimax regret, as well as providing computational separation between classical and online learning \citep{hazan2016computational}.  Due to the statistical and computational hardness results presented in the aforementioned work, there has been great interest in finding realistic, robust assumptions, such as smoothness, that allow for efficient learning.

\paragraph*{Smoothed Online Learning.}  Smoothed analysis was first proposed in \citet{spielman2004nearly} as a way to explain the success of the simplex algorithm of \citet{klee1972good} by combining the polynomial time bounds of an average-case analysis with the verisimilitude of a worst-case analysis.  Since then, smoothed analysis has been applied to explain the empirical success of many algorithms \citep{roughgarden2021beyond}.  In the learning setting, \citet{rakhlin2011online} proposed smoothed adversaries and proved regret bounds for linear thresholds in $\rr^d$; their proof, however, was nonconstructive and did not achieve logarithmic regret in the realizable setting.  The use of smoothed adversaries was essential due to the hardness results discussed above. In a series of works \citet{haghtalab2020smoothed,haghtalab2021smoothed} generalized \citet{rakhlin2011online} and showed the regret depending on the VC dimension was possible in the smoothed online learning setting, albeit with computationally \emph{inefficient} algorithms.

Recently, \citet{block2022smoothed,haghtalab2022oracle} generalized \citet{haghtalab2021smoothed} to allow for continuous labels and, more importantly, provided \emph{oracle-efficient} algorithms for achieving vanishing regret in the smoothed setting.  These papers also showed that the dependence on $\sigma$ in the regret bounds of their oracle-efficient algorithms, which was polynomial, could not in general be reduced to the logarithmic dependence achievable by the inefficient algorithms, thereby exposing a statistical-computational gap.  Unlike other recent works such as \citet{block2022smoothed,haghtalab2022oracle}, we do not use the coupling approach \citep{haghtalab2021smoothed} to prove our regret bounds.  

\paragraph*{Classification with Linear Thresholds.}  Considering the ubiquity of linear thresholds in classification, the list of relevant references is far too long to include here; as such, we highlight only those most germane to our work.  The perceptron algorithm was introduced in \cite{rosenblatt1958perceptron} and a margin-based mistake bound was proved in \citet{novikoff1963convergence}.  There have been many variations on and applications of this bound, from \citet{ben2009agnostic} using it to bound the Littlestone dimension of linear thresholds with margin to dealing with non-realizable samples \citep{crammer2006online,freund1999large}.  To the best of our knowledge our work constitutes the first to explore the effect that a smoothed adversary has on the perceptron algorithm.

\paragraph*{Disagreement Coefficient and Active Learning.} Intuitively, our analysis is similar to works in active learning based on the disagreement coefficient \citep{hanneke2007,hanneke2011rates,hanneke2014theory,wang2011smoothness}.  Indeed, as we shall see, our regret bounds arise by bounding the probability that a point falls into the disagreement region in a similar way as, for example, \citet{hanneke2007} controls the label complexity of active learning.  We will note in \Cref{rem:disagreement}, however, that an approach grounded purely in the disagreement coefficient cannot hope to achieve regret logarithmic in $\sigma$ in the smoothed setting.  Indeed, our approach incorporates a finer understanding of the geometry, accomodated by the more limited scope of application of our techniques, which allows us to prove tight rates.

In \Cref{sec:prelims}, we setup the learning problem and introduce some necessary notions from convex geometry, as well as fixing notation.  In \Cref{sec:masterthms}, we highlight two technical results that form the foundation of our approach, before, as a warmup, applying them to the case of classification with linear thresholds in \Cref{sec:warmup}.  In \Cref{sec:generalizations}, we generalize beyond linear thresholds to allow for offset and nonlinear features.  Finally, in \Cref{sec:kpiece}, we move beyond binary classification by extending our results to $K$-class affine classification, piecewise affine regression, and noiseless contextual bandits.
%!TEX root = ../neurips_submission.tex

\section{Preliminaries}\label{sec:prelims}

In this section, we provide basic definitions and setup the learning problem.  We begin by defining a smooth distribution, as in \citet{block2022smoothed,haghtalab2021smoothed}:
\begin{definition}\label{def:smootheddist}
    Let $\mu$ be a probability measure on a measurable space $\cX$.  For some $0 < \sigma \leq 1$, we say that a measure $p$ on $\cX$ is $\sigma$-smooth with respect to $\mu$ if the likelihood $\frac{\rmd p}{\rmd \mu} \leq \frac 1\sigma$ is uniformly bounded.
\end{definition}
We consider the smoothed online learning setting.  First, a horizon $T \in \mathbb{N}$ is fixed and a distribution $\mu$ on $\cX$ is chosen.  For each step $1 \leq t \leq T$, Nature chooses a distribution $p_t$, possibly depending on the history, such that $p_t$ is $\sigma$-smooth with respect to $\mu$ and samples $x_t \sim p_t$ as well as choosing some $y_t \in \cY$.  The learner sees $x_t$, chooses $\yhat_t$ and suffers loss $\ell(\yhat_t, y_t)$.  Given a function class $\F$ of functions mapping $\cX \to \cY$, the learner attempts to minimize regret, where regret is defined as:
\begin{equation}\label{eq:regretdef}
    \reg_T = \sum_{t = 1}^T \ell(\yhat_t, y_t) - \inf_{f \in \F} \sum_{t = 1}^T \ell(f(x_t), y_t)
\end{equation}
In the sequel, for the sake of simplicity, we take $\cX = \cB_1^d$ to be the unit ball, $\mu = \mu_d$ to be the uniform measure on $\cB_1^d$, and $\ell(\yhat_t,y_t) = \I(\yhat_t \ne y_t)$ to be the 0-1 loss. 

\begin{remark}[Scaling of $\sigma$]
    A natural example of a smoothed adversary is one that is allowed to place $\widehat{x}_t$ in a worst-case manner, which gets perturbed by some small additive noise, chosen uniform on $\epsilon \cdot \cB_1^d$, to become $x_t$; this adversary is $\sigma = \epsilon^d$ smooth.  For such situations, polynomial dependence on $\sigma$ in the regret translates into something exponential in dimension.
\end{remark}

\begin{remark}[Other measures $\mu$]
    Assuming the dominating measure $\mu = \mu_d$ is not overly strong: if $\mu$ is another measure on $\cB_1^d$ for which  $\frac{\rmd \mu}{\rmd \mu_d} \geq c > 0$, then, because our regret bounds are logarithmic in $\sigma$, our results will still hold with an additive term of $\log \left(\frac 1c\right)$.
\end{remark}

For much of the paper, we assume that Nature is \emph{realizable with respect to $\cF$}, i.e., for some $\fst \in \F$, $\fst(x_t) = y_t$ for all $1 \le t \le T$. In this case, $\reg_T$ is just a mistake bound: $\reg_T = \sum_{t = 1}^T \I\{\yhat_t\ne  y_t\}$.  The foundations of our analysis consider the class of linear threshold classifiers
\begin{align}
\Flin^d := \left\{x \mapsto \sign(\inprod{w}{x}) | w \in \cB_1^d\right\} \label{eq:Flin}
\end{align}
We identify $\Flin^d$ with the set of $w$'s defining it, so that we may treat it as, itself, a subset of $\cB_1^d$; other function classes are similarly identified with their parameters (without further comment).

At the core of our base algorithm is the computation of the \emph{John ellipsoid} \citep{john1948extremum,ball1997elementary}, the maximal volume ellipsoid contained in a convex body.\footnote{Some authors refer to the minimal volume ellipsoid \emph{containing} a convex body as the John ellipsoid.}  It is well-known that given a polytope in $\rr^d$, the John ellipsoid can be computed in time polynomial in $d$ and the number of faces \citep{boyd2004convex}.  In particular, we compute the John ellipsoid of the \emph{version space}, $\F_t$, where for any time $t$, we let $\F_t = \{f \in \Flin^d | f(x_s) = y_s \text{ for all } s < t\}$, which is a polytope with $t \le T$ faces.  An important concept in our analysis is the notion of \emph{Hausdorff measure}, which generalizes the standard notions of volume and surface area in $\rr^d$; we will denote the $k$-dimensional Hausdorff measure (see \Cref{def:hausdorff}) by $\vol_k(\cdot)$.  More detail on both the John ellipsoid and the Hausdorff measure can be found in \Cref{app:prelims}.

\paragraph{Notation.} For a set $\cU \subset \rr^d$, we denote by $\partial \cU$ its boundary.  We let $\cB_r^d$ denote the ball of radius $r$ around the origin in $\rr^d$ and let $S^{d-1} = \partial B_1^d$.  Letting $\Gamma$ denote the $\Gamma$-function, let $\omega_d = \frac{\pi^{d/2}}{\Gamma(d/2 + 1)}$ denote the volume of $\cB_1^d$ and let $\mu_d$ denote the uniform measure on $\cB_1^d$ normalized to be a probability measure.  If $\phi:\rr^n \to \rr^m$ is Lipschitz, we denote the Jacobian by $D \phi$. Lastly, we use ``$\lesssim$'' to denote inequality up to universal, problem-independent constants.

%!TEX root = ../neurips_submission.tex

\section{The Technical Workhorses}\label{sec:masterthms}
In this section we introduce the two key workhorse results that provide the technical foundation for the rest of the paper.  The first result is a purely probabilistic statement that we use as a blackbox throughout the paper to turn probabilistic and geometric theorems into regret bounds in the realizable, smoothed online learning setting.  The second result is a geometric statement that allows us to apply the black box regret bound to the case of classification with affine thresholds.
\subsection{An Abstract Decay Analysis}
We begin with an abstract, technical result that will form the basis for all of our regret bounds.  We first introduce the following definition:
\begin{definition}
    Let $\mu$ be a measure on some set $\cZ$ and let $\ell_t: \cZ \to \{0,1\}$ be a sequence of loss functions.  For $R >0$ and $0 < c < 1$, we say that the sequence $(\ell_t, z_t)$ satisfies $(R, c)$-geometric decay with respect to $\mu$ if there exists a sequence of nonnegative numbers $R_t$ with $R_1 = R$ satisfying the following two properties:
    \begin{enumerate}
        \item For all $t$, $\mu\left(\{z : \ell_t(z) = 1\}\right) \leq R_t$.
        \item For any $t$ such that $\ell_t(z_t) = 1$, we have $R_{t+1} \leq c R_t$.
    \end{enumerate}
\end{definition}
To motivate this admittedly abstract definition, consider the case of online classification with thresholds $f_\theta(x) = \sign(x - \theta)$ from the introduction, with $\mu$ uniform on $[0,1] \times\{\pm 1\}$ (note that this does not precisely fit into the linear setting described above due to the offset); take $z_t = (x_t, y_t)$ and $\ell_t(z_t) = \bbI[\yhat_T \neq y_t]$, where the learner predicts $\yhat_t$ at each time $t$.  By realizability,  $\ell_t(z) = 1$ only when $x_t$ falls in the ``region of disagreement,''i.e. the interval the rightmost $x_s$ labelled $-1$ and the leftmost $x_s$ labelled $1$.  To see why this is true, note that the ``version space,'' i.e., the set of thresholds that correctly classify all the data so far, is exactly this interval; for us to make a mistake, there must be two functions in the version space that disagree on $x_t$, which can only happen if $x_t$ itself is in the version space.  If the learner denotes by $w_t$ the midpoint of the region of disagreement, then any mistake forces the version space, and thus the disagreement region, to shrink by a factor of 2.  We see then that $(\ell_t, z_t)$ satisfy $\left(1, \frac 12\right)$-geometric decay with respect to the uniform measure.

%At present, there is no apparent connection between the notion of geometric decay and a regret bound as, in the above example, the adversary is still allowed to choose $x_t$ in the region of disagreement for all $t$ and thus make mistakes likely.  On the other hand, 

If the adversary were constrained to choose $x_t \sim \mu$ at each time step, it is intuitive that we should not expect many mistakes to be made because, after any mistake, the probability that we make a mistake in some future interval decreases.  In the following result, we show that this intuition holds in the more general smoothed setting:
\begin{lemma}[Abstract Decay Lemma]\label{lem:masterreduction}
    Suppose that a sequence $(\ell_t,z_t) $ satisfies $(R, c)$-geometric decay with respect to some $\mu$ on $\cZ$, and that for all $t$, there is some $p_t$ that is $\sigma$-smooth with respect to $\mu$ and $z_t \sim p_t$.  Then for all $T \in \bbN$, with probability at least\footnote{Here, as in the rest of the paper, we made no effort to optimize constants.  We include them only to demonstrate that they are not unreasonably large.} $1 - \delta$, 
    \begin{equation}
        \sum_{t=1}^T \ell_t(z_t) \leq 4 \frac{\log\left(\frac{2 T R}{\sigma \delta}\right)}{\log\left(\frac 1c\right)} + \frac{e - 1}{1 - \sqrt[]{c}}.
    \end{equation}
\end{lemma}
\begin{proof}[Proof Sketch]
    We break our analysis into epochs whose lengths $h_m$ are tuned at the end of the proof.  We then consider a sequence of stopping times $\tau_m$ that count the number of epochs of length $h_m$ we experience in between the $(m-1)^{st}$ and $m^{th}$ time that $\ell_t = 1$.  We then show that if $h_m$ is not too large relative to the Probability that $\ell_t = 1$, then $\tau_m - \tau_{m-1}$ is large with high probability and apply Ville's inequality \citep{ville1939etude} to conclude that if $m_T$ is the maximal epoch-index $m$ such that $\tau_m \leq T$, then $m_T$ cannot be too large.  We again apply Ville's inequality to show that if $h_m$ is not too large then the probability of multiple mistakes per epoch is small.  Because of the geometric decay property, the probability that $\ell_t = 1$ decreases exponentially in the number of mistakes and thus we may let $h_m$ grow exponentially in $m$ and still not be too large to apply the above argument.  We then conclude by noting that if $h_m$ are growing exponentially in $m$ then $m_T$ has to be logarithmic in $T$.  The details can be found in \Cref{app:masterreduction}.
\end{proof}
If we return to the above example of online classification with thresholds, we see that \Cref{lem:masterreduction} immediately yields the first regret bound for realizable, smoothed online learning with thresholds that is logarithmic both in the horizon $T$ and the smoothness parameter $\sigma$.  The intuition gleaned from one-dimensional thresholds that geometric decay suffices to ensure logarithmic regret will be key to the more general regret bounds we exhibit below.

\subsection{A Volumetric Lemma}\label{sec:orthogonality}
In the previous section, we saw that in the setting of realizable, smoothed online classification with one-dimensional thresholds, the learner can force the indicator of a mistake at time $t$ to satisfy geometric decay; our second workhorse result will allow us to extend this fact to higher dimensions.  In the case of thresholds in the unit interval, the key intuition leading to geometric decay was the fact that the disagreement region was exactly the version space and thus shrinking the version space tautologically shrank the disagreement region as well.  In higher dimensions the situation is significantly more complicated.  We have the following result:

%This relationship is important because it is intuitive that the size of the version space should decrease with the number of points observed and so bounding the size of the disagreement region by that of the version space will in turn control the probability of a mistake.  While in one dimension, the relationship between disagreement region and version space is so simple as to be barely worth mentioning, in higher dimensions the situation is significantly more complicated.  We have the following result:
\begin{lemma} \label{lem:orthogonality}
    Let $x_1, \dots, x_t \in \cB_1^d$ and suppose that $y_1, \dots, y_t$ are realizable with respect to $\Flin^d$.  Define the disagreement region
    \begin{align}\label{eq:Dt}
        D_t := \left\{x \in \cB_1^d \mid \text{there exist } f, f' \in \F_t \text{ such that } f(x) \neq f'(x)\right\}
    \end{align}
    where $\F_t$ is the version space, defined in \Cref{sec:prelims}.  Then, recalling that $\partial \F_t$ is the boundary of $\F_t$,
    \begin{equation}\label{eq:orthogonality}
        \mu_d(D_t) \leq 2 \cdot 4^{d-1} \mu_d(\F_t) + \frac{4^{d+1}}{\omega_d} \vol_{d-1}(\partial \F_t).
    \end{equation}
    
\end{lemma}
Note that by controlling the size of $D_t$ by that of $\F_t$, \Cref{lem:orthogonality} is a direct generalization of the one-dimensional case; however, in contradistinction to that setting, the proof is much more difficult and the bound includes an extra term corresponding to the surface area of $\F_t$, which is unavoidable in general. The full proof is in \Cref{app:orthogonality}, but we summarize the key points here.  Though the conclusion of \Cref{lem:orthogonality} is intuitive, it requires significant technical effort to prove.

\begin{proof}[Proof Sketch of \Cref{lem:orthogonality}]   We first note that $D_t$ is contained in the set of points $x$ such that there is some $w \in \F_t$ with $\inprod{w}{x} = 0$; thus the conclusion of \Cref{lem:orthogonality} reduces to a geometric statement about the volume of the set of points orthogonal to at least one point in a given set can be.  It may seem like this should ``obviously'' be the volume of a $(d-1)$-dimensional ball multiplied by the volume of $\F_t$, but this is false:  if $\F_t$ is the equator of the sphere $S^{d-1}$, then $\mu_d(\F_t) = 0$, but the set of points orthogonal to at least one point in $\F_t$ is the entirety of $\cB_1^d$. 
%, 
Ruling out this pathology requires several steps, including a covering argument to reduce to the case where $\F_t$ is a ball, and application of (a generalized) Steiner's formula, and  a deep geometric fact called Weyl's Tube Formula \citep{weyl1939volume,gray2003tubes} that governs how much volume we can add to $\F_t$ by ``fattening'' to include all points distance at most $\epsilon$ from $\F_t$. 
%we require a term controlling the size of the boundary of $\F_t$ is necessary. To do so, construct a covering of $\hat{F}_t = \F_t \cap S^{d-1}$ (restricting to normalized linear classifiers by positive homogeneity). 
%we invoke We will use the formula to show that the $\mu$-measure of the fattened $\F_t$ is controlled by $\mu(\F_t)$ and $\vol_{d-1}(\partial\F_t)$ and then apply a covering argument to reduce to the case where $\F_t$ is a ball, where we again apply Weyl's Tube formula to conclude the proof in that setting. 
\end{proof}

%!TEX root = ../neurips_submission.tex

\section{Warmup with Linear Classification}\label{sec:warmup}
In this section, we begin to apply our results from \Cref{sec:masterthms} to get tight regret bounds with computationally efficient algorithms for learning halfspaces in the realizable, smoothed online setting:

  \begin{algorithm}[!t]
    \begin{algorithmic}[1]
    \State{}\textbf{Initialize } $\cW_1 = \cB_1^d$, $w_1 = \mathbf{e_1}$
    \For{$t=1,2,\dots$}
    \State{}\textbf{Recieve } $x_t$, and \textbf{predict}  $\yhat_t = \sign(\inprod{w_t}{x_t})$, \qquad\qquad\qquad\quad(\algcomment{$\selfdot\classify(x_t)$})
    \State{}\textbf{Update} $\cW_{t+1} = \cW_t \cap \{w \in \cB_1^d| \inprod{w}{x_t y_t} \geq 0 \}$
    \If{$\yhat_t \ne y_t$} \quad\qquad\qquad\qquad\qquad\qquad\qquad\qquad\qquad\qquad(\algcomment{$\selfdot\errupdate(x_t)$})
    \State{} $w_{t+1} \gets \johnellipsoid(\cW_{t+1})$ \\
    \qquad\qquad \algcomment{returns center of John Ellpsoid of given convex body}
    \EndIf
    \EndFor
    \end{algorithmic}
      \caption{Binary Classification with Linear Thresholds}
      \label{alg:warmup}
    \end{algorithm}

\begin{theorem}\label{prop:warmup}
    Let $\mu$ be the uniform measure on $\cB_1$.  Suppose that we are in the smoothed, realizable online learning setting, where the adversary samples $x_t$ from a distribution that is $\sigma$-smooth with respect to $\mu$.  
    % At each time $t$, let $\F_t$ be as in \Cref{lem:orthogonality} and let $w_t$ be the center of $\cE_t$, the John ellipsoid of $\F_t$.  
    If we predict $\yhat_t$ according to \Cref{alg:warmup}, then for all horizons $T$, with probability at least $1 - \delta$,
    \begin{equation}
        \reg_T \leq 136 d \log(d) + 34 \log\left(\frac{T}{\sigma \delta}\right) + 56.
    \end{equation}
\end{theorem}
\textbf{Computational Efficiency.} The subroutine $\johnellipsoid(\cW_{t+1})$ can be run in time polynomial in $T$ and $d$ by solving a Semi-definite Program (SDP) \citep{boyd2004convex,primak1995modification}. %, which in turn, can be viewed as the output of an ERM-oracle \ac{makse sure erm oracle definition includes this}; thus, the method in Proposition \ref{prop:warmup} is oracle-efficient.  
Note that we change our predictor $f_t$ only at the times $t$ that we make a mistake; thus, the number of calls to the SDP is also logarithmic in $T$.%, making our algorithm highly efficient.%, in direct contradistinction to the algorithms of \citet{block2022smoothed} in the nonrealizable setting, which assumed that $f_t$ changes every round.

%If we make a mistake at time $t$, then $\inprod{w_t}{y_t x_t} < 0$ and thus $\F_{t+1} \subset \F_t \cap \left\{w \in \F | \inprod{w}{y_t x_t} \geq \inprod{w_t}{y_t x_t} \right\}$ 
%is the intersection of $\F_t$ with a halfspace through $w_t$. In light of \Cref{lem:orthogonality}, we want both $\F_t$ and its boundary $\partial \F_t$ to shrink. Conveniently, choosing $w_t$ as the center of the John ellipsoid enjoys this property. 

\begin{proof}[Proof Sketch of \Cref{prop:warmup}] We apply \Cref{lem:masterreduction} with $z_t = (x_t, y_t)$ and $\ell_t(z) = \bbI\left[\yhat_t \neq y_t\right]$.  In order to do this we need to show that $\ell_t$ satisfies $(R, c)$ geometric decay, which amounts to finding a geometrically decreasing sequence of upper bounds on $\mu(D_t)$.  By \Cref{lem:orthogonality}, it will suffice to provide such bounds on both $\mu(\F_t)$ and $\vol_{d-1}(\partial \F_t)$, which is where the specific choice of $w_t$ becomes important.  It is now classical \citep{tarasov1988method,khachiyan1990inequality} that if a polytope is cut by a hyperplane through the center of its John ellipsoid then both halfs have John ellipsoids whose volumes are at most $\frac 89$ times the volume of the original ellipsoid; as we know that $\F_t \subset d \cdot \cE_t$ \citep{john1948extremum}, where $\cE_t$ is the John ellipsoid of $\F_t$, we see that $\mu(d \cdot \cE_t)$ is an upper bound on $\mu(\F_t)$ that decreases by $\frac 89$ every time we make a mistake.  The true utility of the center of the John ellipsoid is that it also allows us to show that $\partial \F_t$ decreases by a constant factor.  Indeed, we show that $\vol_{d-1}(\partial\F_t) \leq \vol_{d-1}(\partial \cE_t)$ using a simple projection argument; we then apply a result of \citet{rivin2007surface} to bound the size of $\partial \cE_t$ by $\mu(\cE_t)$.  The details are in \Cref{app:warmup}.
\end{proof}

\textbf{Importance of the John's Ellipsoid.}  We show in  \Cref{app:naive_lb} that arbitrary predictions $y_t = \tilde{f}_t(x_t)$, for $\tilde{f}_t \in \F_t$ in the version space, can guarantee $1/\sigma$-regret at best. Hence, selecting the correct $w_t$ is key. One natural choice of $w_t$ is the Chebyshev center of $\F_t$ \citep{elzinga1975central}, equivalent to a max-margin estimator; unfortunately it need not decrease the volume sufficiently if $\F_t$ is too `pointy.'  Another choice, the centroid of $\F_t$, ensures decrease of the \emph{polytope}'s volume, but is \#P-hard to compute \citep{rademacher2007approximating}, and does not ensure decay of the surface area.  The former problem can be accomodated with a sampling scheme \citep{bertsimas2004solving}, but the latter is critical.  In contrast, the center of the John ellipsoid controls the decay of both $\F_t$ and its boundary.

\begin{remark}[\emph{Disagreement Coefficient}]\label{rem:disagreement} \emph{Our analysis is similar in spirit to the disagreement-coeffcient analysis of active learning \citep{hanneke2007}, which also exhibits geometric decay of the disagreement region $D_t$. The key difference is that the latter applies to \emph{any algorithm} that selects a classifier from the version space $\cF_t$ at each time $t$. Again, as shown in \Cref{app:naive_lb},  no such analysis can recover a better than $1/\sigma$-regret bound. The culprit is that disagreement-coefficient arguments ensure that $D_t$ shrink only \emph{probabilistically} under samples $x_t \sim \mu$, and this probability may shrink by a factor of $\sigma$ in the smoothed-online setting. In contrast, our choice of classifier as the center of the John's ellipsoid ensures a \emph{deterministic} decay of the disagreement region whenever a mistake is made.} 
%We provide a detailed comparison to this approach in Appendix ??. What we find is that a naive adaptation of that analysis to our setting would yield regret of $\frac{d^{3/2}\log T}{\sigma}$, which is both suboptimal in $d$, but, more strikingly, polynomial in $1/\sigma$. Note that the disagreement-based analysis analyzes \emph{arbitrary} classifiers $w_t$ in the version space. Our use of the John's ellipsoid forces the classifier to be chosen to decrease the version space (and its boundary) as agressively as possible. 
\end{remark} 

%Therefore, if we want the $\F_{t+1}$ to be smaller than $\F_t$, we should choose $w_t$ to the ``center'' of $\F_t$. The John's ellip

 %the question remains, which center should we choose?  One natural choice would be the max-margin estimator, ; unfortunately, volume reduction is not ensured.  Another option without that problem is the centroid of $\F_t$, as Gr{\"u}nbaum's inequality \citep{grunbaum1960partitions} requires the volume of $\F_t$ to decrease by a constant factor; unfortunately, even if it were not \#P-hard to find the centroid of $\F_t$ \citep{rademacher2007approximating} and we were to instead use a randomized algorithm akin to that of \citet{bertsimas2004solving}, in order to apply \Cref{lem:orthogonality}, we also need the measure of $\partial \F_t$ to decrease.  For this reason, motivated by the ellipsoid algorithm \citep{shor1977cut,khachiyan1979polynomial}, we instead choose $w_t$ to be the center of the John ellipsoid, i.e., the maximum volume ellipsoid inscribed in $\F_t$\footnote{Note that the ellipsoid algorithm is run with the minimum volume ellipsoid containing the convex set.  While our analysis would still apply if we were to instead choose $w_t$ the center of that ellipsoid, for technical reasons we would lose a factor of $d$ in the regret bound.}
%\mscomment{as opposed to centroid and max margin}

\paragraph{Lower Bound.}Before we move on to the more complicated settings, we note that this regret bound is tight up to a logarithmic factor in $d$.  A proof of the following proposition, based on Ville's inequality, can be found in \Cref{app:warmup}.
\begin{proposition}\label{prop:warmuplowerbound}
    Suppose that we are in the situation of \Cref{prop:warmup}.  Then there is a realizable adversary such that any classifier experiences
    \begin{equation}
        \ee\left[\reg_T\right] \geq \Omega\left(d + \log\left(T/\sigma\right)\right).
    \end{equation}
\end{proposition}
%Comparing the results of \Cref{prop:warmup,prop:warmuplowerbound}, we see that our upper bound is tight up to a factor logarithmic in $d$, which we believe comes from a slight looseness in our application of \Cref{lem:orthogonality}.  

\subsection{Smoothed classification via the Perceptron algorithm}

Next, we present a guarantee for the classical Perceptron algorithm \cite{rosenblatt1958perceptron}, which requires a much weaker notion of smoothness. We say that the adversary satisfies $\sigdir$ directional smoothness if, for any fixed $w \in S^{d-1}$, it holds that for all $t$,  $\langle x_t, w \rangle$ is $\sigdir$-smooth with respect to the Lebesgue measure on the real line. As we explain in \Cref{example:random_noise} in \Cref{app:perceptron}, the directional smoothness $\sigdir$ can be nontrivial even when the smoothness parameter $\sigma = 0$. We now show that the perceptron satisfies the following mistake bound under directional smoothness. 
\begin{theorem}\label{thm:Tsyb_body} Fix any $w^\star \in S^{d-1}$ and $b^\star \in \rr$. And suppose that the adversary satisfies $\sigdir$-directional smoothness.  Then, with probability $1-\delta$, the online Perceptron (\Cref{alg:perceptron} in \Cref{app:perceptron}) satisfies
\begin{align*}
\reg_T = {\textstyle\sum_{t=1}^T \I\{\yhat_t \ne y_t\}}  \lesssim ( T/\sigdir)^{\frac{2}{3}} \cdot (\Nerr(\wst,\bst) )^{\frac{1}{3}} + \log(\ceil{\log T }/\delta), 
\end{align*}
where $\Nerr(\wst,\bst) = 1 + \sum_{i=1}^T \I\{y_t \ne \sign(\bst + \langle \wst,x_t \rangle)\}$ controls deviation from realizability.
\end{theorem}
For simplicity, \Cref{thm:Tsyb_body} is stated relative to a \emph{fixed} $w^\star \in S^{d-1}$ and $b^\star \in \rr$; \emph{uniform} bounds can be derived via a covering argument, at the expense of an additive $d\log(T/\delta\sigdir)$ term in the error bound. Unlike other algorithms proposed in this paper, \Cref{thm:Tsyb_body} accomodates possibly non-realizable adversaries. It is also slightly more computationally expedient, not requiring the computation of the center of a John's ellipsoid. In contrast, its bound is polynomial in $T$ and $1/\sigdir$, rather than logarithmic in $T$ and $1/\sigma$.  There are situations where \Cref{alg:warmup} performs exponentially better than the Perceptron approach:  suppose $x_t$ is  uniform on an $\epsilon$-ball whose center is chosen by the adversary.  Then we have $\sigma = \epsilon^{-d}$ and so \Cref{prop:warmup} implies that the John ellipsoid approach gives regret that scales as $O(d \log(d/\epsilon) + \log(T))$, whereas $\sigdir \approx 1/\epsilon$ and so \Cref{thm:Tsyb_body} only ensures regret that is polynomial in $\epsilon$.  For further comparison, consult \Cref{rem:compare_perceptron} in \Cref{app:perceptron}.

\section{Beyond the Linear Case}\label{sec:generalizations}
While the results of \Cref{sec:warmup} are technically interesting and have broad applications, they are limited to the specific linear setting.  In this section, we show how our results can be extended, first to the more general affine setting, where the decision boundaries do not have to go through the origin, and then to a more general regime where we do not require linear decision boundaries.
\subsection{Affine Classification}
Our first generalization of \Cref{prop:warmup} is to the setting where we allow our decision boundaries to be offset.  Thus instead of assuming realizability with respect to $\Flin^d$, we will assume that the adversary isrealizable with respect to
\begin{equation}
    \Faff^d = \left\{x \mapsto \sign(\inprod{w}{x} + b) | w \in \cB_1^d \text{ and } b \in \rr\right\} \label{eq:Faffdef}.
\end{equation}
We have the following result:
\begin{corollary}\label{prop:affinereduction}
    Let $\mu$ be the uniform measure on $\cB_1^d$ and suppose that we are in the smoothed online learning setting, where the adversary samples $x_t$ from a distribution that is $\sigma$-smooth with respect to $\mu$.  Suppose that the adversary is realizable with respect to the function class  $\Faff^d$ defined in \eqref{eq:Faffdef}.  Then \Cref{alg:affine} in \Cref{app:affine} is a computationally efficient algorithm for choosing $f_t \in \Faff^d$ such that for all $T$, with probability at least $1 - \delta$, it holds that
    \begin{equation*}
        \reg_T \leq 268 d \log(d) + 34 \log\left( T / (\sigma \delta)\right) + 56.
    \end{equation*}
\end{corollary}
As $\Flin^d \subset \Faff^d$, the lower bound of \Cref{prop:warmuplowerbound} holds and \Cref{prop:affinereduction} is tight up to a factor logarithmic in dimension.  The proof is given in \Cref{app:affine} and proceeds by reducing to the linear setting of \Cref{prop:warmup} by imbedding the problem into an online learning problem with contexts $\tilde x_t \in \rr^{d+1}$, carefully randomized so as to preserve their smoothness with respect to $\mu_{d+1}$.

\subsection{Linear Classification Under a Feature map}
One limitation of the above discussion has been the assumption of linearity, which can be overly strong in many cases.  In this section, we weaken this assumption in two ways.  First, we show that if we transform the features with a well-behaved function, then we may still apply our above machinery.  Second, we will show that our approach actually generalizes to polynomial decision boundaries through an elegant reduction.  In both cases, the key technical challenge is to show that our transformed features remain smooth with respect to the uniform measure on a ball.  Note that it is immediate that $\phi(x_t)$ is smooth with respect to $\phi_\ast \mu_d$; in order to apply our results, however, we require smoothness with respect to the uniform measure.  As it is not true that $\phi_\ast \mu_d$ is smooth with respect to $\mu_d$ for general $\phi$, we require additional assumptions.  We have the following result:
\begin{theorem}\label{prop:samedimensionfeatures}
    Let $\phi: \cB_1^d \to \cB_1^d$ be a function such that each coordinate function, $\phi_i: \rr \to \rr$ satisfies $\phi_i'(u) \geq \alpha$ for some $\alpha > 0$.  If we run \Cref{alg:samedim} in \Cref{app:samedim} then, for all $T$, with probability at least $1 - \delta$, it holds that
    \begin{equation*}
        \reg_T \leq 136 d \log\left( d/\alpha\right) + 34 \log\left(T/(\sigma \delta)\right) + 56. 
    \end{equation*}
\end{theorem}
\Cref{alg:samedim}, the algorithm that achieves the above regret bound, is actually quite simple as it just runs \Cref{alg:warmup} on the data sequence $(\phi(x_t), y_t)$.  A proof of a more general result, which applies to a larger class of maps $\phi$, is available in \Cref{app:samedim}.  Even in the setting of \Cref{prop:samedimensionfeatures}, though, standard transformations like the sigmoid already apply.  

We now turn to the more challenging case of polynomial features.  We have the following result:
\begin{theorem}\label{prop:featuremap}
    Let $\phi: \cB_1^d \to \cB_1^m$ be an $L$-Lipschitz function whose coordinates are polynomials of degree at most $\ell$ in the coordinates of $x \in \cB_1^d$.  Suppose that we are in the smoothed online learning setting where the $x_t$ are $\sigma$-smooth with respect to $\mu_d$ and the $y_t$ are realizable with respect to $\Flin^m \circ \phi$.  Suppose further that the Jacobian of $\phi$ satisfies for some $\alpha > 0$,
    \begin{equation*}
        \det\left(\ee_{x \sim \mu_d}\left[D \phi(x) D \phi(x)^T\right]\right) \geq \alpha ^2.
    \end{equation*}
    Then \Cref{alg:polynomials} in \Cref{app:polies} is a computationally efficient algorithm such that for all $T$, with probability at least $1 - \delta$,
    \begin{equation*}
        \reg_T \lesssim  m \log(m) + \log\left(\frac 1\alpha\right) + \ell^2 m^2 d \log^2\left(\frac{d \ell T L}{\sigma \delta}\right).
    \end{equation*}
\end{theorem}
\Cref{alg:polynomials} is a bit more complicated than simply applying \Cref{alg:warmup} to $(\phi(x_t), y_t)$ because if $d \leq m$, then $\phi(x_t)$ can never be smooth with respect to $\mu_m$ by dimension constraints.  To escape this difficulty, we define a ``meta-point,'' $\overline{x}_\tau$, which is the average of $\phi(x_t)$ for multiple different $t$.  To understand why this might fix the problem, consider the identity imbedding of $S^{d-1} \subset \cB_1^d$: if we sample $x$ uniformly on $S^{d-1}$, then the law of $x$ will not even be absolutely continuous with respect to $\mu_d$ but if we sample two points $x, x' \sim S^{d-1}$ then their average \emph{is} absolutely continuous with respect to $\mu_d$.  We note that the conditions on $\phi$ are fairly mild due to the logarithmic dependence on both the Lipschitz constant and the lower bound on the determinant, which is typically no less than exponentially small in $d$ and $m$.
\begin{proof}[Proof Sketch of \Cref{prop:featuremap}]
    \Cref{alg:polynomials} proceeds initially in a similar way to \Cref{alg:warmup}: we maintain a version space $\F_t \subset \Flin^m$ that gets updated every round and, when we change $w_t$, we set it to be the center of the John ellipsoid of the version space.  In contradistinction to the earlier algorithm, however, we do not update $w_t$ every time we make a mistake.  Instead, for some parameter $p$, we wait until we have misclassified a label $p$ times, i.e., we guessed $-1$ but $y_t$ was 1 $p$ times (or the reverse) and construct $\overline{x}_\tau$ to be the average of the $\phi(x_t)$ for each of these $p$ mistakes.  Using a novel anti-concentration bound for determinants of certain random matrices (\Cref{prop:anti_conc}) as well as some techniques from geometric measure theory (\Cref{prop:imbeddingissmooth}), we show that $\overline{x}_\tau$ is smooth with respect to $\mu_m$.  We then apply the abstract decay lemma (\Cref{lem:masterreduction}) in much the same way as we did in the proof of \Cref{prop:warmup}.  The details are in \Cref{app:polies}.
\end{proof}

%!TEX root = ../neurips_submission.tex

\section{Beyond Binary Classification}\label{sec:kpiece}
In the previous sections, we restricted our focus to binary classification; in this section we expand our scope to a $K$-class setting and then further extend to a regression setting.  Our results for the regression setting, combined with the reduction of \citet{foster2020beyond}, are applied to the setting of contextual bandits in \Cref{sec:conbandits}.

\newcommand{\Flink}{\F_{K\text{-}\mathrm{lin}}}
\subsection{Multi-Class Classification}
We first generalize our results to multi-class classification.  The targets are  $y_t \in [K]$  some fixed $K$ and classifications are assigned by maximum inner-product:\footnote{For simplicity, we interpret the $\argmax$ lexicographically.  }%Observe that when all $(w^1, \dots, w^K)$ are distinct, and when $x$ is drawn from a distribution with density with respect to the Lebesgue measure, which $\sigma$ smoothness ensures, then, the $\argmax_{1 \leq i \leq K}$ is unique with probability 1. Thus, for simplicity, we interpret the argmax as lexicographic (and thus unique), which is valid  whenever the ground truth $\fst$ to has distinct $(w^1, \dots, w^K)$.  }
\begin{equation}\label{eq:kpiececlass}
    \Flink^d = \{x \mapsto f_{\mathbf w}(x) = \argmax_{1 \leq i \leq K} \inprod{w^i}{x} \mid \mathbf{w} = (w^1, \dots, w^K) \in \left(\cB_1^d \right)^K  \}.
\end{equation}
Our algorithm is a direct reduction to binary classification. For each $i<j$, we maintain an instance $\Abin^{(i,j)}$ of \Cref{alg:warmup} which makes binary predictions $\yhat_t^{(i,j)}$ of $y_t^{(i,j)} = \sign (\langle w_\star^i - w_{\star}^j,x_t \rangle)$. We then set our $K$-class prediction $\yhat_t$ as the first index $i$ for which $\yhat_t^{(i,j)} = 1$ for all $j > i$.
The key insight is that, even though the learner does not recieve  feedback on \emph{all} $\yhat_t^{(i,j)}$ in this way, we can always assign a mistake $\yhat_t = y_t$ to an  error $y_t^{(i,j)} \ne \yhat_t^{(i,j)}$ for \emph{some} $i<j$. Formal pseudocode is given \Cref{alg:K_class} and a proof of the following regret bound is given \Cref{app:kpiececlassification}.
\begin{theorem}\label{prop:kpiececlassification}
    Suppose we are in the realizable, smoothed, online learning setting where the adversary is realizable with respect to the $\Flink^d$ in \eqref{eq:kpiececlass}.  Then, then for all $T$, with probability at least $1 - \delta$, the regret of \Cref{alg:K_class} is at most
    \begin{equation}
        \reg_T \leq  136 K^2 d \log(d) + 91K^2 \log\left(TK^2/(\sigma \delta)\right).
    \end{equation}
\end{theorem}
The efficiency of the above algorithm follows from the efficiency of the binary classifers $\Abin^{(i,j)}$. We conjecture that the dependence on $K^2$ is an artifact of our reduction to $\binom{K}{2}$ base classifiers. 

%Before proceeding to the regression setting, we remark that finding $\mathbf{w}_t$ remains efficient for the same reason that the linear predictor from \Cref{sec:warmup} was.  Indeed, we have that $\overline{\F}_t^{(i,j)}$ is a polytope defined by at most $K^2 t$ halfspaces and so finding the center of the John ellipsoid is efficient; to find $\mathbf{w} \in \F_t$, we simply need to find a solution to a linear system that lies in a polytope, which remains efficient.

\subsection{Piecewise Regression}
%In the previous section, we explained how our techniques generalize to $K$-class classification.  
This section extends $K$-class classification to piecewise affine regression.  We now suppose that the targets $y_t$ are real-valued, and realizable with respect to the following class of functions:
\begin{equation}\label{eq:kpieceregclass}
    \G_\F = \left\{x \mapsto g_f(x) = {\textstyle\sum_{i = 1}^K} g_i(x) \bbI[f(x) = i] \bigg| g_i (x) = \inprod{a_i}{x} \text{ for } a_i \in \rr^{d} \text{ and } f \in \Flink^d\right\}.
\end{equation}
 In contradistinction to the rest of the paper, where the adversary is allowed to play the $y_t$ adaptively subject only to the condition of realizability, in this section we suppose that the adversary is \emph{semi-oblivious} in the sense that there is a ground-truth function chosen before the start of play and after learning begins, the adversary is only allowed to choose the contexts, $x_t$.  This assumption is natural in the aforementioned contextual bandits application in \Cref{sec:conbandits}. 
\begin{theorem}\label{prop:kpieceregression}
    Adopt the semi-oblivious, smoothed online learning setting, where the adversary begins by choosing $g_{f^\ast}^\ast \in \G_\F$ from \eqref{eq:kpieceregclass}, and, at each time $t$, draws $x_t$ from a distribution that is $\sigma$-smooth with respect to $\mu$ and sets $y_t = g_{f^\ast}^\ast(x_t)$.  Then, \Cref{alg:general_reg} is an algorithm that is efficient in the number of calls to an ERM oracle over $\G_\F$ that satisfies for all $T$, with probability at least $1 - \delta$,
    \begin{equation}
        \reg_T \leq 136 K^2 d \log(d) + 91K^2 \log\left(TK^2/(\sigma \delta)\right) +  K^2(\ell + 1).
    \end{equation}
\end{theorem}
In  \Cref{app:piecewise_reg_guarantees} we prove a more general version of the above result that allows the regression functions on each piece to be polynomial.  The intuition is to reduce $K$-piece regression to $K$-class classification, but where each of the ``classes'' materialize sequentially, once there are sufficiently many points observed to ``determine'' one of the pieces. The algorithm and proof are considerably more subtle, and are given in \Cref{app:piecewise_reg_proof,app:piecewise_reg_alg}, respectively. We note that \Cref{alg:general_reg} only requires the ERM oracle to be called on sets of size independent of $T$, making the total runtime of the algorithm logarithmic in the horizon.

\section*{Acknowledgements}
AB acknowledges support from the National Science Foundation Graduate Research Fellowship under Grant No. 1122374.  We also would like to thank Sofiia Dubova for her help in translating \citet{tarasov1988method} and Michel Goemans for pointing to the closely related, English-language work of \citet{khachiyan1990inequality}.

\bibliographystyle{plainnat}
\bibliography{smoothedonlineoptimistic.bib}

\appendix
\newpage
\section{Contextual Bandits}\label{sec:conbandits}
In this section, we apply \Cref{prop:kpieceregression} and the approach of \citet{foster2020beyond} to the setting of contextual bandits with contexts drawn from a smooth distribution, considered in \citet{block2022smoothed}.  Unlike in that work, however, we will realize regret bounds achievable by an oracle-efficient algorithm that are polynomially improved both in the horizon and the number of actions in the particular case of noiseless rewards that are piecewise linear.  

We consider the following setting: the learner has access to the context set $\cB_1^d$ and an action set $\cA$ with $\abs{\cA} = A < \infty$.  Let
\begin{equation*}
    \cG_\F^{\cA} = \left\{\mathbf{g_f} = \left(\mathbf{g}_{\mathbf f}^a\right)_{a \in \cA}| \mathbf{g}_{\mathbf f}^a \in \cG_\F\right\}
\end{equation*}
where $\G_\F$ is as in \Cref{prop:kpieceregression}, be a class of functions $\mathbf{g_f}: \cB_1^d \times \cA \to \rr$.  Before the game begins, Nature selects some $\ell^\ast \in \G_\F^\cA$ unknown to the learner.  At each time $t$, Nature draws $x_t$ from a $\sigma$-smooth distribution on $\cB_1^d$; the learner then chooses $a_t \in \cA$, observes $\ell^\ast(x_t, a_t)$ and suffers the same loss.  Given $\ell^\ast \in \G_\F^\cA$, it is clear that the best policy, given a context is greedy:
\begin{equation*}
    \pi_{\ell^\ast}(x) = \argmin_{a \in \cA} \ell^\ast(x, a).
\end{equation*}
The goal of the learner is to minimize regret, $\reg_T$, to the optimal policy $\pi_{\ell^\ast}$.  The primary difference between our setting and that of \citet{foster2020beyond,block2022smoothed}, other than the fact that we are considering a particular function class $\G_\F^\cA$, is that our losses are \emph{noiseless}, while the prior works allow for some noise that is mean zero conditional on the history.  We have the following regret bound:
\begin{corollary}\label{prop:conbandits}
    Suppose that we are in the contextual bandit setting outlined above with $\G_\F$ from \eqref{eq:kpieceregclass} and $\cX \times \cA$ identified with some subset of $\cB_1^d$.  Then there is an oracle-efficient algorithm that, for all $T$, with probability at least $1 - \delta$, achieves
    \begin{equation*}
        \reg_T \leq 80 \cdot A\sqrt{ T \left(K^2 d \log(d) + K^2 \log\left(ATK/(\sigma \delta)\right)  \right)} + 8 \cdot \sqrt{A T \log\left(4/ \delta\right)}.
    \end{equation*}
\end{corollary}
We prove \Cref{prop:conbandits} in \Cref{app:conbandits} using the reduction of \citet[Theorem 1]{foster2020beyond} and \Cref{prop:kpieceregression}.  Note that, in contradistinction to the corresponding bound proved as \citet[Theorem 12]{block2022smoothed}, we achieve the optimal $\sqrt{T}$ regret, albeit with stronger assumptions on the setting.
%!TEX root = ../neurips_submission.tex
\subsection{Proof of Corollary \ref{prop:conbandits}} \label{app:conbandits}

In this section, we prove \Cref{prop:conbandits} by applying the black box reduction of \citet{foster2020beyond} to our \Cref{prop:kpieceregression}.  The key lemma is as follows:
\begin{algorithm}[!t]
    \begin{algorithmic}[1]
        \State{}\textbf{Init:}  $A$ instances of the Piecewise Regressor (\Cref{alg:general_reg}) $\algfont{regressor}(a)$ for $a \in \cA$, learning rate $\gamma > 0$, exploration parameter $\mu > 0$.
        \For{each time $t= 1,2,\dots$}
        \State{}\textbf{recieve } $x_t$
        \For{each action $a \in \cA$}
        \State{} \textbf{predict} $\yhat_t(a) = \algfont{regressor}(a)\algfont{.predict}(x_t)$ \algcomment{Prediction step of \Cref{alg:general_reg}}
        \EndFor
        \State{}Assign $b_t \gets \argmin_{a \in \cA} \yhat_t(a)$
        \For{each $a \neq b_t$}
        \State{} Assign
        \begin{align*}
            p_{t,a} \gets \frac{1}{\mu + \gamma (\yhat_t(a) - \yhat_t(b_t))} \tag{\algcomment{Inverse Gap Weighting}}
        \end{align*}
        \EndFor
        Assign
        \begin{align*}
            p_{t,b_t} \gets 1 - \sum_{a \neq b_t} p_{t,a} \tag{\algcomment{Inverse Gap Weighting}}
        \end{align*}
        \State{} \textbf{sample} $a_t \sim p_t$ and \textbf{play} $a_t$
        \State{} \textbf{observe} $\ell_t(a_t)$
        \State{} \textbf{update} $\algfont{regressor}(a)\algfont{.update}(x_t, a_t, \ell_t(a_t))$ \algcomment{Update step of \Cref{alg:general_reg}}

    \EndFor
      \end{algorithmic}
      \caption{Inverse Gap Weighting \citep{foster2020beyond} with Piecewise Regression}
      \label{alg:igw}
    \end{algorithm}

\begin{lemma}\label{lem:conbanditsreduction}
    Suppose that we are in the setting of \Cref{prop:conbandits} and that we predict $\yhat_t(a)$ and sample $a_t$ according to \Cref{alg:igw}.  Then, for all $T$, with probability at least $1 - \delta$, we have
    \begin{equation*}
        \sum_{t = 1}^T \bbI[\yhat_t(a_t) \neq \ell_t(a_t)] \leq A\left(136 K^2 d \log(d) + 91 K^2 \log\left(\frac{4 A T^2 K^2}{\sigma \delta}\right) + K^2 (\ell + 1)\right)
    \end{equation*}
\end{lemma}
\begin{proof}
    We begin by noting that
    \begin{equation*}
        \sum_{t =1}^T \bbI[\yhat_t(a_t) \neq \ell_t(a_t)] = \sum_{a \in \cA} \sum_{t = 1}^T \bbI[\yhat_t(a) \neq \ell_t(a)] \bbI[a_t = a]
    \end{equation*}
    let
    \begin{equation*}
        \cU = \left\{\text{for all } 1 \leq t \leq T \text{ and } a \in \cA \text{ if } p_{t,a} \leq \frac{\delta}{2 A T} \text{ then } a_t \neq a\right\}
    \end{equation*}
    A union bound implies that $\pp(\cU) \geq 1 - \frac \delta 2$.  Restricting to $\cU$, we note that for any $B \subset \cB_1^d$ measurable,
    \begin{align*}
        \pp_t\left(x_t \in B | a_t = a\right) \leq \frac{\pp_t(x_t \in B)}{p_{t, a_t}} \leq \frac{2 A T \mu_d(B)}{\sigma \delta}
    \end{align*}
    thus after restricting to $\cU$, the distribution of $x_t$ conditioned on $a_t = a$ is $\left(\frac{\delta \sigma}{2 A T}\right)$-smooth with respect to $\mu_d$.  Thus for each $a$, we may apply the regret bound from \Cref{prop:kpieceregression} and, summing over $a \in \cA$ concludes the proof.
\end{proof}
We can now prove \Cref{prop:conbandits}:
\begin{proof}[Proof of \Cref{prop:conbandits}]
    Note that by Lipschitzness and boundedness, twice the mistake bound is larger than the square loss regret considered in \citet{foster2020beyond}.  Applying \citet[Theorem 1]{foster2020beyond} concludes the proof.
\end{proof}
\newpage
%!TEX root = ../neurips_submission.tex

\section{Preliminaries}\label{app:prelims}
In this section, we provide some key definitions and results that come up in our analysis.  We divide the section by theme, with the first part collection results on probability and concentration, the second part on geometric measure theory, and the third on convex geometry.
\subsection{Probability and Concentration}
We begin by stating the foundation of our regret bounds.
\begin{lemma}[Ville's Inequality \citep{ville1939etude}]\label{lem:ville}
	Let $\scrF_t$ denote a filtration and suppose that the sequence of random variables $A_t$ is a supermartingale with respect to $\scrF_t$.  Suppose that
	\begin{equation*}
		\pp\left(A_t > 0 \text{ for all } t > 0\right) = 1.
	\end{equation*}
	Then for any $x > 0$, the following inequality holds:
	\begin{equation*}
		\pp\left(\sup_{t > 0} A_t \geq x \right) \leq \frac{\ee\left[A_0\right]}{x}.
	\end{equation*}
\end{lemma}
We will also require a standard Chernoff bound. 
\begin{lemma}[Chernoff Bound]\label{lem:chernoff} Let $X_1,\dots,X_t$ be a sequence of binary random variables such that $\Exp[X_i \mid X_{1},\dots,X_{i-1}] \ge \eta$. Then, 
\begin{align*}
\Pr\left[\sum_{i=1}^t X_i \le t\eta/2\right]\le \exp(-t\eta/8).
\end{align*}
\end{lemma}
Finally, we will clear up any confusion about which distribution is smooth: that of contexts $x_t$ or that of samples $(x_t, y_t)$.
\begin{lemma}
	Suppose that $x \sim p$ and $(x, y) \sim \widetilde{p}$ where $p, \widetilde{p}$ are distributions.  Suppose that $y \in \{\pm 1\}$.  Then if $p$ is $\sigma$-smooth with respect to $\mu$ then $\widetilde{p}$ is $\left(\frac \sigma 2\right)$-smooth with respect to $\mu \otimes \Unif(\{\pm 1\})$.  Conversely, if $\widetilde{p}$ is $\sigma$-smooth with respect to $\mu \otimes \Unif(\{\pm 1\})$ then $p$ is $\sigma$-smooth with respect to $\mu$.
\end{lemma}
\begin{proof}
	The converse follows immediately from \Cref{lem:smoothnesspushforward}, proved in \Cref{app:affine}.  To prove the first statement, note that any distribution on $\{\pm1\}$ is $\frac 12$-smooth with respect to $\Unif(\{\pm 1\})$.  Thus, decomposing $\widetilde{p}(x,y) = p(x) \cdot \widetilde{p}(y | x)$ concludes the proof.
\end{proof}

\subsection{Geometric Measure Theory}
The key definition is that of Hausdorff measure, which formally generalizes our intuitive notion of volume and surface area.
\begin{definition}[Hausdorff Measure \citep{federer2014geometric}]\label{def:hausdorff}
	Let $\cX$ be a metric space.  For any $k \in \rr_+$, we define the $k$-dimensional Hausdorff measure of a set $A \subset \cX$ to be
	\begin{equation*}
		2^{-k} \omega_k \lim_{\epsilon \downarrow 0} \vol_k^\epsilon(A),
	\end{equation*}
	where
	\begin{equation*}
		\vol_k^\epsilon(A) := \inf\left\{\sum_{i = 1}^\infty (\diam U_i)^k \bigg| A \subset \bigcup_{i = 1}^\infty U_i \text{ \emph{and} } \diam U_i < \epsilon  \right\}
	\end{equation*}
	and $\diam U_i$ is the diameter of the set $U_i$, i.e., the maximal distance between any two points contained in $U_i$.  We define the Hausdorff dimension $\dim(A) = \inf\{k > 0 | \vol_k(A) > 0\}$.  As is common, when we integrate with respect to the Hausdorff measure, we denote the measure in the integral as $d \cH^k$ instead of $d \vol_k$.
\end{definition}
Note that when $\cX = \rr^d$ then $\vol_d$ exactly coincides with the Lebesgue measure \citep{federer2014geometric}.  The following is an immediate consequence of the definition:
\begin{lemma}
	For a given set $A \subset \cX$, let $N(A, \epsilon)$ denote the minimal number of balls of radius $\epsilon$ required to cover $A$.  Then
	\begin{equation*}
		\vol_k(A) \leq \omega_k \epsilon^k N(A, \epsilon).
	\end{equation*}
\end{lemma}
\begin{proof}
	It is immediate from the definition that $\vol_k^\epsilon$ is monotone nonincreasing as $\epsilon \downarrow 0$.  The result follows by letting $U_i$ be the set of balls of radius $\epsilon$ covering $A$.
\end{proof}
We also use the co-area formula:
\begin{theorem}[Co-area Formula \citep{federer2014geometric}] \label{thm:coarea}
	Let $\phi: \rr^n \to \rr^m$ be a Lipschitz function with $n \geq m$.  Then, for $A \subset \rr^m$,
	\begin{equation*}
		\int_{\phi^{-1}(A)} \sqrt{\det(D\phi(x) D \phi(x)^T)} d \cH^m(x) = \int_{A} \vol_{n - m}(\phi^{-1}(y)) d \cH^m(y).
	\end{equation*}
\end{theorem}
This in turn implies the projection formula:
\begin{corollary}[\citep{federer2014geometric}]\label{lem:projection}
	Let $\phi: \cX \to \cY$ denote a $1$-Lipchitz map between $m$-dimensional sets $\cX, \cY$.  Then $\vol_m(\phi(\cX)) \leq \vol_m(\cX)$.
\end{corollary}
\begin{proof}
	By \Cref{thm:coarea},
	\begin{align*}
		\vol_m(\phi(\cX)) = \int_{\phi(\cX)} d \cH^m(y) \leq \frac{\sup_{x} \sqrt{\det(D\phi(x) D \phi(x)^T)}}{\inf_y \vol_0(\phi^{-1}(y))} \int_{\cX} d \cH^m(x) \leq \vol_m(\cX),
	\end{align*}
	where the last inequality holds because the Lipschitz assumption bounds the largest singular value of $D \phi$ and for any $y \in \phi(\cX)$, there is at least one point $x \in \phi^{-1}(y)$.
\end{proof}

\subsection{Convex Geometry}
We first define a polytope:
\begin{definition}
	We say that a set $A \subset \rr^d$ is a polytope if it is the intersection of a finite number of halfspaces.  If $A$ is the intersection of $K$ halfspaces, we say that it has $K$ faces.
\end{definition}
We now define an ellipsoid:
\begin{definition}
	Let $A \in \rr^{d\times d}$ be a positive definite matrix and let $a \in \rr^d$ be a point.  We define an ellipsoid to be
	\begin{equation*}
		\cE(A, a) = \left\{w \in \rr^d | (w - a)^T A^{-1} (w - a) \leq 1 \right\}.
	\end{equation*}
	Note that the volume of an ellipsoid is given by $\vol_d(\cE(A, a)) = \omega_d \sqrt{\det(A)}$.
\end{definition}
We now define the John ellipsoid associated with a convex body:
\begin{theorem}[John Ellipsoid \citep{john1948extremum,ball1997elementary}]\label{thm:john}
	Let $A \subset \rr^d$ be a convex body, i.e., a convex set with nonempty interior.  Then there is a unique ellispoid $\cE_A$ that has maximal volume subject to the condition that the ellipsoid is contained in $A$.  Furthermore, $A \subset d \cdot \cE_A$.
\end{theorem}
We require the following general fact about ellipsoids:
\begin{lemma}[Corollary 15 from \citet{rivin2007surface}] \label{lem:rivin}
	Suppose $\cE = \cE(A, a)$ is an ellipsoid and $A$ has eigenvalues given by $q = (q_1, \dots, q_d)$.  Then,
	\begin{equation}
		\vol_{d-1}(\partial \cE) \leq \norm{q} \cdot \sqrt{d} \cdot \vol_d(\cE).
	\end{equation}
\end{lemma}
Finally, we have the following result about cutting planes through the center of the John ellipsoid:
\begin{lemma}[\citet{tarasov1988method,khachiyan1990inequality}]\label{lem:tarasov}
	Let $A \subset \rr^d$ be a polytope with John ellipsoid $\cE_A$ with center $a$.  Let $A'$ be the intersection of $A$ and a halfspace going through $a$, i.e., there is some $w \in \rr^d$ such that
	\begin{equation*}
		A' = A \cap \left\{w \in \rr^d | \inprod{w}{x} \geq \inprod{a}{x} \right\}.
	\end{equation*}
	If $\cE_{A'}$ is the John ellipsoid of $A'$, then
	\begin{equation*}
		\vol_d(\cE_{A'}) \leq \frac 89 \vol_d(\cE_A).
	\end{equation*}
\end{lemma}

\newpage
%!TEX root = ../neurips_submission.tex

\section{Technical Workhorses}
This appendix proves the technical workhorses, the abstract decay lemma (\Cref{lem:masterreduction}), and the main geometric lemma, \Cref{lem:orthogonality}.

\subsection{Proof of Lemma \ref{lem:masterreduction}: The Abstract Decay Lemma}\label{app:masterreduction}
    We prove a slightly more general form of the lemma, with a weaker assumption on the sequence of $z$:
    \begin{lemma}\label{lem:generalizedmasterreduction}
        Suppose that a sequence $(\ell_t,z_t) $ satisfies $(R, c)$-geometric decay with respect to some a $\mu$ on $\cZ$, and define a sequence of stopping times $t_m$ where $t_m = t$ if $t$ is the $m^{th}$ time that $\ell_s(z_s) = 1$.  Let $m_t$ denote the maximal $m$ such that $t_m < t$ and thus $t_{m_t}$ is last time before $t$ that $\ell_s = 1$.  Suppose that for all $t$, the distribution of $z_t$ conditional on $t_{m_t}$ is $\sigma$-smooth with respect to $\mu$.  Then for all $T \in \bbN$, with probability at least $1 - \delta$, 
    \begin{equation}
        \sum_{t=1}^T \ell_t(z_t) \leq 4 \frac{\log\left(\frac{2 T R}{\sigma \delta}\right)}{\log\left(\frac 1c\right)} + \frac{e - 1}{1 - \sqrt[]{c}}.
    \end{equation}
    \end{lemma}
    \begin{proof}
    Fix a sequence of positive integers $h_k$ for $k \in \mathbb{N}$, whose values we tune at the end of the proof.  Let $\tau_0 = 0$ and for all $m > 0$, let
    \begin{align*}
        \tau_m &= \tau_{m-1} + \inf\left\{k > 0 \bigg| \sum_{t = \tau_{m-1} + (k-1)h_m}^{\tau_{m-1} + k h_m } \ell_t(z_t) = 1 \right\}\\
        &= \tau_{m-1} + \inf\left\{k > 0 \bigg| \exists t \in \tau_{m-1} + [(k-1)h_m,  k h_m-1] \text{ s.t. } \ell_t(z_t) = 1 \right\}.
    \end{align*}
    Furthermore, let $T(m) = \sum_{k = 1}^{m} (\tau_k - \tau_{k-1}) h_k$ and
    \begin{equation}
        t_m = \inf\left\{t > T(m-1) | \ell_t(z_t) = 1\right\}.
    \end{equation}
    In words, we consider epochs of length $h_m$, whose length can change every time we make a mistake in an epoch.  We have $T(m)$ the time of the $m^{th}$ change of epoch and $\tau_m$ the number of epochs of length $h_m$ we have to go before we make a mistake; we also have $t_m$ is the time of the first mistake after the $m^{th}$ change of epoch size.  Let
    \begin{equation}
        A_m =  \sum_{s = t_m + 1}^{T(m) - 1} \ell_s(z_s).
    \end{equation}
    be the number of mistakes in a given epoch other than the first mistake.  Let $\pi_m = \min\left(\frac{R_m}{\sigma}, 1\right)$, where we abbreviate $R_m = R_{t_m}$.  We first claim that with probability at least $1 - \delta$, for all $m$ it holds that:
    \begin{align}
        A_m \leq \log\left(\frac 1\delta\right) + (e-1) \sum_{k = 1}^m \pi_k (h_k - 1) \numberthis \label{eq:extramistakes}.
    \end{align}
    To see this, let 
    \begin{equation*}
        B_m^\lambda = \exp\left(\lambda A_m - \left(e^\lambda - 1\right)\sum_{k = 1}^m \pi_k h_k \right).
    \end{equation*}
    We show that $B_m^\lambda$ is a supermartingale for all $\lambda > 0$.  To see this, we have
    \begin{align*}
        \ee\left[B_m^\lambda | B_{m-1}^\lambda\right] = B_{m-1}^\lambda \ee\left[\exp\left(\lambda \sum_{s = t_m + 1}^{T(m) - 1} \mathbb{I}[\yhat_s \neq y_s] -\left(e^\lambda - 1\right) \pi_m (h_m - 1)\right) \bigg| B_{m-1}^\lambda \right] \leq B_{m-1}^\lambda,
    \end{align*}
    where the inequality follows because the conditional probability of a mistake for $t_m + 1 \leq T(m) - 1 \leq \pi_m$ by the assumption of smoothness conditional on a sub-sigma algebra of that generated by $t_{m-1}$ and realizability and $T(m) - 1 - (t_m + 1) \leq h_m - 1$ by construction.  Thus we may apply Ville's inequality from \Cref{lem:ville} and recover \eqref{eq:extramistakes}.

    \begin{claim}\label{claim:tau_separataion} With probability at least $1 - \delta$, it holds for all $m$ that
    \begin{align}
        \tau_m - \tau_{m-1} \geq \max\left(1, \log\left(\frac{\delta}{\pi_m h_m T}\right)\right). \label{eq:tauseparation}
    \end{align}
    \end{claim}
    \begin{proof}[Proof of \Cref{claim:tau_separataion}]
        For any $\tau_{m-1} + (k-1) h_m \leq t < \tau_{m-1} + k h_m$, smoothness implies
        \begin{equation}
            \pp\left(\ell_t(z_t) = 1 | \ell_s(z_s) = 0 \text{ for all } s < t\right) \leq h_m \pi_m.
        \end{equation}
        where we note that the event that $\ell_s(z_s) = 0$ for $s < t$ is contained in the sigma-algebra generated by $t_{m_t}$.  A union bound then implies that
    \begin{align*}
    \Pr\left[ \exists t \in \tau_{m-1} + [(k-1)h_m,~k h_m) \text{ s.t. } \ell_t(z_t) = 1   \bigg{|} \ell_s(z_s) = 0, \forall s \in [\tau_{m-1},(k-1)h_m)\right] \le h_m \pi_m.
    \end{align*}
    Hence, letting $X_m$ be a random variable distributed geometrically with parameter $\widetilde{\pi}_m = \min(h_m \pi_m, 1)$, $\tau_{m} - \tau_{m-1}$ stochastically dominates $X_m$.  Thus, for any $\lambda < - \log(1 - \pi_m)$,
    \begin{equation}
        \ee\left[e^{\lambda (\tau_{m} - \tau_{m-1})}\right] \leq \ee\left[e^{\lambda X_m}\right] = \frac{\widetilde{\pi_m} e^\lambda}{1 - (1 - \widetilde{\pi_m})e^\lambda}.
    \end{equation}
    We further note that
    \begin{align}
        \log\left(1 - (1 - \widetilde{\pi}_m)e^{-1}\right) &\geq 1 - \frac{1}{1 - (1 - \widetilde{\pi}_m)e^{-1}} = - \frac{(1 - \widetilde{\pi}_m)e^{-1}}{1 - (1 - \widetilde{\pi}_m)e^{-1}} \\
        &\geq - \frac{e^{-1}}{1 - e^{-1}}.
    \end{align}
    Thus, setting $\lambda = -1$, we see that with probability at least $\frac{\delta}{T}$,
    \begin{align}
        \tau_m - \tau_{m-1} \geq  1 + \log\left(\frac 1 { \widetilde{\pi_m}}\right) -\log\left(\frac T\delta\right) - \frac{e^{-1}}{1 - e^{-1}} \geq \log\left(\frac{\delta}{h \pi_m T}\right).
    \end{align}
    Because $\tau_m - \tau_{m-1} > 0$ by construction, we may then take a union bound to conclude the proof of the claim.
    \end{proof} %of \eqref{eq:tauseparation}.

    Now we note that
    \begin{equation}
        T \geq T(m) = \sum_{k = 1}^m (\tau_k - \tau_{k-1}) h_k
    \end{equation}
    and, further, that if $m_T$ is the maximal $m$ such that the preceding display holds,
    \begin{equation}
        \reg_T \leq m_T + A_{m_T}.
    \end{equation}
    Thus, combining \eqref{eq:extramistakes} and \eqref{eq:tauseparation}, along with the fact that $\pi_k \leq c^k R_0 / \sigma$, with probability at least $1 - 2 \delta$, we have
    \begin{align}
        T &\geq \sum_{k = 1}^{m_T} \log\left(\frac{\sigma \delta}{c^k R_0 h_k T}\right) h_k \label{eq:mTbound}\\
        \reg_T &\leq m_T + \log\left(\frac 1\delta\right) + (e - 1) \sum_{k  =1}^m c^k \frac{R_0}\sigma (h_k - 1) \label{eq:mTregbound}
    \end{align}
    Now, let $h_k = 1$ for $k \leq 2\log\left(\frac{ T R_0}{\sigma \delta}\right) / \log\left(\frac 1c\right)$ and let $h_k = c^{- \frac k2}$ otherwise.  Then we see that if \eqref{eq:mTbound} and \eqref{eq:mTregbound} hold, then
    \begin{equation}
        m_T \leq 2\frac{\log T}{\log\left(\frac 1c\right)} + \frac{2\log\left(\frac{ T R_0}{\sigma \delta}\right)}{\log\left(\frac 1c\right)} \leq - 4 \frac{\log\left(\frac{T R_0}{\sigma \delta}\right)}{\log c},
    \end{equation}
    and
    \begin{equation}
        \sum_{k = 1}^m c^k \frac{R_0}{\sigma}(h_k - 1) \leq \sum_{j = 0}^\infty c^{\frac j2} =  \frac{1}{1 - \sqrt{c}}.
    \end{equation}
    Thus we see that with probability at least $1 - \delta$.
    \begin{equation}
        \reg_T \leq 4 \frac{\log\left(\frac{2T R_0}{\sigma \delta}\right)}{\log\left(\frac 1c\right)} + \frac{e - 1}{1 - \sqrt[]{c}}.
    \end{equation}
    which proves the result.
\end{proof}
We note that \Cref{lem:masterreduction} follows immediately because $t_{m_t} < t$ almost surely and so the sigma algebra generated by $t_{m_t}$ is contained in that generated by the history up to $t-1$ and so the smoothness assumption of \Cref{lem:masterreduction} implies that of \Cref{lem:generalizedmasterreduction}.

%!TEX root = ../neurips_submission.tex

\subsection{Proof of Lemma \ref{lem:orthogonality}: The Key Geometric Lemma}\label{app:orthogonality}
%   We begin by recalling two definitions, the version space $\F_t$ and the region of disagreement $D_t$:
%   \begin{align}
%       \cF_t &= \left\{w \in \cB_1^d | \langle w, x_s\rangle = y_s \text{ for all } s \leq t\right\} \\
%       D_t &= \left\{x \in \cX | \text{ there exist } f, f' \in \F_t \text{ such that } f(x) \neq f'(x)\right\}.
%   \end{align}

%   \begin{proposition}[Linear Proposition]\label{prop:johns} For each $t$, let $\cE_t$ denote the John's ellipsoid inscribed in $\cF_t$, and let $w_t$ denote the center of $\cE_t$, and suppose  $\ell_t(x) = \I\{\sign(\langle w_t, x \rangle)\ne \sign(\langle w_\star, w \rangle)$. Then, $R_t := 3 \cdot 4^{d+1}\mu_d(d \cdot \cE_t)$ satisfies
%   \begin{itemize}
%       \item $\ell_t(x) \le \I\{x \in \cD_t\}$
%       \item Whenever $\ell_t(x) = 1$, $R_{t+1} \le \frac{8}{9}R_t$
%       \item $\mu_d(\cD_t) \le  R_t$, and $R_0 \le 4^{d+2}d^d$
%   \end{itemize}
%   \end{proposition}
% \subsubsection{Proof of \Cref{prop:johns}}
We begin by proving the following technical geometric lemma, which, for simplicity, considers subsets of the sphere, rather than of the ball. This ultimately suffices due to positive homogeneity of linear classifiers.  
\begin{lemma}\label{lem:tubevolume}
    Let $\hat\F \subset S^{d-1}$ be a measurable subset of the $(d-1)$-dimensional sphere imbedded in $\rr^d$.  Let $D(\hat\F)$ denote the set of points in $S^{d-1}$ orthogonal to at least one point in $\hat\F$, i.e.,
    \begin{equation}
        D(\hat\F) = \left\{x \in S^{d-1} | \text{for some } w \in \hat\F, \quad \inprod{w}{x} = 0 \right\}.
    \end{equation}
    Then, if $\vol_{k}$ is the $k$-dimensional Hausdorff measure on the sphere, we have
    \begin{equation}
        \vol_{d-1}(D(\hat\F)) \leq 2 \cdot 4^{d-1} \vol_{d-1}(\hat\F) +  4^{d + 1} \vol_{d-2}(\partial\hat\F).
    \end{equation}
\end{lemma}
\begin{proof}
    For a given set $A \subset S^{d-1}$, denote by $T(A, \epsilon)$ the ``tube'' of radius $\epsilon$ around $A$, i.e., the set of points in $S^{d-1}$ with distance at most $\epsilon$ from a point in $A$.  

    Note that for any fixed point $w \in \hat\F$, we have $D(w)$ is just the $(d-2)$-sphere formed by intersection the linear space orthogonal to $w$ with $S^{d-1}$.  If $\hat B_\epsilon(w)$ denotes the $\epsilon$-ball around $w$ in $S^{d-1}$, then we claim that $D(\hat B_\epsilon(w)) \subset T(D(w), \epsilon)$.  Indeed, suppose that $v \in \hat B_\epsilon(w)$ so that $\inprod{w'}{v} = 0$ for some $w' \in \hat B_\epsilon(w)$.  Let $\alpha$ be a member of the orthogonal group such that $\alpha w' = v$ and $\inprod{\alpha w}{w} = 0$.  Then $\inprod{v + \alpha(w - w')}{w} = 0$ and $\norm{\alpha (w - w')} = \norm{w - w'} \leq \epsilon$, proving the claim. 
    
    Let $N(\widehat{\F}, \epsilon)$ denote the minimum size of an $\epsilon$-net of $\F$ and let $P(\widehat\F, \epsilon)$ denote the maximum size of an $\epsilon$-packing.  By abuse of notation, we will also use $N(\widehat\F, \epsilon)$ to denote the minimal $\epsilon$-net itself.  The fact that $D(\widehat B_\epsilon(w)) \subset T(D(w), \epsilon)$ implies that
    \begin{equation}
        \vol_{d-1}(D(\hat\F)) \leq \vol_{d-1}\left(\bigcup_{w \in N(\F, \epsilon)} T(D(w), \epsilon)\right) \leq N(\hat\F, \epsilon) \cdot \vol_{d-1}(T(D(w), \epsilon))
    \end{equation}
    By packing, covering duality, we have
    \begin{equation}
        N(\widehat\F, \epsilon) \leq P\left(\widehat\F, \frac \epsilon 2\right) \leq \frac{2^{d-1} \vol_{d-1}(T(\widehat\F, \epsilon))}{\vol_{d-1}(\widehat B_\epsilon(w))}
    \end{equation}
    Now, we may apply \citet[Theorem 10.20]{gray2003tubes}, the generalization of Steiner's formula to submanifolds of a sphere, to get
    \begin{equation}
        \vol_{d-1}(T(\widehat\F, \epsilon)) \leq \vol_{d-1}(\widehat\F) + \vol_{d-2}(\partial \widehat\F) \left(2^{d-1} \epsilon + 2^{d-1} \epsilon^d \right)
    \end{equation}

    Putting this together, we have
    \begin{equation}
        \vol_{d-1}(D(\hat\F)) \leq 2^{d-1} \frac{ \vol_{d-1}(T(D(w), \epsilon))}{ \vol_{d-1}(B_\epsilon(w)))} \left(\vol_{d-1}(\hat\F) + \vol_{d-2}(\partial \hat\F) \left(2^{d-1} \epsilon + 2^{d-1} \epsilon^d\right)\right).
    \end{equation}
    Now we may apply Weyl's tube formula  \citep{weyl1939volume} (see \citet{gray2003tubes,lotz2015volume} for a clear exposition on the topic) to $S^{d-2}$ imbedded as the equator of $S^{d-1}$ to get that for any $\epsilon < 1$,
    \begin{equation}
        \frac{\vol_{d-1}(T(D(w), \epsilon))}{\vol_{d-1}(\widehat B_\epsilon(w))} \leq \frac{2 \omega_{d-1} \left((1 + \epsilon)^{d-1} - (1 - \epsilon)^{d-1}\right)}{\epsilon^{d-1} \omega_{d-1}} =2  \frac{\left((1+\epsilon)^{d-1} - (1 - \epsilon)^{d-1}\right)}{\epsilon^{d-1}}.
    \end{equation}
    As $\epsilon \uparrow 1$, the above expression tends to $2^{d}$.  Putting everything together, we have
    \begin{equation}
        \vol_{d-1}(D(\hat\F)) \leq 2 \cdot 4^{d-1} \vol_{d-1}(\hat\F) + 2 \cdot 2^{2 d + 1} \vol_{d-2}(\partial\hat\F).
    \end{equation}
    as desired.
\end{proof}
We now use the homogenity of the inner product to show that it suffices to consider the sphere:
\begin{lemma}\label{lem:normalizing}
    Let $\F \subset \cB_1^d$ and let $\widehat{\F}_t$ denote its projection to $S^{d-1}$.  Suppose that $\F$ is such that
    \begin{equation*}
        \F = \left\{r \widehat{x} | 0\leq r \leq 1 \text{ and } \widehat{x} \in \widehat{\F}\right\}.
    \end{equation*}
    Then,
    \begin{equation*}
        \frac{\vol_d(\F)}{\vol_d(\cB_1^d)} = \frac{\vol_{d-1}(\widehat{\F})}{\vol_{d-1}(S^{d-1})}.
    \end{equation*}
\end{lemma}
\begin{proof}
    Let $\widehat{\F}_r = r \widehat{\F}$.  Then we see from \Cref{thm:coarea} that $\vol_{d-1}(\widehat \F_r) = r^{d-1} \vol_{d-1}(\widehat \F)$.  Thus,
    \begin{equation*}
        \vol_d(\F) = \int_0^1 \vol_{d-1}(\widehat{\F}_r) d r = \vol_{d-1}(\widehat{\F}) \int_0^1 r^{d - 1} d r.
    \end{equation*}
    In particular, this holds for $\F = \cB_1^d$.  Thus, we have
    \begin{align*}
        \frac{\vol_d(\F)}{\vol_d(\cB_1^d)} &= \frac{\vol_{d-1}(\widehat{\F}) \int_0^1 r^{d - 1} d r}{\vol_{d-1}(\widehat{\cB_1^d}) \int_0^1 r^{d - 1} d r} = \frac{\vol_{d-1}(\widehat{\F})}{\vol_{d-1}(S^{d-1})}.
    \end{align*}
    as desired.
\end{proof}
We now put everything together:
\begin{proof}[Proof of Lemma \ref{lem:orthogonality}]
    Let $\widehat{D}(\F)$ be the set of $x \in D$ such that $\norm{x} = 1$ and let $\widehat{\F}$ be defined similarly.  By the positive homogeneity of both $D(\F)$ and $\F$, we have
    \begin{align}
        \mu_d(D) &= \frac{\vol_d(D)}{\vol_d(\cB_1)} = \frac{\vol_{d-1}(\widehat{D}(\F))}{\vol_{d-1}(\partial \cB_1)} \\
        \mu_d(\F) &= \frac{\vol_d(\F)}{\vol_d(\cB_1)} = \frac{\vol_{d-1}(\widehat{\F})}{\vol_{d-1}(\partial \cB_1)}
    \end{align}
    where $\vol_{d-1}(\cdot)$ denotes the $(d-1)$-dimensional Hausdorff measure.  Thus, it suffices to compare $\vol_{d-1}(\widehat{D}(\F))$ with $\vol_{d-1}(\widehat{\F})$, which is the content of Lemma \ref{lem:tubevolume}.  The result follows.
\end{proof}
\newpage
%!TEX root = ../neurips_submission.tex

\section{Proofs From Section \ref{sec:warmup}}\label{app:warmup}
	In this appendix, we provide proofs of \Cref{prop:warmup} and \Cref{prop:warmuplowerbound}.
	\subsection{Proof of Theorem \ref{prop:warmup}} 
	% \subsection{Explanation}
	% Before we prove Proposition \ref{prop:warmup}, we first prove a more general lemma that shows that logarithmic regret is possible whenever $D_t$ shrinks geometrically in the number of mistakes.  We will then show that our particular choice of estimator is such that any mistake forces the region of disagreement to shrink by a constant factor.
	% 	Thus, in order to get the desired regret guarantee, it suffices to find a geometrically decreasing sequence of upper bounds on the volume of the region of disagreement.  For our estimator, we will do this by first relating $D_t$ to $\F_t$, then proving a technical geometric lemma that will imply that $\mu(D_t)$ can be bounded by $\mu(\F_t)$ up to a constant factor depending only on the dimension.
	In order to apply \Cref{lem:masterreduction} to prove \Cref{prop:warmup}, we need to show that the loss functions satisfy $(R, c)$-geometric decay.  We will show this for $R = 4^{d + 1} d^{2d}$ and $c = \frac 89$ using \Cref{lem:orthogonality} and some more convex geometry.  We begin by proving the following characterization of the disagreement region, which will in turn allow us to apply \Cref{lem:orthogonality}:
	\begin{lemma}\label{lem:inclusioninorthogonal}
		Suppose that we are in the situation of \Cref{prop:warmup}.  Then we have
		\begin{equation}
			D_t \subset \left\{x \in \cB_1^d | \inprod{w}{x} = 0 \text{ for some } w \in \F_t\right\}.
		\end{equation}
	\end{lemma}
	\begin{proof}
		Recall that the version space is defined as
		\begin{equation}
			\F_t = \left\{w \in \F |  \inprod{w}{y_s x_s} \geq 0 \text{ for all } s < t \right\}.
		\end{equation}
		If $x \in D_t$, then there are $w, w' \in \F_t$ such that $\sign(\inprod{w'}{x}) \neq \sign(\inprod{w''}{x})$.  Consider the continuous function $h(\lambda) = \inprod{\lambda w'' + (1 - \lambda) w'}{x}$; by the intermediate value theorem, there is some $0 < \lambda^\ast < 1$ and $w = \lambda^\ast w'' + (1 - \lambda^\ast) w'$ such that $\inprod{w}{x} = 0$.  By convexity of $\F_t$, then $w \in \F_t$ and thus $D_t$ is contained within the set of points orthogonal to at least one point in $\F_t$.
	\end{proof}
	With \Cref{lem:inclusioninorthogonal} in hand, we will be able to apply \Cref{lem:orthogonality} and it will suffice to control $\mu(\F_t)$ and $\vol_{d-1}(\F_t)$.  The next result bounds these quantities in terms of their analogues in $\cE_t$:
	\begin{lemma}\label{lem:surfacearea}
	    Let $\F \subset \rr^d$ be a convex body with John ellipsoid $\cE$.  Then we have
	    \begin{align}
	        \mu(\F) \leq d^d \mu(\cE) && \vol_{d-1}(\partial \F) \leq 2 d^d \mu(\cE).
	    \end{align}
	\end{lemma}
	\begin{proof}
	    Note that it is a classical fact that $\F \subset d \cdot \cE$ \citep{john1948extremum} and thus $\mu(\F) \leq d^d \mu(\cE)$.  We thus only have to prove the second bound.  To do this we first note that
	    \begin{equation}
	        \vol_{d-1}(\partial \F) \leq \vol_{d-1}(\partial (d \cdot \cE)).
	    \end{equation}
	    To see that this is the case, consider $\pi: \partial(d \cdot \cE) \to \partial \F$ be projection onto the convex $\F$.  Then $\pi$ is a contraction and thus shrinks Hausdorff measure as per \Cref{lem:projection}.  We now apply \Cref{lem:rivin} and note that because our ellipsoids are contained in the ball, $\norm{q} \leq 2 \cdot \sqrt{d}$, where $q$ is the vector of semi-axis lengths, i.e., the eigenvalues of the associated positive definite matrix.  Thus we have
	    \begin{equation}
	        \vol_{d-1}(\partial(d \cdot \cE)) = d^{d-1} \vol_{d-1}(\partial \cE) \leq d^{d-1} \left(2 d \mu(\cE)\right),
	    \end{equation}
	    and the result follows.
	\end{proof}
	We are now finally ready to apply the geometry that we have done so far to prove a slightly more general form of \Cref{prop:warmup}, which we will require to apply some of our reductions below.
	\begin{proposition}\label{prop:warmupgeneralized}
		Suppose we are in the situation of \Cref{prop:warmup} with the added complication that at any time $t$, the adversary can choose to censor round $t$ from the learner, so the learner does not observe $y_t$ and does not suffer loss at time $t$.  We further allow $x_t$ to be drawn adversarially with the condition that if $x_t$ is drawn adversarially, then the adversary always censors time $t$.  Then the conclusion of \Cref{prop:warmup} holds.
	\end{proposition}
	\begin{proof}
		By \Cref{lem:masterreduction}, it suffices to show that with our choice of $w_t$, the sequence
		\begin{equation*}
			\ell_t = \bbI\left[\sign(\inprod{w_t}{x_t}) \neq y_t \text{ and round } t \text{ is not censored} \right]
		\end{equation*}
		satisfies $(R, c)$ geometric decay with respect to $\mu$ for come $R, c$.  In particular, we need to find an $R \geq \mu(D_1)$ and a $c < 1$ such that if we make a mistake, then $\mu(D_{t+1}) \leq c \mu(D_t)$.  Note that if $\ell_t = 1$ then we must have $x_t \in D_t$ and $y_t$ is not censored.  By Lemma \ref{lem:inclusioninorthogonal} it suffices to control the size of the set of points orthogonal to at least one $w \in \F_t$; by Lemma \ref{lem:orthogonality}, it in turn suffices to control $\mu(\F_t)$ and $\vol_{d-1}(\F_t)$.  Applying \Cref{lem:surfacearea} to the preceding logic, we have:
		\begin{align}
			\mu(D_t) &\leq 2 \cdot 4^{d-1} \mu(\F_t) + \frac{4^{d+1}}{\omega_d} \vol_{d-1}(\partial \F_t) \\
			&\leq 2 \cdot 4^{d-1} \cdot d^d \mu(\cE_t) + \frac{4^{d+1}}{\omega_d} \cdot 2 \cdot d^d \mu(\cE_t) \\
			&\leq 4^{d + 1} d^{2d} \mu(\cE_t).
		\end{align}
		As $\mu(\cE_t) \leq 1$, we may choose $R = 4^{d+1} d^{2d}$ and reduce to showing that every time we make a mistake, $\mu(\cE_{t+1}) \leq c \mu(\cE_t)$ for some $c$.
		
		Now, suppose that we make a mistake at time $t$, i.e., $\inprod{w_t}{y_t x_t} < 0$.  Then, we have
		\begin{align}
			\F_{t+1} = \F_t \cap \left\{w \in \F | \inprod{w}{x_t y_t} > 0\right\} \subset \F_t \cap \left\{w \in \F | \inprod{w}{x_t y_t} \geq \inprod{w_t}{x_t y_t}\right\}.
		\end{align}
		by monotonicity.  Thus $\F_{t+1}$ is a subset of the intersection of $\F_t$ and a halfspace through the center of $\cE_t$.  Thus, by \Cref{lem:tarasov}, $\vol(\cE_{t+1}) \leq \frac 89 \vol(\cE_t)$.  Thus, we may choose $c = \frac 89$ and conclude the proof. 
	\end{proof}
	We remark that \Cref{prop:warmup} trivially follows from \Cref{prop:warmupgeneralized} by restricting the adversary to never censor a time $t$.

\subsection{Proof of Theorem \ref{prop:warmuplowerbound}}
	We construct separate adversaries which regret  $\ee\left[ \reg_T \right] \geq \Omega (d)$ and $\ee\left[ \reg_T \right] \geq \Omega \left( \log\left(\frac T\sigma \right) \right)$. Randomizing between the two with probability one-half gives the lower bound. 

	We first note that any algorithm must experience $\ee\left[\reg_T\right] \geq \frac {d+1}2$ against some adversary; indeed, as a generic set of $d + 1$ points defines a hyperplane, a realizable adversary can choose $y_t$ as independent Rademachers for each $1 \leq d \leq d + 1$ and the learner will suffer expected regret $\frac{d+1}{2}$.

	We now construct an adversary in one dimension that will force $\ee\left[\reg_T\right] \geq \Omega\left(\log\left(\frac T \sigma\right)\right)$; by projecting onto some fixed direction, the higher dimensional case reduces to this setting.  Thus, suppose that the $x_t$ are required to be sampled from a distribution that is $\sigma$-smooth with respect to the uniform measure on the unit interval.  At each time $t$, let $D_t$ be the interval between the rightmost $x_s$ labelled $-1$ and the leftmost $x_s$ labelled $1$, let $R_t$ be its length and $w_t$ its midpoint.  Fix $0 < \epsilon < 1$ to be tuned later and let
	\begin{align}
		\widetilde{D}_t = \left\{x \in D_t | \frac{1 - \epsilon}2 R_t \leq \abs{x - w_t} \leq \frac 12 R_t \right\}
	\end{align}
	be the set of points in the disagreement region close to its boundary.  We let the adversary select the distribution that picks uniformly from $\widetilde{D}_t$ with probability $\min\left(\frac{\mu(\widetilde D_t)}{\sigma}, 1\right)$ and with remaining probability selects $0$.  If $\abs{x_t - w_t} \geq \frac {R_t}2$, then $y_t$ is determined by realizability.  Otherwise, let $y_t$ be an independent Rademacher random variable.

	Let $\pi_m = \min\left(\frac{\epsilon R_m}{\sigma}, 1\right)$ and $t_m$ be the $m^{th}$ time that $x_t \in \widetilde D_t$, we see that $t_{m+1} - t_m$ is geometrically distributed with parameter $\pi_m$ and thus
	\begin{equation}
		B_m^\lambda = \exp\left(\lambda t_m - m\lambda  - \sum_{k = 1}^m \log\left(\frac{\pi_k}{1 - (1 - \pi_k) e^\lambda}\right)\right)
	\end{equation}
	is a supermartingale for $\lambda < \min_{k \leq m }\left(- \log(1 - \pi_k)\right) = - \log(1 - \pi_m)$.  Note that by construction, $R_{m+1} \geq (1 - \epsilon) \R_m$ and thus $R_m \geq (1 - \epsilon)^m$.  Setting $\lambda = \pi_m \leq - \log(1 - \pi_m)$ and applying Ville's Inequality, \Cref{lem:ville}, we get that with probability at least $ 1- \delta$, for all $m$,
	\begin{align}
		t_m \leq m + \frac{\log\left(\frac{1}{\delta}\right)}{\pi_m} +\frac 1{\pi_m} \sum_{k = 1}^m \log\left(\frac{\pi_k}{1 - (1 - \pi_k)e^{\pi_k}}\right).
	\end{align}
	We now note that
	\begin{align}
		\frac{1}{\pi_m} \log\left(\frac{\pi_k}{1 - (1 - \pi_k)e^{\pi_k}}\right)  &\leq \frac 1{\pi_m}\left(\frac{\pi_m}{1 - (1 - \pi_m)e^{\pi_m}} - 1 \right) \\
		&= \frac{e^{\pi_m} - 1}{\pi_m} \frac{1 - \pi_m}{1 - (1 - \pi_m)e^{\pi_m}} \\
		&\leq (e - 1) \frac{2}{\pi_m^2}
	\end{align}
	using monotonicity, the fact that $\pi_m \leq 1$ and the following computation:
	\begin{equation}
		1 - (1-x) e^x = 1 - (1-x) \sum_{k = 0}^\infty \frac{x^k}{k!} = \sum_{k =2}^\infty x^k \left(\frac{1}{(k-1)!} - \frac{1}{k!}\right) \geq \frac{x^2}{2}.
	\end{equation}
	Now, using the fact that $\pi_m \geq \frac \epsilon \sigma (1 - \epsilon)^m$, we have
	\begin{align}
		t_m \leq m + \frac{\sigma}{\epsilon}(1 - \epsilon)^{-m} \log\left(\frac 1\delta\right) + 2 (e - 1)m\left(\frac{\epsilon}{\sigma}\right)^{-2} (1 - \epsilon)^{-2m} .
	\end{align}
	Setting $\epsilon = 1 - e^{-1}$, we see that there is some constant $c > 0$ such that with probability at least $1 - \delta$,
	\begin{align}
		t_m \leq c \max\left(\sigma e^{m} \log\left(\frac 1\delta\right), m \sigma^2 e^{2m}\right).
	\end{align}
	In particular, there is a universal constant $C$ such that if $m = C \log\left(\frac{T}{\sigma \log\left(\frac 1\delta\right)}\right)$, then with probability at least $1 - \delta$ we have $\reg_T \geq m$ because the probability of a mistake, given that $x_t \in \widetilde{D}_t$ is $\frac 12$.  The result follows.

\subsection{Lower bound against naive play. }\label{app:naive_lb}
\newcommand{\Fthresh}{\F^{\mathrm{thres}}}
\newcommand{\fhatnaive}{\hat{f}^{\mathrm{naive}}}
\newcommand{\xtil}{\tilde{x}}

In this section, we show that it is necessary to choose the half-spaces $w_t$ intelligently in order to attain logarithmic-in-$1/\sigma$ regret.

Consider $d = 1$, so $\mu_1$ is the uniform measure on the interval $[-1,1]$. We define $\Fthresh := \{x \mapsto \sign(x -c) \}$ as the set of (monotone) threshold classifiers.  Given a function class $\cF$, we say that a learning strategy is \emph{consistent}, if at each $t \in [T]$, it selects an $f_t \in \cF_{t}$ in the version space $\cF_{t} := \{f \in \cF: f(x_s) = y_s, \quad 1 \le s \le t-1\}$. Define the left and right endpoints of the negative and positive regions
\begin{align*}
\xtil_t := \max\left\{-1,\max_{1 \le s \le t}\left\{x_t : y_t = -1\right\}\right\}, \quad \bar{x}_t := \min\left\{1,\min_{1 \le s \le t}\left\{x_t: y_t = 1\right\}\right\}.
\end{align*}
For a given $\eta > 0$, we consider the strategy
\begin{align}
\yhat_t = \begin{cases}  \sign\left(x_t - \xtil_{t-1} - \eta\right) & \xtil_{t-1} + \eta < \bar{x}_t \\
\sign\left(x_t - \frac{1}{2}\left(\xtil_{t-1} + \bar{x}_{t-1}\right)\right) & \text{otherwise}.
\end{cases} \label{eq:strategy_cons}
\end{align}
This is consistent with $\Fthresh$, since the thresholds are always chosen strictly between $\xtil_{t-1}$ and $\bar{x}_{t-1}$. However, the strategy is very naive, since it defaults to setting the threshold only slightly to the right of $\xtil_{t-1}$. As a consequence, we show it suffers $\Omega(1/\sigma)$ regret when $\eta$ is small.
\begin{proposition} Fix $\eta > 0$. For  $T \ge 1$ and $\sigma \in (1/T,1/4]$, there exists an $\Fthresh$-realizable, $\sigma$ smooth adversary such that  the strategy in \Cref{eq:strategy_cons} suffers expected regret linear in $1/\sigma$ for $\eta$ small:
\begin{align*}
\Exp[\reg_T] \ge  \floor{\frac{1}{\sigma}}\cdot\left(1-\frac{\eta}{2\sigma}\right).
\end{align*}
\end{proposition}
\begin{proof}  At each time $1 \le t \le T_0 : = \floor{1/\sigma}$, the adversary selects 
\begin{align*}
x_t = -1 + 2\sigma(t-1) + 2\sigma a_t, \quad a_t \sim \Unif([0,1]).
\end{align*}
For times $t \ge T_0$, the adversary selects $x_t \sim \Unif([-1,1])$. This adversary is clearly $\sigma$ smooth, satisfies $x_t \in -1 + 2\sigma[t-1,t]$ until $T_0 $, and then plays arbitrarily. Moreover, for $\sigma \le 1/T$, $x_t \in [-1,1]$ for all $t$.Fixing a ground-truth classifier $\fst(x) = \sign(x- 1)$,  we see $y_s = \fst(x_t) = -1$ is realizable for all $t$. 

Lastly, we analyze the regret of \Cref{eq:strategy_cons}; notice that under the above adversary, $\bar{x}_t = 1$, so we are always in the first case $y_t = \sign(x_t - \max_{1 \le s \le t-1}x_s - \eta)$. Then, for any $t \le T_0$,
\begin{align*}
\pp\left(\yhat_t = y_t | \filt_{t-1}\right) &= \pp\left(x_t \leq \eta + \max_{1 \leq s \leq t-1} x_s | \filt_{t-1}\right)\\
&\leq \Pr\left[x_t \le \eta + 2(t-1)\sigma - 1 \mid \filt_{t-1}\right]\\
&= \Pr_{a_t \sim \Unif([0,1])}\left[-1 + 2\sigma(t-1) + 2\sigma a_t \le \eta + 2(t-1)\sigma - 1\right]\\
&= \Pr_{a_t \sim \Unif([0,1])}\left[a_t \le \frac{\eta}{2\sigma}\right] = \frac{\eta}{2\sigma}.
\end{align*}
Hence, 
\begin{align*}
\Exp[\reg_t] &= \sum_{t=1}^{T}\Exp[\Pr[\yhat_t \ne y_t \mid \filt_{t-1}]]\\
&\ge \sum_{t=1}^{T_0}\Exp[\Pr[\yhat_t \ne y_t \mid \filt_{t-1}]]\\
&= \sum_{t=1}^{T_0}1-\Exp[\Pr[\yhat_t = y_t \mid \filt_{t-1}]]\\
&\ge \sum_{t=1}^{T_0}\left(1-\frac{\eta}{2\sigma}\right) = T_0\left(1-\frac{\eta}{2\sigma}\right), \quad T_0 := \floor{1/\sigma}.
\end{align*}
\end{proof}

\newpage
%!TEX root = ../neurips_submission.tex

\section{Proofs from Section \ref{sec:generalizations}}\label{app:generalizations}

%!TEX root = ../neurips_submission.tex

\subsection{Proof of Corollary \ref{prop:affinereduction}}\label{app:affine}
The key technical result is contained in the following lemma, which says that we can lift a $\sigma$-smooth distribution on $\cB_1^d$ to one on $\cB_1^{d+1}$ and only lose a factor that is exponential in dimension.  Because our regret guarantees are only logarithmic in $\sigma$, this will translate into a factor that is only linear in $d$ by reducing to the setting of \Cref{prop:warmup}.  We have the following result:
\begin{lemma}\label{lem:affine_conversaion} 
    There exist a probability kernel $\cK: \cB_1^d \to \Dists(\cB_1^{d+1})$ that satisfies the following two properties: first,
\begin{align*}
\pp_{\tilde x \sim \cK(\cdot \mid x)}\left(\text{for all } \tilde{w} = (w,b) \in \rr^d\times \rr, ~\sign(\langle w,x \rangle + b) = \sign(\langle \tilde w, \tilde x \rangle)\right) = 1,
\end{align*}
and second, if $p$ is $\sigma$-smooth with respect to $\mu_d$, then $\cK \circ p$ is $\sigma'$-smooth with respect to $\mu_{d+1}$, where $\sigma' = \sigma/4^{d+2}$ and $\widetilde x \sim \cK \circ p$ if $x \sim p$ and $\widetilde x \sim \cK(\cdot | x)$.
\end{lemma}
\begin{proof}
For general $b$,  define $\widetilde{w} := (w,b)$, let $\phi(x, z) = \frac{z(x, 1)}{4}$, and let 
\begin{align}
    \widetilde{x} = \phi(x_t, z_t) && z_t \sim \Unif(1,2) \label{eq:affineinclusion}.
\end{align}
Note that whenever  $x \in \cB_1^d$, $\pp[\widetilde{x}_t \in \cB_1^{d+1} \mid x_t = x] = 1$. Moreover,  
\begin{align*}
\sign(\langle w,x_t \rangle + b) = \sign(\langle \widetilde{w},(x_t,1)\rangle ) = \sign(\langle \widetilde{w},z_t(x_t,1)\rangle ) = \sign(\langle \widetilde{w}, \widetilde{x}_t \rangle).
\end{align*}
Since our proposed algorithm is a function only of the \emph{version space} $\cF_t$, and not the \emph{disagreement region}, it follows that we can assume without loss of generality that the learner interacts with the distribution $\widetilde{p}_t$ induced by drawing $x_t \sim p_t$, $z_t \sim \Unif[1,2]$, and $\widetilde{x}_t = \phi(x_t,z_t)$. 

To conclude, we must argue that if $x_t \sim p_t$ is $\sigma$-smooth with respect to the uniform measure $\mu_d$ on $\cB^d_1$, then  $\tilde{x}_t \sim \tilde{p}_t$ is $\sigma'$-smooth with respect to the uniform measure $\mu_{d+1}$ on $\cB^{d+1}_1$, for an appropriate $\sigma'$.  

Let $\widetilde{\mu}_{d+1}$ denote the density of  $\widetilde{x} = \phi(x,z)$, where $z \in \Unif[1,2]$ when $x \sim \mu_{d}$. Then, 
\begin{align}
\frac{\rmd \widetilde{p}_t(\widetilde x)}{\rmd \mu_{d+1}(\widetilde x)} &= \frac{\rmd \tilde{p}_t(\tilde x)}{\rmd \tilde \mu_{d+1}(\tilde x)} \cdot \frac{\tilde \mu_{d+1}(\tilde x)}{\mu_{d+1}(\tilde x)}. \label{eq:affinelemma}
\end{align}
To bound the first term, consider $\phi^{-1}$, the inverse of $\phi$ from $\cB_{1/4}^d \times [\frac{1}{4},1/2] \to \cB_{1/4}^d \times [1,2]$ given by 
\begin{align*}
\phi^{-1}:\tilde{x} = (x,z) \mapsto ((x/z), 4z)
\end{align*}
Then, $\tilde{p}_t$ is the pushforward under $\phi^{-1}$ of the measure $p_t \otimes \Unif[\frac{1}{4},\frac{1}{2}]$, and $\tilde{\mu}_{d+1}$ the pushforward of $\mu_d \otimes \Unif[\frac{1}{4},\frac{1}{2}]$. Thus, by \Cref{lem:smoothnesspushforward}, we have that $\widetilde{p}_t$ is $\sigma$-smooth with respect to $\widetilde{\mu}_{d+1}$.

Now, we compute $\frac{\rmd \widetilde \mu_{d+1}(\widetilde x)}{\rmd \mu_{d+1}(\widetilde x)}$. It suffices to show that for any set $H  \subset \cB_{1/2}^d \times [\frac{1}{4},1/2]$, we have
\begin{align}
\widetilde{\mu}_{d+1}(H) \le C\rmd \mu_{d+1}(H),
\end{align}
for some desirable constant $C$. Let $J_{\phi^{-1}}$ denote the Jacobian of the map $\phi^{-1}: (x,u) \mapsto (x/z, 4z)$. Then, $J_{\phi^{-1}}$ is a triangular matrix with determinant $4(1/z)^d$. Thus, on $H \subset \cB_{1/2}^d \times [\frac{1}{4},1/2]$, its determinant is at most $|\det_{J_{\phi}}(x,z)| = 4^{d+1}$. Hence, 
\begin{align*}
\tilde{\mu}_{d+1}(H) &= \pp_{(x,z) \sim \cB_{1}^d \times \Unif[1,2]}[\phi(x,z) \in H]\\
&= \pp_{(x,z) \sim \cB_{1}^d \times \Unif[1,2]}[(x,z) \in \phi^{-1}(H) ]\\
&= \frac{\int_{\phi^{-1}(H)}  \rmd x \rmd z}{\vol_{d+1}(\cB_{1}^d \times \Unif[1,2])}\\
&= \frac{\int_{H} |\det(J_{\phi^{-1})(x,z)}| \rmd x \rmd z}{\vol_{d+1}(\cB_{1}^d \times \Unif[1,2])}\\
&\le \frac{ 4^{d+1}\vol(H) }{\vol_{d+1}(\cB_{1}^d \times \Unif[1,2])}\\
&=\mu_{d+1}(H) \cdot \frac{ 4^{d+1}\vol_{d+1}(\cB_1^{d+1}) }{\vol_{d+1}(\cB_{1}^d \times \Unif[1,2])}\\
&=\mu_{d+1}(H) \cdot \frac{ 4^{d+1}\vol_{d+1}(\cB_1^{d+1}) }{\vol_{d}(\cB_{1}^d)}\\
&= 4^{d+1}\frac{\sqrt{\pi}}{d+\frac{1}{2}} \le 4^{d+2}.
\end{align*}

Combining these computations with \eqref{eq:affinelemma} yields
\begin{align*}
\frac{\rmd \tilde{p}_t(\tilde x)}{\rmd \mu_{d+1}(\tilde x)} &= \frac{\rmd \tilde{p}_t(\tilde x)}{\rmd \tilde \mu_{d+1}(\tilde x)} \cdot \frac{\tilde \mu_{d+1}(\tilde x)}{\mu_{d+1}(\tilde x)} \le \sigma^{-1} 4^{d+2},
\end{align*}
which concludes the proof.
\end{proof}

\begin{lemma}\label{lem:smoothnesspushforward}
    Suppose that $f: \cX \to \cY$ is a measurable map and suppose that $p,\mu$ are measures on $\cX$ and $p$ is $\sigma$-smooth with respect to $\mu$.  Define the pushforward measure on $\cY$ by taking $f_\ast \mu(B) = \mu(f^{-1}(B))$ for any measurable $B \subset \cY$.  Then $f_\ast p$ is $\sigma$-smooth with respect to $f_\ast \mu$.
\end{lemma}
\begin{proof}
    Let $B \subset \cY$ be measurable.  Then
    \begin{align}
        f_\ast p(B) = p(f^{-1}(B)) \leq \frac{\mu(f^{-1}(B))}{\sigma} \leq \frac{f_\ast \mu(B)}{\sigma}.
    \end{align}
    As this holds for any $B \subset \cY$, the result follows.
\end{proof}

\begin{algorithm}[!t]
    \begin{algorithmic}[1]
    \State{}\textbf{Initialize } $\widetilde \cW_1 = \cB_1^{d+1}$, $\widetilde w_1 = \mathbf{e_1} \in \widetilde \cW_1$,
    \For{$t=1,2,\dots$}
    \State{}\textbf{Recieve } $x_t$, \textbf{draw} $z_t \sim \Unif(1,2)$, and \textbf{assign} 
    \begin{align*}
        \widetilde{x}_t \gets \phi(x_t, z_t) = \frac{z_t (x_t, 1)}{4}
    \end{align*}
    \\
    \State{}\textbf{predict}
    \begin{align*}
    \yhat_t = \sign(\inprod{\widetilde w_t}{\widetilde x_t}), \tag{\algcomment{$\selfdot\classify(x_t)$}}
    \end{align*}
    \State{}\textbf{Update} $\widetilde \cW_{t+1} = \widetilde \cW_t \cap \{\widetilde w \in \cB_1^{d+1}| \inprod{\widetilde w}{\widetilde x_t y_t} \geq 0 \}$
    \If{$\yhat_t \ne y_t$} \quad\qquad\qquad\qquad\qquad\qquad\qquad\qquad\qquad\qquad(\algcomment{$\selfdot\errupdate(x_t)$})
    \State{} $\widetilde w_{t+1} \gets \johnellipsoid(\widetilde \cW_{t+1})$ \\
    \qquad\qquad \algcomment{returns center of John Ellpsoid of given convex body}
    \EndIf
    \EndFor
    \end{algorithmic}
      \caption{Binary Classification with Affine Thresholds}
      \label{alg:affine}
    \end{algorithm}
We now describe \Cref{alg:affine}.  At each time $t$, we draw $z_t \sim \Unif(1,2)$ independently and form $\widetilde{x}_t = \phi(x_t, z_t)$, where $\phi$ is as in \eqref{eq:affineinclusion}.  We then run the classify and update subroutines of \Cref{alg:warmup} at each time step on the new data sequence $(\widetilde{x}_t, y_t)$.  We are now ready to prove \Cref{prop:affinereduction}:
\begin{proof}[Proof of \Cref{prop:affinereduction}]
    We use \Cref{alg:affine} to reduce the problem to the situation of \Cref{prop:warmup}.  Indeed, by \Cref{lem:affine_conversaion}, the data sequence $(\widetilde{x}_t, y_t)$ satisfies the property that $\widetilde{x}_t$ is $\left(4^{- d - 2} \sigma\right)$-smooth with respect to $\mu_{d+1}$ and is realizable by the function class $\Flin^{d+1}$.  The result then follows immediately from \Cref{prop:warmup}.
\end{proof}

\subsection{Proof of Theorem \ref{prop:samedimensionfeatures}}\label{app:samedim}
We now prove begin generalizing beyond linear function classes with \Cref{prop:samedimensionfeatures}.  The key technical result shows that if $\phi: \cB_1^d \to \cB_1^d$ is well-behaved, then $\phi_\ast \mu_d$ is $\sigma$-smooth with respect to $\mu_d$, which will then allow us to apply \Cref{prop:warmup}.
\begin{algorithm}[!t]
    \begin{algorithmic}[1]
    \State{}\textbf{Initialize } $\widetilde \cW_1 = \cB_1^{d+1}$, $\widetilde w_1 = \mathbf{e_1} \in \cW_1$, $\phi: \cB_1^d \to \cB_1^d$
    \For{$t=1,2,\dots$}
    \State{}\textbf{Recieve } $x_t$,
    \textbf{predict}
    \begin{align*}
    \yhat_t = \sign(\inprod{\widetilde w_t}{\phi(x_t)}), \tag{\algcomment{$\selfdot\classify(x_t)$}}
    \end{align*}
    \State{}\textbf{Update} $\widetilde \cW_{t+1} = \widetilde \cW_t \cap \{\widetilde w \in \cB_1^{d+1}| \inprod{\widetilde w}{ \phi(x_t) y_t} \geq 0 \}$
    \If{$\yhat_t \ne y_t$} \quad\qquad\qquad\qquad\qquad\qquad\qquad\qquad\qquad\qquad(\algcomment{$\selfdot\errupdate(x_t)$})
    \State{} $\widetilde w_{t+1} \gets \johnellipsoid(\widetilde \cW_{t+1})$ \\
    \qquad\qquad \algcomment{returns center of John Ellpsoid of given convex body}
    \EndIf
    \EndFor
    \end{algorithmic}
      \caption{Binary Classification with Nonlinear Features}
      \label{alg:samedim}
    \end{algorithm}
\begin{lemma}\label{lem:samedimensionfeatures}
    Suppose that $p$ is a measure on $\cB_1^d$ that is $\sigma$-smooth with respect to $\mu_d$ and suppose that $\phi: \cB_1^d \to \cB_1^d$ is a function satisfying the following two properties:
    \begin{itemize}
        \item There is some $c > 0$ such that $\abs{\det(D \phi(x))} > c$ for all $x \in \cB_1^d$.
        \item There is some $N \in \bbN$ such that for every $x \in \cB_1^d$, it holds that $\abs{\phi^{-1}(x)} \leq N$, where $\phi^{-1}(x) = \left\{y \in \cB_1^d | \phi(y) = x \right\}$.
    \end{itemize}
    Then, $\phi_\ast p$ is $\left(\frac{c}{N} \sigma\right)$-smooth with respect to $\mu_d$.
\end{lemma}
\begin{proof}
     By \Cref{lem:smoothnesspushforward}, we have that $\phi_\ast p$ is $\sigma$-smooth with respect to $\phi_\ast \mu_d$.  Thus, as
     \begin{equation*}
         \frac{\rmd \phi_\ast p}{\rmd \mu_d} = \frac{\rmd \phi_\ast p}{\rmd \phi_\ast \mu_d} \cdot \frac{\rmd \phi_\ast \mu_d}{\mu d}
     \end{equation*}
     it suffices to bound the latter factor.  By the area formula \citep{federer2014geometric}, we have for any $B \subset \cB_1^d$ that
     \begin{align*}
         \phi_{\ast}\mu_d(B) &= \int_{\phi^{-1}(B)} \rmd \mu_d(x) \\
         &= \int_{\phi^{-1}(B)} \frac{\abs{\det(\rmD \phi(x))}}{\abs{\det(\rmD \phi(x))}} \rmd \mu_d(x) \\
         &\leq \frac 1c \int_{\phi^{-1}(B)} \abs{\det(\rmD \phi(x))} \rmd \mu_d(x)  \\
         &= \frac 1c \int_B \abs{\phi^{-1}(y)} \rmd \mu_d(y) \leq \frac{N}{c} \mu_d(y)
     \end{align*}
     Thus we see that for any $B$,
     \begin{equation*}
         \frac{\phi_\ast \mu_d(B)}{\mu_d(B)} \leq \frac Nc
     \end{equation*}
     and so the result follows.
\end{proof}
As a corollary, we generalize \Cref{prop:warmup} to adversaries that are now realizable to a class linear in some new set of features:
\begin{corollary}\label{cor:samedimensionfeatures}
    Let $\phi$ be a map as in \Cref{lem:samedimensionfeatures} and suppose that we are in the smoothed online learning setting with an adversary realizable with respecto to $\F \circ \phi = \left\{x \mapsto f(\phi(x)) | f \in \F\right\}$.  If we run \Cref{alg:samedim} on the data $(\phi(x_t), y_t)$, then for all $T$, with probability at least $1 - \delta$, it holds that
    \begin{equation*}
        \reg_T \leq 136 d \log(d) + 34 \log\left(\frac Nc\right) + 34 \log\left(\frac T{\sigma \delta}\right) + 56
    \end{equation*} 
\end{corollary}
\begin{proof}
    The statement follows immediately from applying \Cref{prop:warmup} to the data sequence $(\phi(x_t), y_t)$ and using \Cref{lem:samedimensionfeatures} to bound the smoothness.
\end{proof}
Finally, we prove the simpler result stated in \Cref{sec:generalizations}:
\begin{proof}[Proof of \Cref{prop:samedimensionfeatures}]
    By \Cref{cor:samedimensionfeatures}, it suffices to bound $N$ and $c$ in \Cref{lem:samedimensionfeatures}.  Suppose that $\phi(x) = (\phi_1(x_1), \dots, \phi_d(x_d))$ as in the statement of the result.  Then we see that $\rmD\phi(x)$ is diagonal with $\phi_i'(x_i)$ as the $i^{th}$ element of the diagonal and thus
    \begin{equation*}
        \det(\rmD \phi(x)) = \prod_{i = 1}^d \abs{\phi_i'(x_i)} \geq \alpha^d
    \end{equation*}
    where the final inequality follows from the assumption.  Note that if $\phi_i' > 0$ for all $i$, then $\phi$ is strictly increasing coordinate wise and thus we may take $N = 1$.  The result follows.
\end{proof}

\subsection{Proof of Theorem \ref{prop:featuremap}} \label{app:polies}
In this section, we show that our techniques extend to polynomial decision boundaries.  Morally, we proceed on similar lines as to the proof of \Cref{prop:samedimensionfeatures} outlined in the previous section, but there are a number of new technical subtleties that appear in this analysis that were not present before.  The most salient difference between the maps considered above and that which is required for a polynomial decision boundary is that polynomial features require imbedding our problem into a higher dimensional space; while \Cref{lem:smoothnesspushforward} ensures that the pushforward of the law of each $x_t$ is smooth with respect to the pushforward of $\mu_d$, our analysis is very specific to the dominating measure being uniform on the ball, which can never happen if we are pushing $\mu_d$ forward into a higher dimensional space.  In order to resolve this difficulty, we will present a reduction that allows us to combine multiple points into one `meta-point,' whose law will be smooth with respect to the uniform measure on the higher dimensional ball.  We will then be able to reduce to a similar setting as considered in \Cref{prop:warmup} and deduce a similar regret bound.  We prove the following result:
\begin{proposition}\label{prop:polynomialimbedding}
    Suppose that $\phi: \cB_1^d \to \cB_1^m$ satisfies the following properties:
    \begin{itemize}
        \item $\phi$ is $L$-Lipschitz.
        \item $\phi$ is polynomial in the sense that each of the coordinates of $\phi$ is a polynomial in the coordinates of $x \in \cB_1^d$ with degree at most $\ell$.
        \item There is an $\alpha > 0$ such that the Jacobian $\rmD \phi$ satisfies:
        \begin{equation*}
            \det\left(\ee_{x \sim \mu_d}\left[\rmD \phi(x) \rmD \phi(x)^T\right]\right) \geq \alpha^2
        \end{equation*}
    \end{itemize}
    Suppose further that the $x_t \in \cB_1^d$ are generated in a $\sigma$-smooth manner and the $y_t$ are realizable with respect to $\Flin^m \circ \phi$.  Then there is a universal constant $C$ such that for all $T$, if we set
    \begin{equation*}
        p = C m \ell \log\left(\frac{L \ell T}{\delta}\right)
    \end{equation*}
    and run \Cref{alg:polynomials}, then with probability at least $1 - \delta$,
    \begin{equation*}
        \reg_T \leq C \left(m \log(m) + \log\left(\frac 1\alpha\right) + \ell^2 m^2 d \log^2\left(\frac{d \ell T L}{\sigma \delta}\right) \right).
    \end{equation*}
\end{proposition}
\begin{proof}
    Consider the sequence of stopping times $\rho_\tau$, where $\rho_0 = 0$ and, for $\tau >0$,
    \begin{equation*}
        \rho_\tau = \inf\left\{t > \rho_{\tau-1} | \max\left(\sum_{s=  \rho_{\tau - 1}}^t \bbI[y_t = 1 \text{ and } \yhat_t = -1],\sum_{s=  \rho_{\tau - 1}}^t \bbI[y_t = -1 \text{ and } \yhat_t = 1] \right) \geq p \right\}.
    \end{equation*}
    for some $p$ to be determined.  Furthermore, let
    \begin{equation*}
        \ell_t = \bbI\left[t \in (2 p + 1)\bbN \text{ and } t- 2p \leq \rho_\tau \leq t \text{ for some } \tau\right]
    \end{equation*}
    We begin by claiming that the following inequality holds:
    \begin{equation}
        \reg_T \leq (2p + 1)\left(1 + \sum_{t = 1}^T \ell_t\right) \leq (2p + 1)\left(1 + \sum_{t' = 1}^{\left\lfloor \frac T{2p+1} \right\rfloor} \ell_{(2p+1)t'}\right). \label{eq:polynomialproof}
    \end{equation}
    Indeed, we note that the sum is equal to $\tau_T$, the maximal $\tau$ such that $\rho_\tau \leq T$ and the pigeonhole principle tells us that we suffer at most $2p + 1$ mistakes in the interval $\rho_{\tau - 1} \leq t \leq \rho_\tau$.  There are at most $2p$ mistakes in the interval $\rho_{\tau_T} \leq t \leq T$ and so the first inequality holds.  The second inequality follows from noting that $\ell_t = 1$ implies that $t = (2p + 1) t'$ for some $t'$.  For each $1 \leq \tau \leq \tau_T$, we let
    \begin{equation*}
        \overline{x}_\tau = \frac 1p \sum_{\substack{\rho_{\tau - 1} < t \leq \rho_\tau \\ y_t = y_\tau \text{ and } \yhat_t \neq y_t}} \phi(x_t)
    \end{equation*}
    Now, fix $\beta, \gamma > 0$ to be set later and let
    \begin{align*}
        \cU_1 &= \left\{\text{for all } t \text{ such that }  \pp_t(\yhat_t \neq y_t) < \gamma, \text{ it holds that } \yhat_t = y_t \right\} \\
        \cU_2 &= \left\{\text{for all } t \text{ such that } \pp_t(y_t = y) < \beta \text{ for some } y \text{, it holds that } y_t \neq y \right\}.
    \end{align*}
    We claim that for some $p$, there is a sequence of $\overline{x}_\tau' \in \cB_1^m$ such that if
    \begin{align*}
        \cU &= \left\{\widetilde{x}_\tau = \widetilde{x}_\tau' \text{ for all } \tau\right\} \\
        \cU &= \cU_1 \cap \cU_2 \cap \cU_3
    \end{align*}
    then first, $\pp(\cU) \geq 1 - 2 T \beta - T \gamma - \frac \delta 4$ and second, the $\overline{x}_\tau'$ are $\sigma'$-smooth with respect to $\mu_m$.  This claim, and the dependence of $\sigma'$ on the relevant parameters is the subject of \Cref{prop:polynomialmasterprop} below.  For now, we will take it as given.  Now, recalling the disagreement region and version space notation $D_t, \F_t$, as from \Cref{app:warmup}, we note that if $\ell_{t'} = 1$ and $\tau$ is maximal subject to $\rho_\tau \leq (2p + 1) t'$, then we must have $\overline{x}_\tau \in D_{\rho_{\tau - 1}}$.  To see this, note that $w_t = w_{\rho_{\tau - 1}}$ for $\rho_{\tau - 1} < t \leq \rho_\tau$ and thus $w_{\rho_{\tau - 1}}$ is such that $\inprod{w_{\rho_{\tau - 1}}}{y_s \phi(x_s)} < 0$ for each such $s$.  By linearity, we have
    \begin{equation*}
        \inprod{w_{\rho_{\tau - 1}}}{\overline{y}_\tau \overline{x}_\tau} = \frac 1p \sum_{i = 1}^p \inprod{w_{\rho_{\tau - 1}}}{y_{\tau_i}\phi(x_{\tau_i})} < 0
    \end{equation*}
    Realizabilty implies that there is some $w$ such that $\inprod{w}{y_s \phi(x_s)} \geq 0$ for all $s$, and so linearity implies that, for that $w$,
    \begin{equation*}
        \inprod{w}{\overline{y}_\tau \overline{x}_\tau} = \frac 1p \sum_{i = 1}^p \inprod{w}{\overline{y}_\tau \overline{x}_\tau} \geq 0
    \end{equation*}
    Thus we see that $\overline{x}_\tau \in D_{\rho_{\tau - 1}}$.  We now note that for any $t'$,
    \begin{equation*}
        \pp_{\rho_{\tau - 1}}(\ell_{(2p + 1)t'} = 1) \leq \frac{(2p + 1) \mu_m(D_{\rho_{\tau - 1}})}{\sigma'}
    \end{equation*}
    where $\sigma'$ is as in \Cref{prop:polynomialmasterprop}.  Applying now \Cref{lem:inclusioninorthogonal,lem:orthogonality,lem:surfacearea} in the same way as in the proof of \Cref{prop:warmup}, we place ourselves now into the situation of \Cref{lem:generalizedmasterreduction}, the generalized version of the master reduction \Cref{lem:masterreduction}.  Thus, we have that for all $T$, with probability at least $1 - \delta$,
    \begin{align*}
        \sum_{t' = 1}^{\left\lfloor \frac T{2p+1} \right\rfloor} \ell_{(2p+1) t'} &\leq 4 \frac{\log\left(\frac{2 T (2p + 1) 4^{m+1} m^{2m}}{(2p + 1) \sigma ' \delta}\right)}{\log\left(\frac 98\right)}  + \frac{e -1}{1 - \sqrt{\frac 89}} \\
        &\leq C\left(1 + m \log(m) + \log\left(\frac T\delta\right)  + \log\left(\frac 1{\sigma'}\right)\right) 
    \end{align*}
    If we set
    \begin{equation*}
        p = C m \ell \log\left(\frac{L \ell T}{\delta}\right)
    \end{equation*}
    then plugging in the penultimate display into \eqref{eq:polynomialproof}, taking $\gamma = \frac{\delta}{4 T}$ and $\beta = \frac{\delta}{8 T}$, and plugging in the bounds from \Cref{prop:polynomialmasterprop} concludes the proof.

\end{proof}

\begin{algorithm}[!t]
    \begin{algorithmic}[1]
    \State{}\textbf{Initialize } $\widetilde \cW_1 = \cB_1^{m}$, $\widetilde w_1 = \mathbf{e_1} \in \cW_1$, $\phi: \cB_1^d \to \cB_1^m$, $\cM_1, \cM_{-1} = \{\}$, $p \in \bbN$
    \For{$t=1,2,\dots$}
    \State{}\textbf{Recieve } $x_t$,
    \textbf{predict}
    \begin{align*}
    \yhat_t = \sign(\inprod{\widetilde w_t}{\phi(x_t)}), \tag{\algcomment{$\selfdot\classify(x_t)$}}
    \end{align*}
    \State{}\textbf{Update} $\widetilde \cW_{t+1} = \widetilde \cW_t \cap \{\widetilde w \in \cB_1^{d+1}| \inprod{\widetilde w}{ \phi(x_t) y_t} \geq 0 \}$
    \If{$\yhat_t \ne y_t$} \quad\qquad\qquad\qquad\qquad\qquad\qquad\qquad\qquad\qquad(\algcomment{$\selfdot\errupdate(x_t)$})
    \State{}\textbf{Update} $\cM_{y_t} \gets \cM_{y_t} \cup \{x_t \}$
    \EndIf
    \If{$\max\left(\abs{\cM_1}, \abs{\cM_{-1}}\right) = p$}
    \State{} \textbf{Update}
    \State{} $\widetilde w_{t+1} \gets \johnellipsoid(\widetilde \cW_{t+1})$ \\
    \qquad\qquad \algcomment{returns center of John Ellpsoid of given convex body}
    \State{} \textbf{Reset} $\cM_1, \cM_{-1} \gets \{\}$
    \EndIf
    \EndFor
    \end{algorithmic}
      \caption{Binary Classification with Polynomial Features}
      \label{alg:polynomials}
    \end{algorithm}

We note that \Cref{prop:featuremap} follows immediately from \Cref{prop:polynomialimbedding}.  The key difficulty in the proof of \Cref{prop:polynomialimbedding}, that we left until now, is the smoothness of the $\overline{x}_\tau$.  We state this fact, and provide a quantitative bound on the smoothness parameter, in the next proposition:
\begin{proposition}\label{prop:polynomialmasterprop}
    Let $\phi: \cB_1^d \to \cB_1^m$ be an $L$-Lipschitz function whose coordinates are polynomials in the coordinates of $x \in \cB_1^d$, with degree at most $\ell$.  Suppose $\phi$ is such that
    \begin{equation*}
        \det\left(\ee_{x \sim \mu_d}\left[\rmD \phi(x) \rmD \phi(x)^T\right]\right) \geq \alpha^2.
    \end{equation*}
    Fix any $p \in \bbN$ such that $p d \geq m$.  Suppose that $(x_1, y_1), \dots, (x_T, y_T)$ is a data sequence satisfying the following four properties:
    \begin{itemize}
        \item The distribution of $x_t$ conditional on the history is $\sigma$-smooth with respect to $\mu_d$.
        \item For $y \in \{\pm 1\}$, for any $t$, $\pp_t(y_t = y) \geq \beta$.
        \item The $y_t$ are realizable with respect to $\Flin^m \circ \phi$.
        \item For any $t$, and any choice of $\yhat_t$ by the learner, $\pp_t\left(\yhat_t \neq y_t\right) \geq \gamma$.
    \end{itemize}
    where $\pp_t$ is the conditional probability of the history up to time $t$.  Now, consider the set of stopping times $\rho_\tau$ with $\rho_0 = 0$ and
    \begin{equation*}
        \rho_\tau = \inf\left\{t > \rho_{\tau-1} | \max\left(\sum_{s=  \rho_{\tau - 1}}^t \bbI[y_t = 1 \text{ and } \yhat_t = -1],\sum_{s=  \rho_{\tau - 1}}^t \bbI[y_t = -1 \text{ and } \yhat_t = 1] \right) \geq p \right\}
    \end{equation*}
    Let $\overline{y}_\tau = y_{\rho_\tau}$ and let
    \begin{equation*}
        \overline{x}_\tau = \frac 1p \sum_{k = 1}^p \phi\left(x_{\tau_k}\right)
    \end{equation*}
    where $\tau_1, \dots, \tau_p$ are the $p$ times $\rho_{\tau - 1} < t \leq \rho_\tau$ satifying $y_t = \overline{y}_\tau$ and $\yhat_t \neq y_t$.  There is a universal constant $C$ such that if 
    \begin{equation*}
        p \geq C m \ell \log\left(\frac{L \ell T}{\delta}\right)
    \end{equation*}
    then there is a data sequence a sequence $\overline{x}_t' \in \cB_1^m$ satisfying the following three properties.  First, the sequence $(\overline{x}_t', \overline{y}_t)$ is realizable with respect to $\Flin^m$.  Second, with probability at least $1 - \delta$, for all $t$, $\overline{x}_t = \overline{x}_t'$.  Third, the $\overline{x}_t'$ is $\sigma'$-smooth with respect to $\mu_m$, where
    \begin{equation}
        \sigma' = c \cdot \alpha \cdot p^{- \frac m2} \left(\frac{\beta \gamma \sigma}T\right)^{2 \ell m p} \ell^{-m\left(\ell + pd\right)} d^{- pd} \label{eq:imbeddedsigma}
    \end{equation}
    wisth $c$ a universal constant
\end{proposition}
Intuitively, we wait until we have misclassified a class $p$ times and then form a `meta-point' $(\widetilde{x}_\tau, \widetilde{y}_\tau)$ that will allow us to reduce to the setting of \Cref{prop:warmup}.  The meta-point will be constructed by averaging samples $x_t$ in order to ensure smoothness with respect to $\mu_m$.

We will show that \Cref{prop:polynomialmasterprop} follows from three results that we will prove below.  First, we show that if $\phi$ is well-behaved, and the $x_t$ are $\sigma$-smooth with respect to $\mu_d$, then the $\widetilde{x}_t$ are $\sigma'$-smooth with respect to $\mu_m$.
\begin{proposition}\label{prop:imbeddingissmooth}
    Suppose that $\phi: \cB_1^d \to \cB_1^m$ is a smooth map between Euclidean balls of dimensions $d$ and $m$ with Jacobian $\rmD \phi$.  Consider the function $f: (\rr^d)^p \to \rr^m$ defined as
    \begin{equation}
        \psi(x_1, \dots, x_p) = \frac 1p \sum_{i = 1}^p \phi(x_i)
    \end{equation}
    Suppose that that the following three conditions are satisfied:
    \begin{itemize}
    \item There is some $V \subset (\cB_1^d)^{\times p}$ and $c > 0$ such that for $\mu_d^{\otimes p}$-almost every $x \in V$, $\det(\rmD\psi(x) \rmD\psi(x)^T) \geq \alpha^2$.
    \item  For some $\ell \geq 2$
    \begin{equation}
        \sup_{x \in \cB_1^d} \max_{\substack{\abs{\nu} = \ell+1 \\ 1 \leq i \leq m }} \abs{\partial^\nu \phi_i(x)} \leq 2^{- (1+\ell)}
    \end{equation}
    In particular, this holds if $\phi(x)$ is a polynomial of degree at most $\ell$. 
    \item 
    Finally, suppose that the joint distribution of $(x_1, \dots, x_p)$ is $\sigma^p$-smooth with respect to $\mu_{d}^{\otimes p}$.  
\end{itemize}
If $pd \geq m$, then the law of $\psi(x_1, \dots, x_p)$, conditioned on $(x_1, \dots, x_p) \in V$ is $\sigma'$-smooth with respect to $\mu_m$, the uniform measure on $\cB_1^m$, where
    % \begin{equation}
    %     \sigma' = \left(\frac{c \sigma}{p}\right)^p \frac {\mu_d(V)}{\ell^{2 m + m p d} d^{pd}}
    % \end{equation}
    \begin{equation}
        \sigma' =  \frac {\alpha \sigma^p \cdot \pp\left((x_1, \dots, x_p) \in V\right) }{\ell^{2 m + m p d} d^{pd}}
    \end{equation}
\end{proposition}
Second, we show that if $\phi$ is a polynomial, then it is well-behaved in the sense of \Cref{prop:imbeddingissmooth} with high probability, by proving the more general small-ball type estimate below:
\begin{proposition}\label{prop:anti_conc} There exists a univesal constant $C$ such that the following holds. Let $\Psi:\R^d \to \bbS_{+}^D$ be any function whose image is contained in the set of PSD matrices and whose entries are polynomials of degree at most $\ell$, and let $x_1,\dots,x_p$ be a sequence of random-variables such that, for each $t$, $x_t \mid x_{1},\dots,x_{t-1}$ is $\sigma$-smooth with respect to a common log-concave measure $\mu$, and $\Pr_{x \sim \mu}[\lambda_{\max}(\Psi(x)) \le B] = 1$. Define
\begin{align*}
\Lambda := \Exp_{x \sim \mu}[\Psi(x)]
\end{align*}
Suppose that  $p \ge 16\log(1/\delta) +  \frac{D}{4}\log (24B) + \frac{1}{4}(\log \det (\Lambda) + D\ell\log (C\ell))$. Then, 
\begin{align*}
\Pr\left[\det\left(\frac{1}{p}\sum_{i=1}^p \Psi(x_i)\right)\le \left(\frac \sigma{C\ell}\right)^{\ell D}\det(\Lambda)\right] \le   \delta.
\end{align*}
\end{proposition}
Finally, we will show that if the probabilities corresponding to each label are well-controlled, then the laws of $x_{\tau_k}$ are smooth with respect to $\mu_d$:
\begin{proposition}\label{prop:conditioningsmooth}
    Let $(x_{\tau_1}, y_{\tau_1}), \dots, (x_{\tau_{p}}, y_{\tau_{p}})$ be the sequence of points defined in \Cref{prop:polynomialmasterprop}, arising from a sequence of $(x_t, yt)$ with the $x_t$ being $\sigma$-smooth conditional on the history and for each $t$, and $y \in \{\pm 1\}$ it holds that $\pp(y_t = y )  \geq \beta$ and that $\pp(\yhat_t \neq y_t) \geq \gamma$.  Then, for each $i$, it holds that the law of $x_{\tau_i}$ conditional on the history up to time $\tau_{i-1}$ is $(\beta \gamma \sigma / T)$-smooth with respect to $\mu_d$.  In particular, the law of $(x_{\tau_1}, \dots, x_{\tau_p})$, conditional on the sigma-algebra generated by $\rho_{\tau-1}$, is $(\beta \gamma \sigma / T)^{p}$-smooth with respect to $\mu_d^{\otimes p}$.
\end{proposition}
\begin{proof}
    Note that a peeling argument and induction show that the second statement follows immediately from the first.  For any $\tau, i$, denote probability conditioned on the history up to time $\tau_{i-1}$ by $\pp_{\tau_{i-1}}$.  Let $B \subset \cB_1^d$ be measurable.  Then we compute that $\pp_{\tau_{i-1}}\left(x_{\tau_i} \in B\right)$ can be given by:
    \begin{align*}
         &\sum_{1 \leq t \leq T} \pp_{\tau_{i-1}}(t = \tau_i) \pp_{\tau_{i-1}}\left(x_t \in B | \tau_i = t\right) \\
        &\leq \sum_{1 \leq t \leq T} \pp_{\tau_{i-1}}(\tau_i > t - 1) \pp_{\tau_{i-1}}(x_t \in B | y_t \neq \yhat_t \text{ and } y_t = \widetilde{y}_\tau) \\
        &\leq \sum_{1 \leq t \leq T} \pp_{\tau_{i-1}}(\tau_i > t - 1) \frac{\pp_{\tau_{i-1}}(x_t \in B |  y_t = \widetilde{y}_\tau) }{\gamma} \\
        &\leq \sum_{1 \leq t \leq T} \pp_{\tau_{i-1}}(\tau_i > t - 1) \tfrac{\pp_{\tau_{i-1}}(x_t \in B |  y_t = 1, \, \widetilde{y}_\tau = 1)\pp_{\tau_{i-1}}(\widetilde{y}_\tau = 1) + \pp_{\tau_{i-1}}(x_t \in B |  y_t = -1,\, \widetilde{y}_\tau = -1) \pp_{\tau_{i-1}}(\widetilde{y}_\tau = 1)}{\gamma} \\
        &= \sum_{1 \leq t \leq T} \pp_{\tau_{i-1}}(\tau_i > t - 1) \frac{\pp_{\tau_{i-1}}(x_t \in B )}{\beta\gamma} \\
        &\leq \frac{T \mu_d(B)}{\beta \gamma \sigma}
    \end{align*}
    Thus, the result follows.
    % Let $\rho_\tau^i$ denote the time of the $i^{th}$ mistake after $\rho_\tau$.  Then, conditioning on the sigma-algebra generated by $\rho_\tau$, the distribution of $x_{\rho_\tau^i}$ is $\beta \gamma \sigma$-smooth with respect to $\mu_d$.  Indeed, we see that, if $\pp_\tau$ denotes conditioning on $\rho_\tau$, we compute
    % \begin{align*}
    %     \pp_\tau(x_{\rho_\tau^i} \in B) &\leq \pp_\tau\left(x_t \in B | \yhat_t \neq y_t \text{ and } y_t = \widetilde{y}_\tau\right) \\
    %     &= \pp_\tau\left(x_t \in B | \yhat_t \neq y_t \text{ and } y_t = 1\right) + \pp_\tau\left(x_t \in B | \yhat_t \neq y_t \text{ and } y_t = -1\right) \\
    %     &\leq \frac{\pp_\tau\left(x_t \in B | y_t = 1\right) + \pp_\tau\left(x_t \in B | y_t = -1\right)}{\gamma} \\
    %     &\leq \frac{\pp_\tau(x_t \in B)}{\beta \gamma} \\
    %     &\leq \frac{\mu_d(B)}{\beta \gamma \sigma}
    % \end{align*}
    % as desired.  Let $n_\tau$ be the number of mistakes made between times $\rho_{\tau}$ and $\rho_{\tau + 1}$ and note that the pigeonhole principle implies that $n_\tau \leq 2p + 1$.  Thus we may consider $\left(x_{\rho_\tau^1}, \dots x_{\rho_\tau^{n_\tau}, \dots, x_{\rho_\tau^{2p + 1 - n_\tau}}}\right)$, the tuple of $x_{\rho_\tau^i}$ possibly padded by future $x_t$'s so that there are exactly $2 p + 1$ of these.  Then we see by induction and peeling that this tuple has joint distribution that is $(\beta \gamma \sigma)^{2p + 1}$-smooth with respect to $\mu_d^{\otimes (2p + 1)}$.  Considering the projection onto the relevant $p$ coordinates and applying \Cref{lem:smoothnesspushforward} yields the desired result.
\end{proof}
\Cref{prop:imbeddingissmooth,prop:anti_conc} will be shown below but for now we will take them as given.  We can now prove the key proposition:
\begin{proof}[Proof of \Cref{prop:polynomialmasterprop}]
    Let
    \begin{equation*}
        \cE = \left\{\det\left(\frac 1p \sum_{i = 1}^p (\rmD \phi \rmD \phi^T)(x_{\tau_i})\right)  > \left(\frac{\beta \gamma \sigma}{T}\right)^{(2p + 1) \ell m} (C \ell)^{- \ell m} \alpha^2 \text{ for all $\tau$}\right\}
    \end{equation*}
    and note that by the fact that $\tau \leq T$
    \begin{equation*}
        p \geq C m \log\left(\frac{\ell L \alpha T}{\delta}\right),
    \end{equation*}
    applying a union bound to \Cref{prop:anti_conc} and using \Cref{prop:conditioningsmooth} to ensure that the hypothesis holds, shows that $\pp(\cU) \geq 1 - \delta$.  On $\cU$ we will let $\overline{x}_\tau' = \overline{x}_\tau$ and on $\cU^c$, we will draw $\overline{x}_\tau'$ from $\mu_m$, conditioned on $(\overline{x}_\tau', \overline{y}_\tau)$ being realizable with respect to $\Flin^m$.  Note that we have realizability by construction on $\cU^c$.  On $\cU$, we have that $\overline{x}_\tau' = \overline{x}_\tau$ and note that convexity implies that if $y_{\tau_1} = \cdots = y_{\tau_p}$, then any realizable adversary must classify $\overline{x}_\tau$ as $\overline{y}_\tau$.  Indeed, if $w \in \Flin^m$ is in the version space, and $\overline{y}_\tau = 1$, then
    \begin{equation*}
        \inprod{w}{\frac 1p \sum_{i = 1}^p \phi(x_{\tau_i})} = \frac 1p \sum_{i = 1}^p \inprod{w}{\phi(x_{\tau_i})} > 0
    \end{equation*}
    and similarly if $\overline{y}_\tau = -1$.  Thus, realizability holds.  As we have already seen that
    \begin{equation*}
        \left\{\text{there exists } \tau \text{ such that } \widetilde{x}_\tau' \neq \overline{x}_\tau\right\} \subset \cE^c
    \end{equation*}
    and $\pp(\cU^c) \leq \delta$, it suffices to show smoothness of $\overline{x}_\tau'$.  On $\cU^c$, the construction implies that $\overline{x}_\tau'$ are smooth, so we now restrict to the event $\cE$.  We first compute the Jacobian of $\psi$:
    \begin{equation*}
        \rmD \psi(x_1, \dots, x_p) = \frac 1p\begin{bmatrix}
            \rmD \phi(x_1) & \rmD \phi(x_2) & \cdots & \rmD \phi(x_p)
        \end{bmatrix}
    \end{equation*}
    and thus
    \begin{equation*}
        \rmD \psi \rmD \psi^T = \frac 1{p^2} \sum_{i = 1}^p D \psi(x_i) \rmD \psi(x_i)^T
    \end{equation*}
    which in turn implies:
    \begin{equation*}
        \det(\rmD \psi \rmD \psi^T) = p^{- m} \det\left(\frac 1{p} \sum_{i = 1}^p \rmD \psi(x_i) \rmD \psi(x_i)^T\right).
    \end{equation*}
    Thus, under $\cU$, we have that
    \begin{equation*}
        \det((\rmD \psi \rmD \psi^T)(x_{\tau_1}, \dots, x_{\tau_p})) \geq p^{-m} \left(\frac{\beta \gamma \sigma}{T}\right)^{(2p + 1) \ell m} (C \ell)^{- \ell m} \alpha^2 = \widetilde{\alpha}^2
     \end{equation*}
     We may now use \Cref{prop:conditioningsmooth} to get that $(x_{\tau_1}, \dots, x_{\tau_p})$ has a law that is $(\beta \gamma \sigma / T)^p$ smooth with respect to $\mu_d^{\otimes p}$ and apply \Cref{prop:imbeddingissmooth} to get that, conditional on $\cE$, the law of $\psi(x_{\tau_1}, \dots, x_{\tau_p})$ is $\sigma'$-smooth with respect to $\mu_m$, where
     \begin{equation*}
         \sigma' = \frac{\alpha' \left(\frac{\beta \gamma \sigma}{2T}\right)^p}{\ell^{2m + mpd} d^{pd}}
     \end{equation*}
     where we let $V = \cU$ and note that $\pp((x_{\tau_1}, \dots, x_{\tau_p}) \in \cU) \geq \frac 12$.  The result follows.
\end{proof}

% \begin{algorithm}[!t]
%     \begin{algorithmic}[1]
%     \State{}\textbf{Initialize } $\widetilde \cW_1 = \cB_1^m$, $\widetilde w_1 = \mathbf{e_1} \in \widetilde \cW_1$, $\phi: \cB_1^d \to \cB_1^m$, $\tau = 1$, $p \in \bbN$, $\cM_1,\cM_2 = \{\}$.
%     \For{$t=1,2,\dots$}
%     \State{}\textbf{Recieve } $x_t$, and \textbf{predict}
%     \begin{align*}
%     \yhat_t = \sign(\inprod{\widetilde w_t}{\phi(x_t)}), \tag{\algcomment{$\selfdot\classify(x_t)$}}
%     \end{align*}
%     \State{}\textbf{Update} $\cW_{t+1} = \widetilde \cW_t \cap \{w \in \cB_1^d| \inprod{w}{\phi(x_t) y_t} \geq 0 \}$
%     \If{$\yhat_t \neq y_t$}
%         \State{} \textbf{Set} $\cM_{y_t} \gets \cM_{y_t} \cup {x_t}$
%     \EndIf
%     \If{$\abs{\cM_y} = p$ for some $y \in \{\pm 1\}$}
%         \State{} \textbf{Set}
%         \begin{align*}
%             \widetilde{y}_\tau &= y \\
%             \widetilde{x}_\tau &= \frac 1p \sum_{x \in \cM_{\widetilde{y}_\tau}} \phi(x) \tag{\algcomment{$\selfdot\imbed(x_{\tau_1}, \dots, x_{\tau_p})$}}
%         \end{align*}
%         \State{} $M_1, M_{-1} \gets \{\}$
%     \EndIf
%     \If{$\yhat_t \ne y_t$} \quad\qquad\qquad\qquad\qquad\qquad\qquad\qquad\qquad\qquad(\algcomment{$\selfdot\errupdate(x_t)$})
%     \State{} $w_{t+1} \gets \johnellipsoid(\cW_{t+1})$ \\
%     \qquad\qquad \algcomment{returns center of John Ellpsoid of given convex body}
%     \EndIf
%     \EndFor
%     \end{algorithmic}
%       \caption{Binary Classification with Polynomial Thresholds}
%       \label{alg:polynomials}
%     \end{algorithm}

\subsubsection{Proof of Propossition \ref{prop:imbeddingissmooth}}

We will proceed by using the co-area formula \citep{federer2014geometric}.  For a given $x \in (B_1^d)^n$, let
\begin{equation}
    J(\psi)(x) = \sqrt[]{\det(\rmD\psi(x) \rmD\psi(x)^T)}
\end{equation}
and let $\cH^j$ denote the $j$-dimensional Hausdorff measure.  Then the co-area formula tells us that for any $B \subset V$, we have
\begin{equation}\label{eq:coarea}
    \int_{\psi^{-1}(B)} J(\psi)(x) \rmd \cH^{dp}(x) = \int_B \cH^{dp - m}(\psi^{-1}(y)) \rmd \cH^{m}(y)
\end{equation}
We make use of the following lemma:
\begin{lemma}\label{lem:fibreupperbound}
    Suppose that $\phi$ is as in \Cref{prop:imbeddingissmooth}.  Then for all $y$,
    \begin{equation*}
        \cH^{dp - m}(\psi^{-1}(y)) \leq \ell^{2m + m pd}
    \end{equation*}
\end{lemma}
Before proving the lemma, we require a preliminary result from \citet{yomdin1984set}; we reprove it here in order to keep track of the constant.
\begin{lemma}\label{lem:yomdinconstant}
    Fix $k \in \bbN$ and suppose that $Y \subset \cB_1^{k}$ is a hypersurface and let $\cB_r^k \subset \cB_1^k$.  Suppose that any line passing through $\cB_r^d$ intersects $Y$ in at most $\ell$ points.  Then,
    \begin{equation}
        \vol_{k-1}(Y) \leq  \frac{\ell 2 \pi^{\frac k2}}{\Lambda\left(\frac k2\right)} r^{-k}.
    \end{equation}
\end{lemma}
\begin{proof}
    Because $\cB_r^k$ is convex, we may consider the map $\pi: Y \to \partial \cB_r$ of projection to the boundary.  Recentering so that $\cB_r^k$ has center at the origin, we have $\pi(y) = r\frac{y}{\norm{y}}$.  By the co-area formula introduced as \eqref{eq:coarea}, we have
    \begin{align}
        \int_Y J(\pi)(y) \rmd \cH^{k-1}(y) = \int_{\partial \cB_r^k} \cH^0(\pi^{-1}(z)) \rmd \cH^{k-1}(z)
    \end{align}
    Now, note that $J(\pi(y)) \geq r^k$ and thus we have
    \begin{align}
        \vol_{k-1}(Y) = \int_Y \rmd \cH^{k-1}(y) \leq \frac{\int_{\partial \cB_r^k} \cH^0(\pi^{-1}(z)) \rmd \cH^{k-1}(z)}{r^k} \leq \frac{\ell 2 \pi^{\frac k2}}{\Lambda\left(\frac k2\right) r^k}
    \end{align}
    where the last inequality comes from combining the fact that by assumption $\vol_0(\psi^{-1}(z)) \leq \ell$ and the expression for the surface area of $S^{d-1}$.
\end{proof}
\begin{proof}[Proof of \Cref{lem:fibreupperbound}]
    We apply \citet[Theorem 3 (iii)]{yomdin1984set} iteratively on the coordinates of $\psi$.  In particular, we apply Lemma \ref{lem:yomdinconstant} in order to keep track of the explicit constant in \citet[Lemma 7]{yomdin1984set} and apply \citet[Lemma 4]{yomdin1984set} to show that we may choose $r = (20 \cdot \ell^2)^{-1}$ in Lemma \ref{lem:yomdinconstant}.  Thus we have for any $y \in \cB_1^m$,
    \begin{align}
        \vol_{pd - m}(f^{-1}(y)) &\leq \prod_{j = 0}^{m-1} \left(\frac{\ell 2 \pi^{\frac{pd - j}{2}}}{\Lambda\left(\frac{pd - j}{2}\right)} (20 \cdot \ell^2)^{pd - j}\right) \\
        &\leq \ell^{2m + mpd}
    \end{align}
    using the fact that $pd \geq m$.  The result follows.
\end{proof}
Returning now to the proof of \Cref{prop:imbeddingissmooth}, we see for a given set $B \subset V$, that
\begin{align}
    \pp\left(\psi(x_1, \dots, x_p) \in B | (x_1, \dots, x_p) \in V\right) &\leq \frac{\sigma^{-p} \vol_{dp}(\psi^{-1}(B))}{\omega_{d}^p \pp(V)} \\
    &\leq \frac{\sigma^{-p} \ell^{2m + mpd}}{\omega_d^p \pp(V)} \vol_m(B) \\
    &\leq  \frac{\sigma^{-p} \ell^{2m + m pd} \omega_m}{\omega_p^d \pp(V)} \mu_m(B)
\end{align}
where the first inequality follows from the definition of smoothness, the second inequality follows by \eqref{eq:coarea} and the above claims, and the last inequality follows by the definition of $\mu_m$.  The result follows by using the fact that
\begin{equation}
    \omega_m = \frac{\pi^{\frac m2}}{\Lambda\left(\frac{m}{2} + 1\right)}
\end{equation}
and bounding $\left(\frac m4\right)^{\frac m4} \leq \Lambda\left(\frac m2 + 1\right) \leq \left(\frac m2\right)^{\frac m2}$.

\subsubsection{Proof of Proposition \ref{prop:anti_conc}}
We first introduce a small-ball estimate for sums of PSD random variables, in the spirit of \cite{simchowitz2018learning}. 
\begin{lemma}\label{lem:small_ball} Let $X_1,X_2,\dots,X_p $ be i.i.d. of positive semi-definite, $\rr^{D \times D}$-valued random variables,  and suppose there exists $B,\eta> 0$ and $\Lambda \in \pd{D}$ for which, for all $t \in [p]$,
\begin{align*}
&\Pr[\lambda_{\max}(X_t) > B] = 0\\
&\Pr_{X_t}[ v^\top X v \ge v^\top \Lambda v \mid X_{1},\dots,X_{t-1}] \ge \eta, \quad  \forall v \in \rr^D.
\end{align*}
Then, if $p \ge 8\eta^{-1}\log(1/\delta) +  \frac{D}{4}\log (\frac{12B}{\eta}) + \frac{1}{4}\log \det (\Lambda^{-1})$,
\begin{align*}
\Pr\left[\frac{1}{p}\sum_{i=1}^p X_i \not\succeq \frac{\eta}{4}\cdot \Lambda\right] \le   \delta.
\end{align*}
In particular,
\begin{align*}
\Pr\left[\det\left(\frac{1}{p}\sum_{i=1}^p X_i\right)\le \left(\frac{\eta}{4}\right)^{D}\det(\Lambda)\right] \le   \delta.
\end{align*}
\end{lemma}

\begin{proof} The proof follows along the lines of \cite{simchowitz2018learning}, sharpened slightly for the less general setting.  Let $\Sigma = \sum_{i=1}^p X_i$. By a Chernoff bound (\Cref{lem:chernoff}), for any $v \in \cS^{D-1}$,  $ \Pr[v^\top \Sigma v \le \eta p v^\top \Lambda v/2] = \Pr[ \sum_{i=1}^p  v^\top X_i v  \le\eta p v^\top \Lambda v/2] \le \Pr[\sum_{i=1}^p \I\{v^\top X_i v  \le \eta v^\top \Lambda v\} \le \eta p/2] \le \exp(-\eta p/8)$, where we use that $X_i \succeq 0$. Hence, for any finite subset $\cT \subset \cS^{D-1}$, 
\begin{align*}
\Pr[v^\top \Sigma v \ge v^\top \Lambda v \cdot \frac{\eta p}{2}, \quad  \forall v \in \cT] \ge 1- \exp( - \eta p/8 + \log |\cT|).
\end{align*}
To conclude, we show that there exists a finite set $\cT$ of size at most $\exp(\frac{D}{2}\log (\frac{12B}{\eta}) + \frac{1}{2}\log \det (\Lambda))$ such that, if $v^\top \Sigma v \ge v^\top \Lambda v \cdot \eta T/2$ for all $\cT$, then $\Sigma \succeq \frac{\eta T}{4}$.

We take $\cT$ to be an $\epsilon = \sqrt{\eta/4B}$-net of the set $\cS_{\Lambda} := \{v \in \R^d: v^\top \Lambda v = 1\}$.  Then, if $v^\top \Sigma v \ge v^\top \Lambda v \cdot \eta p/2 = \eta p/2$ for all $\tilde v \in \cS_{\Lambda}$, it holds
\begin{align*}
\tilde v^\top \Sigma \tilde v &\ge \frac{1}{2}v^\top \Sigma v - (\tilde v-v)^\top \Sigma (\tilde v - v) \ge  \frac{1}{2}\eta T - Bp\|\tilde v - v\|^2 \ge \frac{\eta p}{4},
\end{align*}
which means that $\Sigma \succeq \frac{\eta p }{4}\Lambda$. Define the ellipsoid $\cE_{\Lambda} := \{v \in \R^d: v^\top \Lambda v \le 1\}\supset \cS_{\Lambda}$.  Note that since $\lambda_{\max}(\Lambda) \le B$, $\cE_{\Lambda} \supset \{v:\|v\|^2 \le 1/B\} \supset 2\epsilon \cB_1^D$, since $\epsilon = \sqrt{\eta/4B} \le \frac{1}{2\sqrt{B}}$.
\begin{align*}
|\cT| &\le \frac{\vol(\frac{\epsilon}{2}\cB_1^D + \cS_{\Lambda})}{\vol(\frac{\epsilon}{2}\cB_1^D)} \le \frac{\vol(\frac{\epsilon}{2}\cB_1^D + \cE_{\Lambda})}{\vol(\frac{\epsilon}{2}\cB_1^D)} \le \frac{\vol(\frac{5}{4}\cE_{\Lambda})}{\vol\left(\frac{\epsilon}{2}\cB_1^D\right)} = \left(\frac{8}{5\epsilon}\right)^D \det (\Lambda^{-1/2}).
\end{align*}
Hence, we can take
\begin{align*}
\log |\cT| &\le \frac{D}{2}\log (\frac{64}{25\epsilon^2}) + \frac{1}{2}\log \det (\Lambda) \le \frac{D}{2}\log (\frac{12B}{\eta}) + \frac{1}{2}\log \det (\Lambda). 
\end{align*}

\end{proof}

\begin{lemma}\label{lem:psd_anti_conc} Let $\Psi:\R^d \to \bbS_{+}^D$ be any function whose image are PSD matrices whose entries are polynomials of degree at most $\ell$. Let $\rho$ be any distribution which is $\sigma$-smooth with respect to a log-concave measure $\mu$. Then, there exist a universal constant $C$ such that, for any $v \in \R^D \setminus 0$, 
\begin{align*}
\Pr_{x \sim \rho}[ v^\top \Psi(x) v \le \sigma^{\ell}(C\ell)^{-\ell} v^\top \Exp_{x \sim \mu} \Psi(x) v ] \le \frac{1}{2}.
\end{align*}

\end{lemma}
\begin{proof} Consider the polynomial function $f_v(x) = v^\top \Psi(x) v$. This is a polynomial of degree $\ell$ in $x$, and nonnegative. By \citet[Theorem 8]{carbery2001distributional}, with $q =  \ell$, we have
\begin{align*}
\Exp_{x \sim \mu}[f_v(x)]^{1/\ell}\cdot \alpha^{-1/\ell} \cdot \Pr_{x \sim \mu}[ f_v(x) \le \alpha] \le C' \ell,
\end{align*}
where $C'$ is a universal constant. Reparametrizing $\alpha \gets \alpha  \cdot \Exp_{x \sim \mu}[f_v(x)] $, we have
\begin{align*}
\Pr_{x \sim \mu}[ f_v(x) \le \alpha \Exp_{x \sim \mu}[f_v(x)]] \le C' \ell \alpha^{1/\ell}.
\end{align*}
To conclude, take $\alpha = (2C'\ell/\sigma)^{-\ell}$, we get
\begin{align*}
\Pr_{x \sim \mu}[ f_v(x) \le  (2C'\ell/\sigma)^{-\ell} \Exp_{x \sim \mu}[f_v(x)]] \le \frac{\sigma}{2}
\end{align*}
Hence, 
\begin{align*}
\Pr_{x \sim \rho}[ f_v(x) \le  (2C'\ell/\sigma)^{-\ell} \Exp_{x \sim \mu}[f_v(x)]] \le \frac{1}{2}
\end{align*}
Taking $C =2 C'$ and substituting $f_v(x) = v^\top \Psi(x) v$ concludes. 
\end{proof}

\begin{proof}[Proof of \Cref{prop:anti_conc}] Let $\rho_t$ denote the conditional distribution of $x_t \mid x_{1},\dots,x_{t-1}$. By assumption, $\rho_t$ is $\sigma$-smooth with respect to the log-concave measure $\mu$, so 
\begin{align*}
\Pr[ v^\top \Psi(x_t) v \le (C\ell)^{-\ell} v^\top \Exp_{x \sim \mu} \Psi(x) v \mid x_1,\dots,x_{t-1}] \le \frac{1}{2}
\end{align*}
 Hence, we can apply \Cref{lem:small_ball} with $\eta \gets 1/2$, $B \gets B$, and $\Lambda \gets (C\ell)^{-\ell} \Lambda$. Using that $\det((C\ell)^{-\ell} \Lambda) = (C\ell)^{-D\ell}\det(\Lambda)$ concludes.  
\end{proof}

\newpage
%!TEX root = ../neurips_submission.tex

\renewcommand{\gst}{{g}^\star}
\newcommand{\bgst}{\mathbf{g}^\star}
\newcommand{\bg}{\mathbf{g}}
\newcommand{\erm}{\algfont{ERM}}
\newcommand{\wsti}{w_\star^i}
\newcommand{\wstj}{w_\star^j}
\newcommand{\wstij}{w_\star^{(i,j)}}
\newcommand{\wtij}{w_t^{(i,j)}}

\section{Proofs from Section \ref{sec:kpiece}}\label{app:kpiece}
In this section, we prove the extensions of our results to the multipiece setting.

\subsection{Proof of Theorem \ref{prop:kpiececlassification}\label{app:kpiececlassification}}

    \paragraph{Algorithm Description.} 
    \Cref{alg:K_class} gives our algorithm for $K$-class classification. We maintain $\binom{K}{2}$ instances of the binary classification algorithm, \Cref{alg:warmup}. That is, each $\Abin$ maintains a $w_{t}^{(i,j)}$ at each time $t$, and 
    \begin{align*}
     \Abin^{(i,j)}.\classify(x) = \sign(\langle \wtij,w \rangle). 
    \end{align*}

    To gain intuition, recall that we assume the ground-truth classifier to 
    \begin{align}
    \fst(x) = \argmax_{i \in [K]}\langle x,\wsti\rangle, \label{eq:argmax_app}
    \end{align}
    where the argmax is taken lexicographically. Hence, $\fst(x)$ admits the following equivalent representation:
    \begin{align*}
    \fst(x) &= \min_{i \in [K]}\{i:\langle x,w_\star^i\rangle \ge \max_{j > i}\langle x,\wstj\rangle\} \nonumber\\
    &= \min_{i \in [K]}\{i:\sign (\langle x,\wsti - \wstj \rangle  \ge,~ \forall i > i\} \numberthis \label{fst:tourney}
    \end{align*}
    Hence, $\fst(x)$ can be thought of running a lexicographic tournament, picking out the first index $i$ which `wins' over all lesser indices $k$. This is what motivates the selection of $\yhat$ in \Cref{line:pred} of \Cref{alg:K_class}.

    \begin{proof}[Proof of \Cref{prop:kpiececlassification}] 

    We reduce to the generalized, ``censored'' variation of our linear classification setting, depicted in \Cref{prop:warmupgeneralized}. For pairs $i < j$, define
    \begin{align*}
    \wstij := \wsti - \wstj \quad y_t^{(i,j)} := \sign(\langle \wstij, x_t\rangle), \quad \yhat_t^{(i,j)} := \Abin^{(i,j)}.\classify(x_t)
    \end{align*}
    Note that
    \begin{align}
    y_t^{(i,j)} = 1 ~~\forall j > i \text{ whenever } i = y_t. \label{eq:yt_eq}
    \end{align}
    For simplicty, lets assume that $w^{\star,(i,j)} \ne 0$ for $i < j$. We address the edgecase where this term may be zero at the end. 
    Further, let $i_t < j_t$ denote the indices select in \Cref{eq:ij_Kwise}. Then, since the algorithm always selects such a pair $(i_t,j_t)$ whenever a mistake is made (and defining, say $(i_t,j_t) = (0,0)$ to indicate no mistake),
    \begin{align*}
    \sum_{t=1}^T \I\{\yhat_t \ne y\} = \sum_{i<j}\sum_{t=1}^T \I\{(i_t,j_t) = (i,j)\}.
    \end{align*}
    The following claim reduces to binary-losses.
    \begin{claim}\label{claim:K_to_bin} For the indices $i_t < j_t$ selected in \Cref{eq:ij_Kwise}, and any $1 \le i < j \le K$, 
    \begin{align*}
    \I\{(i_t,j_t) = (i,j)\} = \I\{y_t^{(i,j)} \ne \yhat_t^{(i,j)}\}\I\{(i_t,j_t) = (i,j)\}.
    \end{align*}
    Moreover, when $(i_t,j_t) = (i,j)$, $y_t^{(i,j)} = \sign(\yhat_t - y_t)$, and thus can be determined by learner. 
    \end{claim}
    \begin{proof}
    Indeed, at a round where $\I\{(i_t,j_t) = (i,j)\}$, we have $\yhat_t \ne y_t$. We have two cases
    \begin{itemize}
    \item  When $y_t < \yhat_t$, then \Cref{eq:ij_Kwise} selects $i = y_t$ and $j$ as some index for which $\yhat_t^{(i,j)}  = -1$, such an index must exist by the choice of $\yhat_t$ in \Cref{line:pred} (otherwise, either $ \yhat_t < y_t$, or else $y_t$ would be correctly selected as the true class). On the other hand, $y_t^{(i,j)} := \sign(\langle \wstij, x_t\rangle ) = 1 = \sign(\yhat_t - y_t)$  by \Cref{eq:yt_eq}. Thus, $y_t^{(i,j)} \ne \yhat_t^{(i,j)}$
    \item If $\yhat_t < y_t$, then from \Cref{line:pred} it must be the case that $\yhat_t^{(i,j)} = 1$ for $i = \yhat_t$ and $j = y_t$ being the indices selected in \Cref{eq:ij_Kwise}. But by the reverse of \Cref{eq:yt_eq}, $y_t^{(i,j)} = -1 = \sign(\yhat_t - y_t)$. Hence,  $\yhat_t^{(i,j)} \ne y_t^{(i,j)}$.
    \end{itemize}
    \end{proof}
    Hence, we may write
    \begin{align*}
    \sum_{t=1}^T \I\{\yhat_t \ne y\} = \sum_{i<j}\sum_{t=1}^T\ell_t^{(i,j)}, \quad \ell_t^{(i,j)} := \I\{y_t^{(i,j)} \ne \yhat_t^{(i,j)}\}\I\{(i_t,j_t) = (i,j)\}.
    \end{align*}
    We now claim that the losses $\ell_t^{(i,j)} := \I\{y_t^{(i,j)} \ne \yhat_t^{(i,j)}\}\I\{(i_t,j_t) = (i,j)\}$ precisely corresponding to the censored binary setting of \Cref{prop:warmupgeneralized}.Indeed, consider a setting where $x_1,x_2,\dots$ are selected by the $\sigma$-smooth adversary, and the label is $\yhat_t^{(i,j)}$ defined above. $\Abin^{(i,j)}$ does not always see $\yhat_t^{(i,j)}$, but whenever $\ell_t^{(i,j)}  = 1$, \Cref{claim:K_to_bin} shows that the learner does indeed observe the true value  $\yhat_t^{(i,j)}$. Thus, by \Cref{prop:warmupgeneralized}, it holds for any fixed $i < j$ that with probability $1 - \delta$,
    \begin{align*}
    \sum_{t=1}^T\ell_{t}^{(i,j)} \leq 136 d \log(d) + 34 \log\left(\frac{T}{\sigma \delta}\right) + 56
    \end{align*}
    Union bounding over all $\binom{K}{2} \le K^2$ pairs $i < j$ and summing, we conclude that with probability $1-\delta$,
    \begin{align*}
    \reg_T = \sum_{i>j}\sum_{t=1}^T\ell_{t}^{(i,j)} &\leq 136 K^2 d \log(d) + 34K^2 \log\left(\frac{TK^2}{\sigma \delta}\right) + 56K^2. \\
    &\leq 136 K^2 d \log(d) + 90K^2 \log\left(\frac{TK^2}{\sigma \delta}\right) 
    \end{align*}

    \paragraph{Modification for non-unique ground truth classifiers.}
    Here, we can modify $\Abin^{(i,j)}$ with the following rule:  predict $\yhat_t^{(i,j)} = 1$ until there is an  time $t$ for which $(i_t,j_t) = (i,j)$, and then reinitalize $\Abin^{(i,j)}$ to have $w_t^{(i,j)} = e_1$, as in \Cref{alg:warmup}. 

    Consider an $i < j$ with $\wsti= \wstj$. We claim $(i_t,j_t) \ne (i,j)$ for any $t$.  Now, suppose there is a time $t$ that $(i_t,j_t) = (i,j)$, let $\tau$ denote the first time $t$ for which this is true. Then, $\yhat_t^{(i,j)} = 1$.  
    But in addition $y_t \ne j$ for any $t$ because we assume the $\argmax$ in \Cref{eq:argmax_app} is broken lexicograophically. Thus, from \Cref{eq:ij_Kwise}, it must be that  $y_\tau = i$, and that $j$ is such that $\yhat_t^{(i,j)} = -1$; this gives a contradiction.

    Now consider $i < j$ with $\wsti \ne \wstj$. Then our modification of $\yhat_t^{(i,j)}$  only increases $\sum_{t=1}^T\ell_{t}^{(i,j)}$ by at most $1$. This adds at most $\binom{K}{2} < K^2$ to the total regret (modifying the constant of $90$ to $91$).
    %For simplicity, let us assume $w^{\star,(i,j)} \ne 0$ for any $i < j$, we remove this assumption at the end of the proof.

    %Now, let us suppose that at a given time $t$, it holds that $y_t \ne \yhat_t$. Then either $y_t < \yhat_t$, or else $\yhat_t > y_t$. If $y_t < \yhat_t$, then 

    \end{proof}

    \begin{algorithm}[!t]
    \begin{algorithmic}[1]
    \State{}\textbf{Initialize } Binary classifiers $\Abin^{(i,j)}$, $i < j$ 
    \For{$t=1,2,\dots$}
    \State{}\textbf{recieve } $x_t$
    \For{$i < j$} $\yhat_t^{(i,j)} = \Abin^{(i,j)}.\classify(x_t) $
    \EndFor
    \State{}\textbf{predict} $\yhat_t = \min\{i \in [K]: \yhat_t^{(i,j)} = 1,  ~ i < j \le K\}$ \qquad\qquad\qquad(\algcomment{$\selfdot\classify(x_t)$}) \label{line:pred} 
    %\Statex{}\algcomment{$\yhat_t = K$ if $\langle w_1^{(i,j)},x_t \rangle < 0$ for all $i < K$ and $j > i$.}
    \If{$\yhat_t \ne y_t$} \quad\qquad\qquad\qquad\qquad\qquad\qquad\qquad\qquad\qquad(\algcomment{$\selfdot\errupdate(x_t)$})
    \State{} Define
    \begin{align}
    (i,j) = \begin{cases} i = y_t, ~j \in \{j > i: \yhat_t^{(i,j)} = -1 \} &  \text{if } y_t < \yhat_j\\
    i = \yhat_t, ~ j = y_t & \text{if }  \yhat_j < y_t 
    \end{cases} \label{eq:ij_Kwise}
    \end{align}
    \State{} Update $\Abin^{(i,j)}.\errupdate(x_t)$%
    \EndIf
    \EndFor
    \end{algorithmic}
      \caption{$K$-class linear classification}
      \label{alg:K_class}
    \end{algorithm}

\subsection{Formal Guarantees for Piecewise Regression}\label{app:piecewise_reg_guarantees}
    We will prove a slightly more general version of \Cref{prop:kpieceregression} and then derive the result in \Cref{sec:kpiece} as a corollary.  First, we will define what kinds of regression classes our result will apply to:
    \begin{definition}\label{def:elldetermined}
        Let $\G: \cX \to \rr$ be a function class.  We say that $\G$ is $\ell$-determined with respect to some measure $\mu$ on $\cX$ if the following two conditions hold:
        \begin{itemize}
            \item The values on $\ell$ points in general position uniquely determine the function, i.e.,
            \begin{equation}
                \pp\left(\text{there exist } g \neq g' \in \G \text{ such that } g(x_i) = g'(x_i) \text{ for } 1 \leq i \leq \ell \text{ and } x_i \sim \mu \right) = 0
            \end{equation}
            \item Two functions intersect only on measure zero sets, i.e., for all $g, g' \in \G$,
            \begin{equation}
                \mu\left(\left\{x \in \cX: g(x) = g'(x)\right\}\right) = 0
            \end{equation}
        \end{itemize}
    \end{definition}
    Note that linear classes in $\rr^d$ are trivially $d$-determined with respect to the Lebesgue measure, and thus with respect to any measure absolutely continuous with respect to the Lebesgue measure.  Polynomial classes are also $\ell$-determined with respect to the Lebesgue measure for some $\ell$ depending on $d$ and the degree of the polynomials.  We observe that our definition of an $\ell$-determined function class is an offline analogue to the notion of eluder dimension from \citet{russo2013eluder}.

    Now, for a given function class $\G: \cX \to \rr$, we denote by
    \begin{equation}
        \G_\F = \left\{x \mapsto \mathbf{g}_f(x) = \sum_{i = 1}^K g_i(x) \bbI[f(x) = i] \bigg| g_i \in \G \text{ and } f \in \F\right\} \label{eq:Gf_gneral}
    \end{equation}
    where $\F$ is the set of $K$-class linear classifiers from \Cref{prop:kpiececlassification}.  We will continue to suppose that the $x_t$ are drawn from distributions that are $\sigma$-smooth with respect to $\mu$ and that the labels $y_t$ are realizable with respect to $\G_\F$.

\begin{assumption}[Oblivious, realizable smoothed sequential setting]\label{asm:oblvious.}  We suppose  smoothed online learning setting and the adversary is realizable with respect to $\G_\F$ and \emph{oblivious} in the sense that before the learning process begins, the adversary chooses $\bgst = (\gst_1, \dots, \gst_K) \in \G^{K}$ and $\fst \in \F$ and lets $y_t = (\bgst)_{\fst}(x_t)$ for all $t$. We assume further that $\bgst$ has \emph{unique} entries: $\gst_i \ne \gst_j$ for $i \le j$. 
\end{assumption}

Lastly, we assume we have access to the following ERM oracle.
\begin{definition}[ERM Oracle]\label{defn:erm_oracle} Given $\cU = \{(x_1,y_1),\dots,(x_n,y_m)\}$, where $(x_i,y_i) \in \cB_1^d \times \cY$, $\erm(\cU,\cG,K)$ returns a $n \le K$, and $g_1,\dots,g_n$ and partition $C_1,\dots,C_n$ of $\cU$ such that, for all $(x,y) \in C_i$, $g_i(x) = y$. By post-processing, we may also assume that $g_i$ are distinct\footnote{ Note that the ERM Oracle need not cluster with respect to the classifiers (even thought it can certainly be implemented this way). Hence, one can can merge cluster to ensure $g_i$ is distinct. }
\end{definition}

\begin{proposition}[General $\ell$-Determined Regression]\label{prop:generalkpiecereg}

Suppose that we are in the semi-oblivious, smoothed online learning setting, where the adversary begins by choosing $g_{f^\ast}^\ast \G_\F$ from \eqref{eq:Gf_gneral}, and, at each time $t$, draws $x_t$ from a distribution that is $\sigma$-smooth with respect to $\mu$ and sets $y_t = g_{f^\ast}^\ast(x_t)$. Suppose further that $\G$ is $\ell$-determined, in the sence of \Cref{def:elldetermined}. Then, \Cref{alg:general_reg} satisfies for all $T$, with probability at least $1 - \delta$,
    \begin{equation}
        \reg_T \leq 136 K^2 d \log(d) + 91K^2 \log\left(\frac{TK^2}{\sigma \delta}\right) +  K^2(\ell + 1)
    \end{equation}
Moreover, the per-time step computational complexity of \Cref{alg:general_reg} is polynomial in $d$ and the complexity of the $\erm$ oracle \Cref{defn:erm_oracle}, applied to a data set $\cU$ of size no more that $|\cU| \le K(\ell+1)$.
\end{proposition}

\subsection{Algorithm for Piecewise Regression}\label{app:piecewise_reg_alg}
\newcommand{\Ghat}{\hat{\G}}
\newcommand{\khat}{\hat{k}}
\newcommand{\ktil}{\tilde{k}}
\newcommand{\Aclass}{\cA}

\newcommand{\clusmer}{\algfont{clusterMerge}}
\newcommand{\piecereg}{\algfont{PiecewiseReg}}

\Cref{alg:general_reg} proceeeds at follows. We let $N_t$ denote the number of clusters about which we are certain, $\cU_t$ denote the set of points which cannot be assigned to a cluster. We maintain a supervised $K$-class linear classifier, $\Aclass$, described in \Cref{alg:K_supervision_algorithm}. It is similar in spirit to \Cref{alg:K_class}, except it takes in ``side information'' $M_t$ on which it only predicts from the first $M_t$ classes. Lastly, we maintain a growing sequence of regressors $\ghat_1,\ghat_2,\dots \in \cG$ such that $\ghat_i$ does not change once assigned, and $\ghat_i$ is defined for all $i \le N_t$.

At each time $t$, we call $\khat_t = \Aclass.\classify(x_t,N_t)$ to guess the cluster of $x_t$, only among cluster $i \le N_t$ about which we are certain. Then, we predict $\yhat_t = \ghat_{\khat_t}(x_t)$. The idea is that, for $k = \khat_t \le N_t$, we are sure that $\ghat_k$ is the true predictor if $x_t$ is in cluster $k$.We then observe $y_t$. If $y_t$ was correctly predicted by one of that $\ghat_i$ for which $i \le N_t$, but not the $\khat_t$ we guessed, then we update our classifier $\Aclass$. Otherwise, we call the ERM oracle to determine if we can find new cluster(s) to add, appending to our sequence of predictors $\ghat$'s, and growing our number of certain clusters $N_t$. Note that we never maintain an \emph{explicit} clustering of our points, but only cluster retroactively based on whether $\ghat_i(x_t) = y_t$ for some $i$, as a means to recover the classification label.

    \begin{algorithm}[!t]
    \begin{algorithmic}[1]
        \State{}\textbf{Init:}  $K$-class supervised linear classifier $\cA$ (instance of \Cref{alg:K_supervision_algorithm})
        \Statex{}\qquad ERM-oracle $\erm$ (see )
        \For{each time $t= 1,2,\dots$}
        \State{}\textbf{recieve } $x_t$
        \State{}\textbf{predict} $\yhat_t = \ghat_{k}(x_t)$ for $k = \khat_t$, where  $\khat_t := \Aclass.\classify(x_t,N_t)$ \algcomment{$\yhat_t = 0$ if $N_t = 0$}
        \State{}\textbf{observe} $y_t$.
        \If{$\exists k^\star_t \in [N_t]$ with $\ghat_k(x_t) = y_t$}
            \Statex{} \algcomment{update classification}
        \If{$\khat_t \ne k^{\star}_t$}, $\Aclass.\errupdate(x_t,N_t)$
        \EndIf
        \State{}\textbf{maintain} $N_{t+1} \gets N_t$, $\cU_{t+1} \gets \cU_t$
        \Else \algcomment{update clustering}
        \State{} $(C_{1:n},g_{1:n}) \gets \erm(\tilde \cU_t,\cG,K)$ , $\tilde{\cU}_t = \cU_t \cup \{(x_t,y_t)\}$ \label{line:erm} 
        \Statex{}\algcomment{Initialize $\tilde{N} = N_t$}
        \For{each $i: |C_i| \ge \ell +1$} 
        \State{} $\tilde N = \tilde N + 1$, $\ghat_{\tilde{N}} \gets  g_i$,\label{line:add_gs} 
        \EndFor
        \State{} $N_{t+1} \gets \tilde N$, ~~$\cU_{t+1} \gets \tilde \cU_t \setminus \bigcup_{i:|C_i| \ge \ell +1}\{(x,y) \in\tilde \cU_t: g_i(x) = y\}$ \label{line:remove}
        \EndIf
        \EndFor
      \end{algorithmic}
      \caption{General Piecewise Regression}
      \label{alg:general_reg}
    \end{algorithm}

    \begin{algorithm}[!t]
    \begin{algorithmic}[1]
    \State{}\textbf{Initialize } Binary classifiers $\Abin^{(i,j)}$, $i < j$ 
    \For{$t=1,2,\dots$}
    \Statex{}\algcomment{guarantee $y_t \le M_t$}
    \State{}\textbf{Recieve } $(x_t,M_t)$ and \textbf{predict} \qquad\qquad \qquad\qquad\qquad (\algcomment{$\selfdot\classify(x_t,M_t)$})
    \begin{align*}
    \yhat_t = \min\{i \in [M_t]: \Abin^{(i,j)}.\classify(x_t) = 1,  ~ i < j \le M_t\}, 
    \end{align*}
    \State{}\textbf{Observe } $y_t$
    %\Statex{}\algcomment{$\yhat_t = K$ if $\langle w_1^{(i,j)},x_t \rangle < 0$ for all $i < K$ and $j > i$.}
    \If{$\yhat_t \ne y_t$} \qquad\qquad\qquad\qquad\qquad\qquad\qquad\qquad(\algcomment{$\selfdot\errupdate(x_t,M_t)$})
    \State{} Define
    \begin{align}
    (i,j) = \begin{cases} i = y_t, ~j \in \{k > i: \langle w_1^{(i,k)},x_t \rangle < 0\} &  \text{if } y_t < \yhat_j\\
    i = \yhat_t, ~ j = y_t & \text{if }  \yhat_j < y_t 
    \end{cases}
    \end{align}
    \State{} Update $\Abin^{(i,j)}.\errupdate(x_t)$%
    \EndIf
    %\EndIf 
    \EndFor
    \end{algorithmic}
      \caption{$K$-class linear classification with supervision}
      \label{alg:K_supervision_algorithm}
    \end{algorithm}

\newpage
\subsection{Proof of Proposition \ref{prop:generalkpiecereg}}\label{app:piecewise_reg_proof}

\subsubsection{Guarantee for ERM procedure}
    \begin{lemma}\label{claim:right_class} Let $I \subset [T]$ be any subset of time. Then with probability one, it holds that for any  partition $C_1,\dots,C_n$ of $(x_s,y_s)_{s \in [I]}$ and any $g_1,\dots,g_n$  distinct functions such that, for all $(x,y) \in C_n$, $\tilde{g}_i(x) = y$,  then for any index $i$ for which $|C_i| \ge \ell+1$,
    \begin{itemize}
        \item $\fst(x) = \fst(x')$ for all $(x,y),(x',y') \in C_i$
        \item $\tilde{g}_i = \gst_{\fst(x)}$, representative $x \in C_i$
        %\item If $(x,y) \in C_i$ and $s \in I$, and if $\tilde{g}_i(x_s) = y_s$, then, $\fst(x_s) = \fst(x)$. 
    \end{itemize}
    \end{lemma}
    \begin{proof} Let $I_1,\dots,I_m$ denote the times in each cluster $C_1,\dots, C_n$. Without loss of generality, suppose $I_1$ is a cluster for which $|I_1| \ge \ell+1$ (we may handle all simultaneously via a finite union bound.)

    \paragraph{Item 1.} Suppose in fact that there exists $s,s' \in I_1$ with $\fst(x_s) \ne \fst(x_{s'})$. We first argue then that ${g}_1 \ne \gst_k$ for all $k\in [K]$. Indeed, by smoothness and the second condition of \Cref{def:elldetermined}, it holds that with probability $1$, $\gst_k(x_{\tilde s}) \ne \gst_{k'}(x_{\tilde s})$ for all $1 \le \tilde s \le T$ and $k \ne k'$. Set $i_1 = \fst(x_s)$ and $i_2 = \fst(x_{s'})$.  Thus, if $g_i = \gst_{k}$, the fact $g_i(x_s) = \bgst_{\fst}(x_s)$ and $g_i(x_{s'}) = \bgst_{\fst}(x_{s'})$ would require both $\gst_{k}(x_s) = \bgst_{\fst}(x_s) = \gst_{i_1}(x_s)$ and $\gst_{k}(x_{s'}) = \bgst_{\fst}(x_{s'}) = \gst_{i_2}(x_{s'})$. Thus, on the aforementioned probability one event, we would have both $k = i_1$ and $k = i_2$, which contradicts the supposition $i_1 \ne i_2$.

    Next, let $S \subset [T]$ denote a set of indices. Denote $s_{\max} = \max \{s \in S\}$, and define the events
    \begin{align*}
    \cA_{S}(g') := \{\exists g  \in \G \setminus \{g'\}:  g'(x_{s_{\max}}) = y_{s_{\max}}, \quad \forall s \in S, g(x_{s}) = y_s\}\}.
    \end{align*}
    By the above observation that $g_i \ne \gst_{k}$ for any $k$, we see that if there exists $s,s' \in I_1$ with $\fst(s) \ne \fst(s')$ with $|I_1| \ge \ell+1$, then one of the events $\cA_{S}(\gst_k)$ must occur for some $|S| \ge \ell+1$ and $k \in [K]$. Since there are only finitely many such events, it suffices to show that for any \emph{fixed} $S$ and $k$, $\Pr[\cA_{S}(\gst_k)] = 0$.

    Hence, fix $S$ and $k$. For a given $S$ with max element $s_{\max}$, let $\filt_{-1}$ denote history generated by $(x_1,y_1),\dots,(x_{s_{\max}-1},y_{s_{\max}-1})$. Define the  $\cA_{-1} := \{\exists  g  \in \G \setminus \{\gst_k\}: s \in S, g(x_{s}) = y_s, s \in S \setminus \{s_{\max}\}\}$. Then, $\cA_{-1}$ is $\filt_{-1}$ measurable and $\cA_{-1} $ contains $\cA_{S}(g') $. Hence,
    \begin{align*}
    \Pr[\cA_{S}(\gst_k)] &= \Exp[\Pr[\cA_{S}(\gst_k) \mid \filt_{-1}]]\\
    &= \Exp[\I\{\cA_{-1}\} \cdot \Pr[\cA_{S}(\gst_k) \mid \filt_{-1}]].
    \end{align*}
    By the first condition of \Cref{def:elldetermined},  $\cA_{-1} $ coincides with the event $\cA_{-1}' := \{\exists  \text{ a unique } g \in  \G \setminus \{\gst_k\}: s \in S, g(x_{s}) = y_s, s \in S \setminus \{s_{\max}\}\}$ almost surely. Hence, 
    \begin{align*}
    \Pr[\cA_{S}(\gst_k)] &= \Exp[\I\{\cA_{-1}'\} \cdot \Pr[\cA_{S}(\gst_k) \mid \filt_{S-1}]].
    \end{align*}
    Lastly, when $\cA_{-1}'$ holds, let $\hat{g} \ne \gst_k$ denote the unique $g \ne \gst_k$ consistent with examples $s \in S \setminus \{s_{\max}\}$. Since $\hat{g}$ is determined by $\filt_{S-1}$, we have 
    \begin{align*}
    \Pr[\cA_{S}(\gst_k) \mid \filt_{S-1}] \le \Pr[\hat{g}(x_{s_{\max}}) \ne \gst_k(s_{\max})] = 0,
    \end{align*}
    where we use that $\hat{g}$ is \emph{fixed}, that $\hat{g} \ne g'$, and  the second condition of \Cref{def:elldetermined}. The bound follows.

    \paragraph{Item 2.} For any fixed set of indices $\tilde{I}$ with $|\tilde I| \ge \ell +1 \ge \ell$,   the first condition of $\ell$-determination (\Cref{def:elldetermined}) ensures then that $\Pr[\exists g_1 \ne \gst_j: g_1(x_s)= \gst_j(x_s), \forall x \in \tilde{I}] = 0$. The bound follows by union boundig over all $\tilde{I} \subset [T]$ and $j \in [K]$.

    %\paragraph{Item 3.} Above, we have established that, with probability $1$, $g_1 = \gst_j$ for some index $j \in [K]$. But for any index $s \in [T]$ and $k \in [K]$, $\Pr[\gst_k(x_s) = \gst_j(x_s), k \ne j] = 0 $ by the second point of $\ell$-determination (\Cref{def:elldetermined}). By realizability, $y_s = \gst_k(x_s)$ for some $k$. Hence, by a union bound, $k =j$, where again $j$ is such that $\gst_j = g_1$. In other words, $\fst(x_s) = j = \fst(x)$, where $x$ is any representative fo $C_1$.
    \end{proof}

\subsubsection{Distinctness of clustering}
    \begin{claim}[$\cU_t$ is Uncertain Set]\label{claim:Ut} Fix a time $t$. Then, for any $(x,y) \in \cU_t$ and any $i \le N_t$, $\ghat_i(x) \ne y$.
\end{claim}
\begin{proof} This is true vacuously at time $t = 1$, when $\cU_t = \emptyset$. Suppose it holds at time $t$, we prove it for time $t+1$. If $x_t$ is such that there exists an $n \le N_t$ with $\ghat_i(x_t) = y_t$, then $\cU_{t+1}$ does not change from $\cU_t$. Otherwise, if $\ghat_n(x_t) \ne y_t$ for all $n \le N_t$,
\begin{align}
\cU_{t+1} \gets \tilde{U}_t \setminus \bigcup_{i:|C_i| \ge \ell +1}\{(x,y) \in\tilde \cU_t: g_i(x) = y\}, \quad \tilde{U}_t :=\cU_t \cup \{(x_t,y_t)\}  \label{eq:Ut_remove}
\end{align}
where $g_i$ and $C_i$ are the clustering from the ERM oracle.
By the inductive hypothesis and fact that $\ghat_n(x_t) \ne y_t$ for all $n \le N_t$, it follows that $\ghat_n(x) \ne y$ for all $(x,y) \in \cU_t \cup \{(x_t,y_t)\} = \tilde\cU_{t}\supseteq \cU_{t+1}$. Now, if there is some  $n: N_t < n \le N_t$ for which $\ghat_n(x) = y$, then that $\ghat_n$ was added during the ERM step at round $t$: i.e. $\ghat_n= g_i$ for some $i$ such that $|C_i| \ge \ell+1$. But then $(x,y)$ is removed form $\cU_{t+1}$ by \Cref{eq:Ut_remove}. 
\end{proof}

\begin{claim}\label{claim:g_unique} Fix a time $t$. Then, for any $i,j \le N_t$, $\ghat_i\ne \ghat_j$.
\end{claim}
\begin{proof} This is trivially true at time $t=1$. Suppose this is true at time $t$, we establish the claim for time $t+1$. If $x_t$ is such that there exists an $i \le N_t$ with $\ghat_i(x_t) = y_t$, then $N_{t+1} = N_t$ and so the set of $\ghat_i$'s under consideration remains unchanged. 

On the other hand, suppose there is no $i \le N_t$ with $\ghat_i(x_t) = y_t$. Then, all possible new $\ghat_j$'s for $N_t < j \le N_{t+1}$ are correct on some subset of points of $\cU_t \cup (x_t,y_t)$. But by the previous claim (\Cref{claim:Ut}) and the assumption that, for $i \le N_t$ with $\ghat_i(x_t) = y_t$, no element of $\cU_t \cup (x_t,y_t)$ is correctly predicted by any $\ghat_i$ for $i \le N_t$. Thus, none of the new $\ghat_j$'s can equal an $\ghat_i$ for $i \le N_t$. Moreover, by the definition the ERM oracle, \Cref{defn:erm_oracle}, all newly added $\ghat_j$'s are distinct. 
\end{proof}

\subsubsection{Key summary of Algorithm \ref{alg:general_reg}}
We now summarize the results with the following lemma.
\begin{lemma}\label{lem:g_correct} With probability $1$, there exists a permutation $\pi$ such that
\begin{itemize}
    \item For each time $t$ and $i \in [N_t]$, $\ghat_i = \gst_{\pi(i)}$
    \item  For each time $t$ and $i \in [N_t]$, $\ghat_i(x) = y$ if and only if $\fst(x) = \pi(i)$
    \item If $(x,y) \in \cU_t$, then $\pi^{-1}(\fst(x)) > N_t$.
    \item Whenever $\yhat_t \ne y_t$, either $\pi^{-1}(\fst(x)) > N_t$, or $\khat_t := \Aclass.\classify(x_t,N_t)$ has $\pi(\khat_t) \ne \fst(x_t)$. 
 \end{itemize}
\end{lemma}
\begin{proof} For $n = N_T$, let $\ghat_1,\dots,\ghat_n \in \cG$ denote the functions constructed by our algorithm. Since each new $\ghat_i$ is added from a cluster with at least $\ell+1$ points, applying \Cref{claim:right_class} (with a union bound over index sets $I$ ensures that $\ghat_i = \gst_j$) for some $j \in [K]$. This gives us a mapping $\pi: [n] \to [K]$. $\pi$ must be injective, since $\gst_j$ are distinct by assumption, and $\ghat_i$ are unique by \Cref{claim:g_unique} (in particular, $n \le K$). Thus, $\pi$ can be extednded to a permutation from $[K] \to [K]$. By construction, the first item is satisfied.

The second item is a consequence of uniquenessof that  the previous point, uniqueness of $\ghat_i$'s, and the second point of \Cref{def:elldetermined} , since we only need to union bound over finitely many times $t \in [T]$ and pairs $\gst_i,\gst_j$'s. The third item follows similarly, by invoking \Cref{claim:Ut}.
  
For the last point, suppose $\pi^{-1}(\fst(x_t)) \le N_t$. Then, by the previous point, $(x_t,y_t) \notin \cU_t$. Thus, the algorithm classifies $\yhat_t = \ghat_{\khat_t}(x_t)$ where $\khat_t \in [N_t]$. But by the first point of the lemma, $\ghat_{\khat_t} = \gst_{\pi(\khat_t)}$. So if $\pi(\khat_t) = \fst(x_t)$, then we woudl have $\ghat_{\khat_t}(x_t) = \gst_{\fst(x_t)}(x_t) = y_t$, a contradiction.
\end{proof}

\subsubsection{Proof of Proposition \ref{prop:kpieceregression}}

Let $\pi$ denote the permutation ensured by \Cref{lem:g_correct}. We may assume without loss of generality that $\pi$ is the identity permutation (by permuting $\bgst$).  
Let $k_t = \fst(x_t)$.  Recalling also that $\khat_t \le N_t$, the fourth point of \Cref{lem:g_correct} ensures. 
\begin{align*}
\I\{\yhat_t \le y_t\} &\le \I\{k_t > N_t\} + \sum_{t=1}^T \I\{\khat_t \ne k_t, k_t \le N_t\} 
\end{align*}
First, we bound the contribution of $\I\{k_t > N_t\}$:
\begin{claim}\label{claim:not_enough_claim} $\sum_{t=1}^T \I\{k_t = k,\quad k > N_t\} \le K(\ell+1)$. Thus, $\sum_{t=1}^T\I\{k_t > N_t\} \le K^2(\ell+1)$.
\end{claim}
\begin{proof}
Suppose $k > N_t$, and let $S_{t,k} := \{s \le t: \fst(x_s) = k\}$, and define $\tau_k = \max\{t \in [T]: k_t > N_t, k_t = k\}$. Then
\begin{align*}
\sum_{t=1}^T \I\{k_t > N_t, k_t  = k\} = |S_{\tau_k,k}|.
\end{align*}
We claim that $|S_{\tau_k,k}| \le K(\ell+1)$.  Indeed, suppose  $|S_{\tau_k,k}| > K(\ell+1)$. Then, for some $t < \tau_k$, $|S_{t,k}| = K(\ell+1)$ and $k_t = k$ and $k > N_t$.  By \Cref{lem:g_correct},  $\gst_{k_t}(x_t) \ne \ghat_i(x_t)$ for any $i \le N_t$. Hence, our algorithm executes \Cref{line:erm}. By the pidgeon-hole principle, there must be at least one cluster $C_{i}: |C_i| \ge \ell+1 $ which contains at least one $s \in S_{t,k}$. Hence, the update rule ensures that $i \le N_{t+1}$ for which $\ghat_{i}(x_s) = y_s$. But again, by \Cref{lem:g_correct} (and taking the permutation to be the identity), we have $\fst(x_s) =i$. In other words, $i = k_s = k$, i.e. $k_s \le N_{t+1} \le N_{\tau_k}$. This constradictions the definition of $\tau_k$.
\end{proof}
Summarizing our argument thus far, the following holds with probability one
\begin{align}
\sum_{t=1}^T\I\{\yhat_t \le y_t\} &\le K^2(\ell + 1) + \sum_{t=1}^T \I\{\khat_t \ne k_t, k_t \le N_t\}   \label{Eq:first_bit_reg}
\end{align}
Finally, by mirroring the proof of \Cref{prop:kpiececlassification}, we  upper bound 
\begin{align}
\sum_{t=1}^T \I\{\khat_t \ne k_t, k_t \le N_t\}   \leq  136 K^2 d \log(d) + 91K^2 \log\left(\frac{TK^2}{\sigma \delta}\right). \label{eq:second_bit_reg}
\end{align}
The key difference between the above bound and that of  \Cref{prop:kpiececlassification} is that we only see when get feedback $k_t \le N_t$, but at the same time, we only suffer a loss when $k_t \le N_t$. Hence, the bound follows from a near-identical argument, calling the general censored version of our binary classification \Cref{prop:warmupgeneralized}, modified to add the event $\{k_t \le N_t\}$ to the censoring. Combining \Cref{Eq:first_bit_reg,eq:second_bit_reg} concludes.

\qed

\newpage
%!TEX root = ../neurips_submission.tex
\newcommand{\xbar}{\bar{x}}
\newcommand{\ebar}{e}
\newcommand{\xhat}{\hat{x}}
\newcommand{\distw}{\mu}
\newcommand{\wstbar}{\bar w^\star}
\newcommand{\wbar}{\bar w}
\newcommand{\spread}{\mathsf{p}_{\distw}}

\newcommand{\yst}{y^\star}
\newcommand{\wsthat}{\hat{w}^\star}
\newcommand{\kapst}{\kappa^{\star}}

\section{Non-realizable mistake bounds for the Perceptron.}\label{app:perceptron}

For simplicity, we consider regret with respect to a fixed $\bst \in \R, \wst \in \cB_1^d$, and define
\begin{align*}
\yst_t = \yst(x_t), \quad \yst(x):= \sign(\bst + \langle x_t, \wst \rangle). 
\end{align*}
Again, we normalize $x_t$ that $\max_t \|x_t\| \le 1$. We further assume that
\begin{align*}
\|\bst\|^2 +\|\wst\|^2 = 1, \quad \|\wst\| \ge 1/2.
\end{align*}
We show in \Cref{lem:kapst} at the end of this section that this is without loss of generality. We define
\begin{align*}
\wsthat = \wst/\|\wst\|.
\end{align*}
Unlike with our cutting-plane methods, we allow the adversary to deviate from a realizable clasiffier. Specifically, for each time $t$, the adversary selects $x_t \sim p_t$, and may instead choose to play some $y_t \ne \yst_t$. We define,
\begin{align*}
\Nerr := 1+|\{t:y_t \ne \yst_t\}|,
\end{align*}
and obtain non-vacuous mistake bounds provided $\Nerr$ is sublinear in $T$.

Informally, our total mistake bound for the Perceptron is polynomial in the smoothness along the direction of the optimal classifier $\wsthat$. This is formalized in the following definition:
\begin{definition}[Directional $\sigdir$-smoothness]\label{defn:sigdir} We say that the adversary 
\begin{itemize}
\item is $(\sigdir,\wsthat)$ directionally-smooth if $\langle x_t,\wsthat \rangle$ has density at most $1/\sigdir$ with respect to the Lebesgue measure on the real line. 
\item is, more generally, $(\sigdir,\alpha,\wsthat)$ directional-Tsybakov-smooth if $\sup_{a \in \bbR} \Pr_{x_t \sim p_t} [\langle x_t,\wsthat \rangle \in [a,a+\eta]] \le \eta^{1-\alpha}/\sigdir$. 
\end{itemize} 
Note that a $(\sigdir,\alpha,\wsthat)$-Tsybakov adversary is is $(\sigdir,\wsthat)$-smooth. 
\end{definition}
Note that \Cref{defn:sigdir} is a slightly weaker condition than the one consider in \Cref{thm:Tsyb_body} in the body, as it only requires directional smoothness along $\wsthat$ (not uniformly).
As noted in the body, directional smoothness can differ substantially from general smoothness. We provide two examples.
\begin{example}[Additive $d$-Ball Noise]\label{exm:d_ball_noise} Suppore that at each time $t$, the adversary selects $x_t = \xhat_t + e_t$, where $\|\xhat_t\| \le 1/2 $, and $e_t \sim r\cB_{1}^d$ for $r \le 1/2$ (and, for simplicity, $d > 1$). Then, the adversary is $\sigma$-smooth for $\sigma = \vol_d(r\cB_{1}^d)/\vol_d(\cB_1^d) = r^d$. However, if $u \sim \mu_d$ is drawn uniformly from the sphere, then the density $p_1(\cdot)$ of its first coordinate $u_1$ with respect to the Lebesgue measure is 
\begin{align*}
p_1(u_1) &= \frac{\vol_{d-1}(\sqrt{1-u_1^2}\cB_1^{d-1})}{\vol_d(\cB_1^d)} \le \frac{\vol_{d-1}(\cB_1^{d-1})}{\vol_d(\cB_1^d)} = \frac{(d-1)}{2\sqrt{\pi}}. 
\end{align*}
Hence, by rotational symmetry, we see that for any $\wst$, the adversary is $(\sigdir,\wsthat)$ directionally-smooth for $\sigdir =\frac{2\sqrt{\pi}r}{(d-1)}$. Notice that the directional smoothness is now only polynomial in $d$, rather than exponential in it.  
\end{example}
\begin{example}[Additive Noise in a Random-Direction]\label{example:random_noise} Again consider the additive noise setting where at each time $t$, the adversary selects $x_t = \xhat_t + e_t$, where $\|\xhat_t\| \le 1/2 $. However, suppose $e_t$ is selected as follows: before the game, the adversary selects a direction $\hat{e} \sim \mu_d$, and plays $e_t = a_t \hat{e}$, where $a_t$ is drawn uniformly on the interval $[-r/2,r/2]$. Note that this adversary need not be $\sigma$-smooth with respect to $\mu_d$ for any $\sigma > 0$, because after the adversary commits to $\hat{e}$, her smoothing is restricted to a line segment. Still, with constant probability, $\langle \hat{e},\wsthat \rangle \ge c/d$ for some constant $c > 0$. Hence, with constant probability, the adversary is  $(\sigdir,\wsthat)$-directionally smooth for $\sigdir = cr/d$.
\end{example}

We now state our guarantee for the classical Perceptron algorithm \cite{rosenblatt1958perceptron}
\begin{theorem}\label{thm:Tsyb} Suppose that the adversary is $(\sigdir,\alpha,\wsthat)$-Tsybakov, and define $\rho := \frac{2}{3-\alpha} \in [\frac{2}{3},1)$. Then, with probability $1-\delta$, the Perceptron algorithm (\Cref{alg:perceptron}) satisfies
\begin{align*}
\sum_{t=1}^T \I\{\yhat_t \ne y_t\}  \lesssim ( T/\sigdir)^{\rho} \cdot (\Nerr )^{1- \rho} + \log(\ceil{\log T }/\delta). 
\end{align*}
\end{theorem}
\begin{remark}\label{rem:compare_perceptron} Recall \Cref{exm:d_ball_noise}, which shows that directional smoothness $\sigdir$ may scale as $\sim r/d$ when the (standard) smoothness scales as $\sigma = r^{d}$.  Applying \Cref{thm:Tsyb} with $\alpha = 0$ and thus, $\rho = 2/3$, our $( T/\sigdir)^{2/3} \sim (Td/r)^{2/3} $-mistake bound interpolates between the $\log(T/\sigma)\sim d \log(1/r) + \log(T)$ bounds attained in this paper, and the $\mathrm{poly}(1/\sigma) \sim (1/r)^{\Omega(d)}$-regret enjoyed by previous computationally efficient algorithms. In addition, we achieve a robustness to sublinearly-in-$T$ mistakes, which prior approaches do not. 
\end{remark}
In fact, a more general result holds, in terms of a direction-wise anti-concentration of the adversaries distributions. 
\begin{theorem}[Guarantee under Tsyabkov Smoothness]\label{thm:general mistake}
Define the anti-concentration function
\begin{align*}
\spread(\eta;v):= \sup_t\sup_{a \in \bbR} \Pr_{x_t \sim p_t} [\langle x_t,v \rangle \in [a,a+\eta] \mid \filt_{t-1}]  \label{eq:spread}.
\end{align*}
For any  fixed $\gamma \in (0,R)$, with least $1-\delta$, the number of mistakes made by the Perceptron (\Cref{alg:perceptron})  is at most
\begin{align*}
%(8N_1 + 4)\frac{R^2}{\gamma^2} + 8m\spread(\gamma/\|\wstbar\|,\wsthat)  + 32\log(1/\delta)
\sum_{t=1}^T \I\{\yhat_t \ne y_t\}  \lesssim \frac{\Nerr}{\gamma^2} + T\spread(R\gamma,\wsthat) + \log(1/\delta), 
\end{align*}

\end{theorem}

\subsection{Proofs for the Perceptron}

We begin by stating the standard guarantee for the Perceptron algorithm due to \cite{freund1999large}. To emphasizes its generality, we use $\bar x_i$ to denote its inputs, which we allow to have non-normalized radius $R$.
\begin{algorithm}[!t]
    \begin{algorithmic}[1]
    \State{}\textbf{Initialize } $w_1 = \mathbf{e_1} \in \cB_1^d$
    \For{$t=1,2,\dots$}
    \State{}\textbf{Recieve } $x_t$ and \textbf{predict}
    \begin{align*}
    \yhat_t = \sign(\inprod{w_t}{x_t}), \tag{\algcomment{$\selfdot\classify(x_t)$}}
    \end{align*}
    \If{$\yhat_t \ne y_t$} \quad\qquad\qquad\qquad\qquad\qquad\qquad\qquad\qquad\qquad(\algcomment{$\selfdot\errupdate(x_t)$})
    \State{} $ w_{t+1} \gets w_t + y_t x_t$ \\
    \EndIf
    \EndFor
    \end{algorithmic}
      \caption{Online Perceptron}
      \label{alg:perceptron}
    \end{algorithm}

\begin{theorem} Let $(\xbar_i,y_i)_{i=1}^T \in \bbR^{n} \times \bbR$ be a sequence of labeled exampled with $\|\bar x_i\| \le R$. Fix $\wbar \in \cS^{n-1}$, $\gamma > 0$, and define the margin errors
\begin{align*}
d_i := d_i(\wbar,\gamma) = \max\{0,\gamma - y_i\cdot\langle \bar x_i, \wbar \rangle \}
\end{align*}
Then, the number of mistakes make by the online Perceptron is at most
\begin{align*}
\frac{(R+D)^2}{\gamma^2}, \quad D = \sqrt{\sum_{i=1}^T d_i^2}
\end{align*}
\end{theorem}
The following corollary explicitly bounds the term $D^2$,
\begin{corollary}\label{cor:mix_mistakes_error} Fix a $\wbar \in \cS^{n-1},\gamma \in (0,R)$. Let 
\begin{align*}
N_{1} &:= |\{i:\sign(y_i \cdot \langle \xbar_i, \wbar \rangle) < 0\}|\\
N_{2} &:= |\{0 \le y_i\cdot\langle \xbar_i,\wbar \rangle \in [0,\gamma]\}|
\end{align*}
Then, the number of mistakes make by the online Perceptron is at most
\begin{align*}
(8N_1 + 4)\frac{R^2}{\gamma^2} + 2N_2 
\end{align*}
\end{corollary}
\begin{proof}[Proof of \Cref{cor:mix_mistakes_error}] Let $S_{1}:= \{i:\sign(y_i \cdot \langle \xbar_i, \wbar \rangle) \ne 1\}$ and $S_2 := \{i:\sign(y_i \cdot \langle \xbar_i, \wbar \rangle) = 1, \quad y_i\cdot\langle \xbar_i,\wbar \rangle \le \gamma\}$. Note that $N_1 := |S_1|$ and $N_2 := |S_2|$. Moreover, $S_1 \cap S_2 = \emptyset$, and if $i \notin (S_1 \cup S_2)$, $d_i := \max\{0,\gamma - y_i\cdot\langle \xbar_i,\wbar \rangle \} = 0$. Hence, 
\begin{align*}
D^2 &= \sum_{i \in S_1} d_i^2 + \sum_{j \in S_2} d_j^2\\
&\le N_1 \max_{i \in S_1} d_i^2 + N_2 d_j^2.
\end{align*}
For $i \in S_1$, $d_i^2 \le (\gamma + |\langle \xbar_i,\wbar \rangle| )^2 \le (\gamma + R)^2 \le 4R^2$, where we used $\gamma \le R$. For $j \in S_2$, $d_j \in [0,\gamma]$, so $d_j^2 \le \gamma^2$. Thus, $D^2 \le 4R^2N_1 + N_2$. Thus, 
\begin{align*}
\frac{(R+D)^2}{\gamma^2} \le \frac{2R^2 +2 D^2}{\gamma^2} \le (8N_1 + 4)\frac{R^2}{\gamma^2} + 2N_2. 
\end{align*}
\end{proof}

We now return to our specific setting, re-adopting $x_i$ (not $\bar x_i$) for features.  We bound the probability that a given point $x_i$ does not lie within a margin $\gamma$.
\begin{lemma}\label{lem:small_margin} Consider the $\spread$ function from \cref{eq:spread}. Then,  for any interval $I_0 \subset \bbR$,
\begin{align*}
\Pr[\yst_t\cdot(\bst + \langle x_t,\wst\rangle) \rangle \in I_0 \mid \cF_{t-1}] \le 2\spread(2|I_0|,\wsthat),
\end{align*}
In particular,  $\Pr[\yst_t \cdot(\bst + \langle x_t,\wst \rangle) \le \gamma\}  \mid \cF_{t-1}] \le 2\spread(2\gamma,\wsthat)$.
\end{lemma}

\begin{proof}[Proof of \Cref{lem:small_margin}] Note that the ground truth label $y_i$ may depend on $x_i$. We circumvent this with a union bound. Let $I_0$ be any inverval. 
\begin{align*}
\Pr[\yst_t\cdot(\bst + \langle x_t,\wst \rangle) \in I_0]
&\le \sum_{y \in \{-1,+1\}} \Pr[\langle x_t,\wst \rangle\in y I_0 ]\\
&\le \sum_{y \in \{-1,+1\}} \Pr[\bst + \langle x_t,\wsthat \rangle\in (\bst + y I_0)/\|\wst\| ]\\
&\le  2\spread(\|\wst\|^{-1}|I_0|,\wsthat) \le 2\spread(2|I_0|,\wsthat),
\end{align*}
where we recall our assumption $\|\wst\| \ge 1/2$.
\end{proof}

We may now prove Proof of \Cref{thm:general mistake}.
\begin{proof}[Proof of \Cref{thm:general mistake}] We apply \Cref{cor:mix_mistakes_error} with $\wbar = (\wst,\bst)$ and $\xbar_i = (x_i,1)$. Note then that $\|\xbar_i\|^2 = 1+\|x_i\|^2 \le 2$, so we may take $R = \sqrt{2}$. Define 
\begin{align*}
N_{1} &:= |\{i:\sign(y_i \cdot \langle \xbar_i, \|\wbar\| \rangle) < 0\}|\\
N_{2} &:= |\{i:\sign(y_i \cdot \langle \xbar_i, \wbar \rangle) = 1, y_i\cdot\langle \xbar_i,\wbar \rangle \in [0,\gamma]\}|
\end{align*}
it suffices to bound $N_1$ and $N_2$. Sice $\yst_t \cdot \langle \xbar_t, \wbar \rangle) \ge 0$, we see that each
\begin{align*}
N_2 = \sum_{t=1}^T Z_t, \quad Z_t := \I\{y_t\cdot(\bst + \langle x_t,\wst \rangle) \in [0,\gamma]\}.
\end{align*}
Set $t_{\gamma} := 2\spread(2\gamma,\wsthat) + 8\log(1/\delta)/m$. By \Cref{lem:small_margin},
\begin{align*}
\Exp[Z_t \mid \cF_{t-1}] \le 2\spread(2\gamma,\wsthat) \le t_{\gamma}
\end{align*}
Hence, by \Cref{lem:chernoff},
\begin{align*}
\Pr[N_2 \ge 2Tt_{\gamma}] = \Pr[\sum_{t=1}^TZ_t \ge 2Tt_{\gamma}] \le \exp(-Tt_{\gamma}/8) \le \delta.
\end{align*}
Thus, from \Cref{cor:mix_mistakes_error}, applyied to the vectors $(x_t,1)$, the number of mistakes is at most
\begin{align*}
(8N_1 + 4)\frac{R^2}{\gamma^2} + 2N_2 &\le \frac{16(N_1 + 1)}{\gamma^2} + 2N_2 \tag{$R^2 = 2$}\\
&\le \frac{16(N_1 + 1)}{\gamma^2} + 4Tt_{\gamma} \tag{w.p. $1-\delta$}\\
&= (8N_1 + 4)\frac{R^2}{\gamma^2} + 8T\spread(2\gamma,\wsthat)  + 32\log(1/\delta).
\end{align*}

\end{proof}
\begin{proof}[Proof of \Cref{thm:Tsyb}] Fix any $N \in \N$. Under the Tsybakov smoothness of the adversary,
\begin{align}
\frac{N}{\gamma_N^2} + m\spread(2\gamma_N,\wsthat ) \le \frac{N}{\gamma_N^2} + \sigdir^{-1} T(2\gamma_N)^{1-\alpha} \lesssim \frac{N}{\gamma_N^2} + \sigdir^{-1} T(\gamma_N)^{1-\alpha}
 \label{eq:balance}
\end{align}
Balance both terms by setting $\gamma_N^{3-\alpha} = (N)/(\sigdir^{-1} T)$. Then, for $\rho = \frac{2}{3-\alpha}$, this choice of $\gamma$ ensures
\begin{align*}
\frac{N}{\gamma_N^2} &= (N)^{1-\rho}) (T/\sigdir)^{\rho}
\end{align*}
Since $\gamma_N^{3-\alpha}$ balanced the terms in \Cref{eq:balance}, we have
\begin{align}
\frac{N}{\gamma_N^2} + T\spread(2\gamma_N,\wsthat ) \lesssim (N)^{1-\rho} (T/\sigdir)^{\rho}\label{eq:gam_choice}
\end{align}
For $k \in \N$, let $\cE_{k}:= \{2^{k-1} \le \Nerr \le 2^k\}$. Then, $\Pr[\bigcup_{k=1}^{\ceil{\log T}}] = 1$. Moreover, by applying \Cref{thm:general mistake} with $\gamma_{2^k}$ for each $k$, we have with probability $1-\delta$, if $\cE_k$ holds, then  applying \Cref{eq:gam_choice} with $N = 2^k$,
\begin{align*}
\# \text{mistakes} &\lesssim \Nerr\frac{1}{\gamma_{2^k}^2} + T\spread(2\gamma_{(2^k)},\wsthat) + \log(1/\delta)\\
&\lesssim  (T/\sigdir)^{\rho} ((2^k))^{1- \rho} \cdot(\kapst)^{\frac{2-2\alpha}{3-\alpha}} + \log(1/\delta)\\
&\lesssim  (T/\sigdir)^{\rho} (\Nerr )^{1- \rho}  + \log(1/\delta),
\end{align*}
where in the last line, we use $\Nerr \ge 2^k/2$ on $\cE_k$. Taking a union bound over $k \in [\ceil{\log T}]$, with probability $1-\delta$,
\begin{align*}
\# \text{mistakes}  \lesssim (T/\sigdir)^{\rho} (\Nerr )^{1- \rho}  + \log(\ceil{\log T }/\delta). \end{align*}
\end{proof}

\subsection{Lower bound on $1/\|\wst\|$}
\begin{lemma}\label{lem:kapst} There exists $(\tilde{w},\tilde{b})$ for which $\Pr[ \yst(x_t) = \sign(\tilde{b} + \langle x_t, \tilde{w}\rangle),~ \forall t \ge 1] = 1$, and which satisfy $|\tilde{b}|+|\tilde{w}|^2 = 1$, and $\|\tilde{w}\| \ge 1/2$.
\end{lemma}
\begin{proof} We consider two cases.
\begin{itemize}
	\item Case 1:  $\yst(x)$ is  not constant on $\cB^d_1$. Let $(\tilde{b},\tilde{w})$ be equal to $\alpha(\bst,\wst)$, where $\alpha$ is chosen so that $\tilde{b}^2 +\|\tilde{w}\|^2 = 1$. By positive homegenity of $\sign$, $\yst(x) = \sign(\tilde{b} + \langle x_t, \tilde{w}\rangle)$.  Since $\yst(x)$ is not constant on $\cB^d_1$, we must have $\|\tilde{w}\| \ge |\tilde{b}|$. This means $\|\tilde{w}\|^2 \ge \tilde{b}^2 = 1-\|\tilde{w}\|^2 $. Hence, $\|\tilde{w}\|^2 \ge 1/2$. 
	\item Case 2: Since $\yst(x) \equiv \yst$ is  constant on $\cB^1_d$. For some $\epsilon$ small, set  $ \tilde b = \sqrt{1/2+\epsilon}\yst$, and set $\tilde w = e_1 \sqrt{(1 - \tilde b^2)} = \sqrt{(1/2 - \epsilon)}e_1$, where $e_1$ is the first cannonical basis vector. By construction, $\tilde{b}^2 + \|\tilde w\|^2 = 1$, and $\yst(\tilde b + \langle x, \tilde w \rangle) \ge\sqrt{1/2+\epsilon} - \sqrt{1/2+\epsilon} > 0$. To conclude, we take $\epsilon = 1/4$ (though any $\epsilon$ arbitrarily close to zero would work as well). 
	\end{itemize}
\end{proof}

\end{document}